\PassOptionsToPackage{unicode}{hyperref}
\PassOptionsToPackage{hyphens}{url}
\PassOptionsToPackage{usenames,dvipsnames,svgnames,x11names}{xcolor}
\documentclass[a4paper,11pt]{article}

\usepackage{setspace}
\usepackage{iftex}

\newcommand{\argmin}{\arg\,\min}

\newcommand{\R}{\mathbb{R}}

\newcommand{\N}{\mathbb{N}}
\newcommand{\E}{\mathbb{E}}

\newcommand{\bb}[1]{\boldsymbol{#1}}
\newcommand{\var}{\text{Var}}\usepackage[utf8]{inputenc} 
\usepackage[T1]{fontenc}    
\usepackage{url}            
\usepackage{booktabs}       
\usepackage{amsfonts}       
\usepackage{nicefrac}       
\usepackage{microtype}      
\usepackage{geometry, xcolor}         
\usepackage{enumitem}
\usepackage{fancyhdr}
\newcommand{\upd}[1]{\textcolor{black}{#1}}

\definecolor{babypink}{rgb}{1.0, 0.25, 0.5}
\newcommand{\ohl}[1]{\textcolor{babypink}{#1}} 

\newcommand{\ignore}[1]{}
\RequirePackage{amsthm,amsmath,amsfonts,amssymb}
\RequirePackage{natbib}
\RequirePackage[colorlinks,citecolor=blue,urlcolor=blue]{hyperref}
\RequirePackage{graphicx,xcolor}
\usepackage{siunitx}
\usepackage{booktabs}
\usepackage{varwidth}
\usepackage{enumitem}   
\usepackage{comment}
\usepackage{lineno}
\usepackage{amsbsy}

\numberwithin{equation}{section}
\theoremstyle{plain}

\newtheorem{theorem}{Theorem}[section]
\newtheorem{proposition}{Proposition}
\newtheorem{lemma}[theorem]{Lemma}

\newtheorem{assumption}{Assumption}[section]
\newtheorem{definition}[theorem]{Definition}

\theoremstyle{remark}

\newtheorem*{example}{Example}

\newtheorem{remark}{Remark}[section]

\RequirePackage[utf8]{inputenc} 
\RequirePackage[T1]{fontenc}    
\RequirePackage{multicol,booktabs,longtable, multirow, makecell}       
\RequirePackage{bm}  
\RequirePackage{nicefrac}       
\usepackage{textcomp}
\RequirePackage{microtype}      
\RequirePackage{cleveref}       
\RequirePackage{caption,subcaption}
\usepackage{enumitem}

\graphicspath{{../..}}
\usepackage{enumitem}
\usepackage[linesnumbered,ruled,vlined]{algorithm2e}
 \SetAlgoVlined

\SetKwInput{KwInput}{Input}                
\SetKwInput{KwOutput}{Output}              

\theoremstyle{remark}

\def\IE{\mathbb{E}}
\def\R{\mathbb{R}}

\def\N{\mathbb{N}}

\newcommand{\abs}[1]{\left\lvert #1 \right\rvert}
\newcommand{\dash}{^{\prime}}

\newcommand{\bbhat}[1]{\widehat{\boldsymbol{#1}}}

\newcommand{\IP}{\mathbb{P}}

\newcommand{\Var}{\mathrm{Var}}

\newcommand{\fancyname}{\texttt{WISER}}

\begin{document}

\def\spacingset#1{\renewcommand{\baselinestretch}%
{#1}\small\normalsize} \spacingset{1}


\title{Fast segmentation of watermarked texts from large language models through an epidemic change-point framework}

\author{%
\texorpdfstring{%
\begin{tabular}{ccc}
Soham Bonnerjee & Subhrajyoty Roy & Sayar Karmakar \\[0.35em]
\small University of Chicago
&
\small Washington University in St. Louis
&
\small University of Florida
\end{tabular}%
}{Soham Bonnerjee, Subhrajyoty Roy, Sayar Karmakar}%
}

\date{}

\maketitle


\maketitle
\bigskip
\begin{abstract} 
With the growing use of large language models, concerns over content authenticity have spurred a variety of watermarking schemes. These schemes use secret keys to detect machine-generated text while remaining imperceptible to readers. Detection typically reduces to statistical hypothesis testing for the presence of watermarks, a topic that is now well studied. In contrast, the finer-grained task of localizing which segments of a text are watermarked is much less explored; existing approaches often lack scalability or guarantees robust to paraphrasing and post-editing. We bring a new perspective to this segmentation problem through the lens of \textit{epidemic change-points} and, by exploiting this connection, propose \texttt{WISER}, a novel and computationally efficient watermark segmentation algorithm. We establish finite-sample error bounds and consistency for detecting multiple watermarked segments in a single text. Complementing these theoretical results, our extensive numerical experiments show that \texttt{WISER} outperforms state-of-the-art baseline methods, both in terms of computational speed as well as accuracy, on various benchmark datasets embedded with diverse watermarking schemes. Together, these theoretical and empirical results position \texttt{WISER} as an effective tool for watermark localization and illustrate how classical statistical ideas can yield theoretically valid and computationally efficient solutions to a modern problem of immediate importance.
\end{abstract}
\noindent%

\spacingset{1.2} 

\section{Introduction} 
Recent years have seen widespread popularity and adoption of Large Language Models (LLM) in areas such as media, education, healthcare, and finance- domains where content creation, ownership, and automation~\citep{touvron2023llama, achiam2023gpt} occupy central importance. However, an unfortunate consequence of the exponential ascent of LLMs has been an increased propagation of synthetic texts across the internet. This has raised significant security and legal concerns regarding privacy, content authenticity, and copyright infringement over multiple domains \citep{w1,w2,w3,w4,w5,w6,w7,w8}. In particular, the ability of LLMs to generate a large volume of texts makes them vulnerable to intended or unintended misuse by entities, often in violation of the governing guidelines to achieve potential plagiarism or deceit \citep{ahmed2021detecting, lee2023language}. For example, recently, the use of LLM-generated text without proper attribution has evolved into a full-fledged quagmire in the lawsuit between The New York Times and OpenAI \citep{grynbaum2023times}. In the same mold, our colleagues in academia, and educators more generally, often face a perhaps legally less challenging but equally important issue: AI-assisted education. The use of AI may, prima facie, be encouraged in many low-stakes situations. However, an increased proliferation of LLM-generated texts in critical assessments not only constitutes a malpractice, but also deprives students of the potential to embark upon an important learning curve by themselves, while simultaneously propagating unfair advantages to more privileged students who have access to newer LLM models \citep{w5, wang2024large, darvishi2024impact}.

Such concerns were initially addressed by attempting to identify LLM-generated texts via specific patterns or properties of the said texts, such as cross-entropy or perplexity \citep{mitchell2023detectgpt, zerogpt2024, radvand2025zero}. However, the shortcomings of this approach have become increasingly evident as more and more language models gain the ability to mimic the quirks of a human-generated text. Open-access, publicly funded large language models have been conceptualized as another alternative, mitigating strategy \citep{akiki2022bigscience, workshop2022bloom, shrestha2023building, li2023starcoder,  ustun2024aya}. In a different direction, and probably most relevant with regards to fraud detection in education,  ``Watermarking methods''  have been proposed \citep{w1-1, w1-2}, and widely adopted \citep{w1-3, w1-4} as a detection mechanism. Watermarking schemes primarily exploit the tokenization structure of large language models. In principle, given a sequence of tokens $\omega_1 \ldots \omega_{t-1}$, the LLM generates $\omega_t$ from a multinomial distribution $\IP_t$ over the dictionary $\mathcal{W}$, where $\IP_t$, the \textit{Next Token Distribution} (NTP) is allowed to depend on previous tokens $\omega_1, \ldots, \omega_{t-1}$. Then, watermarking is used to embed statistical signals into LLM-generated tokens, which remain largely unnoticeable without additional information. The key insight behind watermark-based detection schemes is the use of the underlying randomness of LLM-generated outputs by incorporating pseudo-randomness into the text-generation process. When a third-party user publishes text potentially containing LLM-generated outputs with watermarks,  the coupling between the LLM-generated text and the pseudo-random numbers serves as a signal that can be used for detecting the watermark. Crucially, the properties of watermarks allow the user to detect machine-generated texts without requiring knowledge of any particular properties of the text or the LLM. For example, it is conceivable that the academic institution penalizing LLM-generated texts may gain access to the pseudo-random numbers from the particular LLM they deploy in their network system used by the students. We emphasize that the knowledge of these pseudo-random numbers is imperative for the detection mechanism to work, making the effect of watermarking untraceable to general users, who usually do not have access to such ``keys''. 

This usefulness has stimulated a plethora of research proposing myriad watermarking schemes \citep{w2-1, w2-3, w2-4, w2-5, w2-6, w2-7, w2-8, w2-9, w2-10}. Concurrently, much attention has landed on the pursuit of efficient, statistically valid detection schemes \citep{li2025statistical, w2-2, w3-1, w3-2, w3-0, cai2025statistical}, as well as on the more general problems of machine-generated text detection or model equality testing \citep{lavergne2008detecting, solaiman2019release, gehrmann2019gltr, su2023detectllm, mitchell2023detectgpt, w3-2,  vasilatos2023howkgpt, hans2024spotting, li2025statistical, w2-2, w3-1, gao2025model, song2025deep, radvand2025zero}. These detection schemes usually rely on the knowledge of the pseudo-random keys or deterministic hash functions to perform a composite-vs-composite test of hypotheses:
$H_0:$ the entire text $\omega_1\ldots \omega_n$ is unwatermarked (i.e. human generated), vs $H_1:$ the entire text is watermarked or $H_1':$ the text contains watermarked segments \citep{mitchell2023detectgpt, bao2024fastdetectgpt, li2025statistical, zhou2025adadetectgpt}. Usually, such tests depend on the \textit{pivot statistic} $Y_t$s, which are formed from the token $\omega_t$ and the watermarking keys $\zeta_t$. The virtues of the pivot statistics stem from their ancillarity with respect to the next token distributions $\{\IP_t\}$, allowing it to be used without requiring specific knowledge about the LLM architecture or its NTP distributions. Recent advances in this direction have started to shed light on detecting more sophisticated modifications of watermarking by allowing arbitrary data misappropriation \cite{cai2025statistical} and arbitrary modifications such as deletion and replacements \cite{w3-0, xie2025watermark}. However, somewhat curiously, the relatively harder and more fine-grained problem of precisely localizing the watermarked segments from an input text has received only sparse attention. Apart from WinMax \citep{w2-1}, which focuses only on Red-Green watermarking, to the best of our knowledge, the only algorithms tackling the segmentation problem in its generality are \cite{li2024segmenting, pan2025waterseeker} and \cite{zhao2024efficiently}. Most of these algorithms are prohibitively slow and thus unsuited for long texts. Moreover, to the best of our knowledge, no such algorithm designed to efficiently identify multiple watermarked segments has sufficient theoretical validity. This gap in the literature is also pointed out by \cite{li2025optimal}. 

In this paper, we propose \texttt{WISER}\ (\textbf{W}atermark \textbf{I}dentification via \textbf{S}egmenting \textbf{E}pidemic \textbf{R}egions): a \textit{first-of-its-kind} computationally efficient and provably consistent algorithm to locate multiple watermarked segments from mixed-source input texts. Our method is inspired from the classical notion of \textit{epidemic} change-points; this perspective is instrumental for both the theoretical validity and computational efficiency of our algorithm. We summarize our main contributions as follows.  

Firstly, in \S\ref{se:perspective}, we introduce a novel, \textit{epidemic change-point} perspective on the watermark segmentation problem by exploiting an inherent property of the watermarking schemes. In particular, research dealing with testing for the existence of watermarks essentially hinges upon a score function $h$ applied over the pivot statistics $Y_t$, which usually has the property that $\IE[h(Y_t)]$ is much larger for the watermarked tokens than for unwatermarked ones. This property can be visualized in Figure \ref{fig:comparison_plot}, and is elaborated on with examples in Section \ref{se:elevated}. While this \textit{elevated alternatives property} (Assumption \ref{ass:alt-mean}) is crucial in achieving significant power for the testing problems, it has not been formally described and analyzed in this context. However, for a localization problem, this property readily relates it to a separate classical problem of \textit{epidemic} change-point detection. Roughly speaking, an epidemic change-point refers to a situation where a stochastic process deviates in one of its features in an interval and returns to the baseline. In simple words, the changes in the related features occur in interval patches, and outside these patches the process behaves in an i.i.d. or stationary fashion. Since the \upd{score-transformed pivot-statistics $X_t = h(Y_t)$} exhibit very similar behavior in watermarked tokens, this interpretation of patches as epidemic change-point intervals enables us to re-purpose some of the classical insights of change-point literature into a state-of-the-art algorithm to provably locate them and thus solve a modern problem in the area of AI-moderation. Even though \cite{li2024segmenting, li2025adaptive} also relate the watermark segmentation problem to a change-point localization problem, their insights are rather limited, since they identify the end-points of watermarked patches as distinct change-points, which does not respect the nature of watermarked tokens appearing as intervals. 

\begin{figure}[htbp]
    \centering
    \begin{minipage}{0.48\linewidth}
        \centering
        \includegraphics[width=\linewidth]{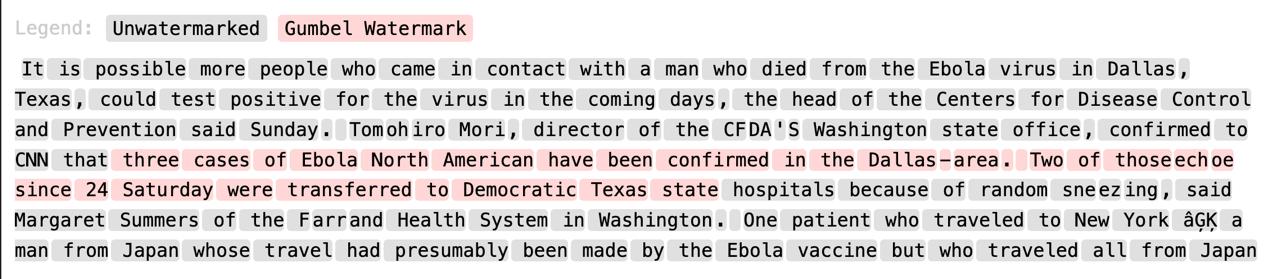}
    \end{minipage}%
    \hfill
    \begin{minipage}{0.48\linewidth}
        \centering
        \includegraphics[width=\linewidth]{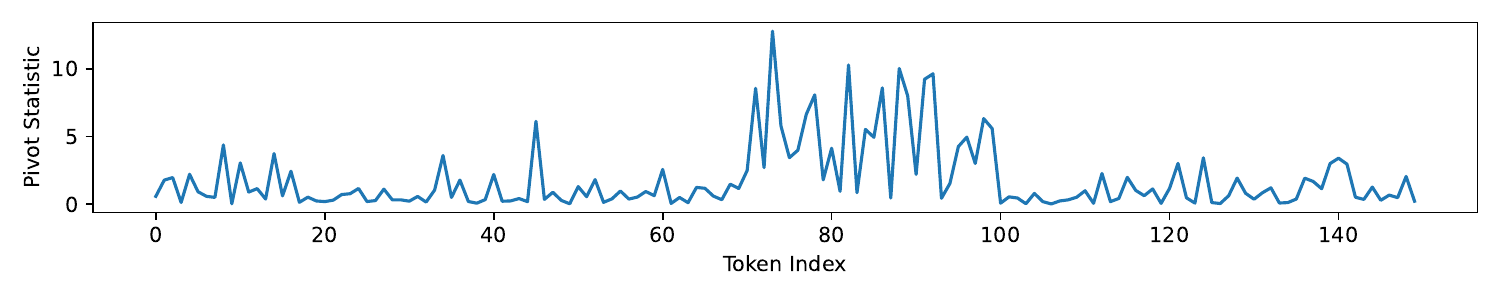}
    \end{minipage}
    \caption{(Left) A text with watermarked tokens $70$-$100$. (Right) The corresponding plot of pivot statistics vs. token.}
    \label{fig:comparison_plot}
\end{figure}
Secondly, in \S\ref{se:multiple-wm}, we transform the epidemic change-point insight into a tractable algorithm, tuned to the peculiarities of the watermarking framework. Specifically, in contrast to the usual setting of independent or stationary random variables in the epidemic change-point literature \citep{levin1985cusum, huskova1995testing, chen2016detecting}, we work with a highly non-stationary setting, devoid of any direct regularity assumption for watermarked-tokens. Motivated by several studies on irregular change-point analysis \citep{kley2024change}, we devise the \fancyname\ algorithm, which is valid irrespective of the LLMs or NTPs. The algorithm is illustrated schematically in Figure \ref{fig:algo}, and also discussed in the Appendix \S\ref{se:algo}. In principle, our algorithm is simple to describe. The \textit{epidemic} interpretation produces a natural estimate for the case of a single watermarked segment, and the general case of multiple watermarked segments can then be dealt with by appropriately restricting the search spaces for each of these segments. The number of such segments is estimated by a series of carefully orchestrated steps (such as block-based tests, and a threshold-based deletion of false-positive blocks), and we further restrict the search space to reduce the computational burden. To summarize, our algorithm simply works with the pivot statistics and the elevated alternatives property, and brings insights from the epidemic change-point theory to tackle the potentially arbitrary non-stationary dependence typically displayed by the pivot statistics corresponding to the watermarked tokens.

Thirdly, in \S\ref{se:multiple-wm} we rigorously establish the theoretical validity of our algorithm in very general scenarios. The theoretical validity of the \texttt{WISER}\ segmentation algorithm arises as an automatic consequence of our perspective. Additionally, we motivate the local estimate used in the last stage of \texttt{WISER}\ by proving in Theorem \ref{thm:single-watermark} that it is consistent in the single watermarked-segment case. To the best of our knowledge, \texttt{WISER}\ is the \textit{first watermark segmentation algorithm with complete theoretical guarantees in the most general case.} It is important to note that the regime of non-stationarity in one or more epidemic patches is different from the usual multiple change-point regime with an even number of breakpoints due to how we perceive signal and noise in a statistical detection problem. Moreover, as we already mentioned above, there is an inherent irregularity intrinsic to how watermarked texts are generated. Our theoretical results settle these issues in a comprehensive fashion. Part of our proof techniques are based on moment, and cumulant generating functions, as well as \cite{danskin1967}'s results, which are novel to both change-point or watermark literature to the best of our knowledge, and these tools can be of independent interest.

Finally, the ingenuity of our algorithm lies not only in its amalgamation of different ideas from statistics, but also in its practicality. In the numerical experiments \S\ref{se:simulation}-\ref{se:shortened-simu}, the theoretical guarantees are reflected in \texttt{WISER}'s superiority over other competitive methods across different watermarking schemes and different language models. In the Appendix \S\ref{se:simu}, we provide additional and extensive numerical experiments to further reinforce the effectiveness of our algorithm, as well as highlighting the novelty of our algorithm compared to the other algorithms in the literature. Another key aspect of its enhanced performance is its speed. \texttt{WISER}\ is specifically designed with many localized steps that reduce its run-time, thereby making it, to the best of our knowledge, the only $O(n)$ watermark segmentation algorithm with provable theoretical guarantees.  

\subsection{Notations}

We delineate some of the notations to be used throughout this paper. The set $\{1, \ldots, n\}$ is denoted by $[n]$. The $d$-dimensional Euclidean space is $\R^d$. For a vector $a \in \R^d$, $|a|$ denotes its Euclidean norm. For a random vector $X \in \R^d$, we denote $\|X\|:=\sqrt{\IE[|X|^2]}$. Throughout the paper, we use the usual Landau notation $O(\cdot), o(\cdot)$ for sequences of real numbers. The analogous stochastic versions, corresponding to stochastic boundedness and in-probability, convergence, are denoted by $O_\IP(\cdot)$ and $o_\IP(\cdot)$ respectively. We also write $a_n \lesssim b_n$ if $a_n \le C b_n$ for some constant $C > 0$, and $a_n \asymp b_n$ if $C_1 b_n \le a_n \le C_2 b_n$ for some constants $C_1, C_2 > 0$. Finally, $\mathcal{L}(X)$ denotes the law of $X$. 

\section{Watermark segmentation: epidemic change-point perspective}\label{se:perspective}

Before we introduce our novel perspective in the context of locating watermarked segments, it is instrumental to establish a consistent framework of watermarking in LLM-generated texts. Let $\mathcal{W}$ denote the dictionary, enumerated as $1,2, \ldots, |\mathcal{W}|$. Given a text input in a tokenized form $\omega_1\ldots \omega_{t-1}$, a watermarked LLM generates the next token $\omega_t$ in an autoregressive manner as $\omega_t = S(\IP_t, \zeta_t)$, where $\IP_t=(P_{t,w})_{w=1}^{|\mathcal{W}|}$ is the next token probability (NTP) distribution at step $t$; $S$ is a deterministic decoder function, and $\zeta_t$ is the pseudo-random variable at $t$. We grant Assumption \ref{ass:ind-of-pseudo} for the $\zeta_t$'s. 
\begin{assumption}\label{ass:ind-of-pseudo}
    For any text $\omega_{1:n}$, there exists corresponding \upd{i.i.d.}  pseudo-random variables $\zeta_{1:n}$ available to the verifier, such that if the token $\omega_t$ at step $t$ is unwatermarked, then $\omega_t$ and $\zeta_t$ are independent conditional on $\omega_{1:(t-1)}$.
\end{assumption}
\noindent It may seem that this assumption invalidates human edits after the LLM generates a text. However, in Appendix \S\ref{se:mixed-source}, we discuss how Assumption \ref{ass:ind-of-pseudo} applies to the mixed-source texts allowing for human edits.
\enlargethispage{2\baselineskip}
\subsection{Pivot statistics and elevated alternatives} \label{se:elevated}
Note that, a text $\omega_{1:n}$ with $K$ disjoint watermarked intervals $I_1, \ldots, I_K$, $I_j \subset [n]$ for $j\in [K]$, can be modeled as
\begin{equation}
    w_t \sim \begin{cases}
        \IP_t, & t \notin I_0:=\cup_{l=1}^K I_k,\\
        S(\IP_t, \zeta_t), & \text{ otherwise},
    \end{cases}
    \quad t = 1, 2, \dots, n.
\end{equation}
We are interested in the statistical problem of estimating the individual intervals $I_1, \ldots, I_K$ as well as $K$. Before proceeding further, it is appropriate to formally introduce the notion of pivot statistics. 
\begin{definition}\label{def:pivotal}
    $Y(\omega, \zeta)$ is called a pivot statistic if $\mathcal{L}(Y)$ is same for all $\omega \in \mathcal{W}$. 
\end{definition}
Pivot statistic has been extremely effective in providing statistically valid testing strategies for the existence of watermarks in mixed-source texts \citep{li2025statistical, w3-0, w3-1}, however, in what follows, we will demonstrate their effectiveness in aiding a localization algorithm. This effectiveness is a result of a simple property of the pivot statistics; they metamorphose the conditional independence of $\omega_t$ and $\zeta_t$ for unwatermarked tokens into $\IP_t$-independent distributions. Formally, this property is described in the following result.
\begin{lemma} \label{lemma:pivot-iid}
    If $S$ denotes the set of unwatermarked tokens, then $\{Y_t\}_{t \in S}$ are i.i.d.
\end{lemma}
This ancillarity is heavily used in all the available statistical analysis of watermarked schemes; nevertheless, for the sake of completion we provide a proof in Appendix \S\ref{se:proof-add}. Lemma \ref{lemma:pivot-iid} enables us to use the notation $\mu_0:=\IE_0[Y(\omega, \zeta)]$ as the expectation of the pivot statistic $Y$ when the token $\omega \sim P$ is not watermarked; on the other hand, $\IE_{1,\IP}[Y(\omega, \zeta)]$ will denote expectation with respect to the randomness of $\zeta$ (i.e. conditional on $P$) when $\omega$ is watermarked according to $(S, \zeta)$-mechanism. Finally, we denote $Y_t:= Y_t(\omega_t, \zeta_t)$. Note that since $Y_t$ is a pivot statistic, so is $h(Y_t)$ for any \textit{score} function $h:\R\to \R$.  Usual tests for watermark detection look at $\sum_{t=1}^n h(Y_t)$ as a statistic for a one-sided test, and put considerable effort into constructing an effective score function $h$ \citep{w2-1, zhao2024efficiently, li2025statistical, cai2025statistical}. Intrinsic to this construction, even though never explicitly stated, is the assumption that $\IE_{1,\IP}[h(Y)]$ is usually larger than $\mu_0$ for any possible NTP distribution $P$. This hypothesis of ``elevated alternatives" can also be empirically viewed in Figure \ref{fig:comparison_plot}.

We formalize this observation with the following hypothesis.
\begin{assumption}[Elevated Alternatives Hypothesis] \label{ass:alt-mean}
     Assume that the next token distribution (NTP) $P$ belongs to a distribution class $\mathcal{P}$. Then, there exists $d>0$ such that $\inf_{\IP \in \mathcal{P}} \IE_{1,\IP}[h(Y)] \geq \mu_0+ d$, where $\IE_{1,\IP}(\cdot)=\IE_1[\cdot | P]$ denotes the unknown distribution of $h(Y)$ when watermarking is implemented on the NTP $P\in \mathcal{P}$ via $S(P, \cdot)$. 
\end{assumption}
This assumption entails that the pivot statistics are effective conditional on any possible NTP from the class $\mathcal{P}$, ruling out trivial cases such as  $Y(\omega, \zeta)\equiv \zeta$. \upd{In practice, $d$ would depend on pivot statistic function $h$, the distribution class $\mathcal{P}$ and watermarking scheme $S$, all of which are usually available to the verifier.} Most standard watermarking schemes satisfy Assumption \ref{ass:alt-mean}; see the following for some concrete examples.

\subsubsection{Examples to Assumption \ref{ass:alt-mean}}\label{se:eg}
In this section, we justify the elevated alternative hypothesis Assumption \ref{ass:alt-mean} by illustrating its occurrence through two popular watermarking schemes.

\begin{example}[Gumbel Watermark, \cite{w1-2}]\label{eg:gumbel}
Let $\zeta = (U_w)_{w \in \mathcal{W}}$ consist of $|\mathcal{W}|$ i.i.d.\ copies of $U(0,1)$. 
The Gumbel watermark is implemented as:
\begin{equation}
    S^{\text{gum}}(\zeta, P) := \arg\max_{w \in \mathcal{W}} \frac{\log U_w}{P_w},
    \label{eq:gumbel}
\end{equation}
The pivot statistic is taken as $Y_t=U_{t,\omega_t}$, $t\in [n]$. From Proposition \ref{lem:gumbel-example} in Appendix, when $\Delta=1/2$, $\inf_{\IP \in \mathcal{P}_{\Delta}}\IE_{1,\IP}[h(Y)] \geq \sum_{n=1}^{\infty}(\frac{1}{n} - \frac{1}{n+2})$, which, in light of $h(Y)\sim \operatorname{Exp}(1)$ entails that $d \geq 1/2$. 
\end{example}

\begin{example}[Inverse Transform Watermark, \cite{w2-2}]
 Consider an NTP distribution $P$ and a permutation $\pi: \mathcal{W} \mapsto S_{|\mathcal{W}|}$, where $S_{|\mathcal{W}|}$ is the group of permutations of $\{1,2,\ldots,|\mathcal{W}|\}$. Further consider 
the multinomial distribution $
\{P_{\pi^{-1}(w)}\}_{w=1}^{|\mathcal{W}|}$. The CDF of this distribution takes the form
\[
F(x;\pi) = \sum_{w' \in \mathcal{W}} P_{w'} \cdot \mathbf{1}_{\{\pi(w') \leq x\}}.
\]
Taking as input $U \sim U(0,1)$, the generalized inverse of this CDF is defined as
\[
F^{-1}(U;\pi) = \min\Bigl\{i : \sum_{w' \in \mathcal{W}} P_{w'} \cdot 
\mathbf{1}_{\{\pi(w') \leq i\}} \,\geq\, U\Bigr\},
\]
which, under the $H_0$ of no watermark, follows the multinomial distribution 
$P$ after applying the permutation $\pi$. The inverse transform watermark is defined as the decoder:
\[
\mathcal{S}^{\mathrm{inv}}(P,\zeta) := \pi^{-1}\!\bigl(F^{-1}(U;\pi)\bigr).
\]
Lemma 4.1 of \cite{li2025statistical} indicates that under the alternative, the distribution of $S^{\mathrm{inv}}$ is intricately interrelated with the NTP $P$. To make the verification of Assumption \ref{ass:alt-mean} tractable, we impose a few assumptions. Assume $|\mathcal{W}|\to \infty$, and with $P_{t, (i)}$ denoting the $i$-th largest co-ordinate of the probability vector $P_{t,(i)}$ for every token $t$ and $i\in [|\mathcal{W}|]$, we also assume 
\[\lim_{|\mathcal{W}|\to\infty} P_{t,(1)} = 1-\Delta 
\ \text{ and }\ 
\lim_{|\mathcal{W}|\to\infty} \log|\mathcal{W}| \cdot P_{t,(2)} = 0 .\]
Consider the pivot statistic
\[ Y_t = \bigl| U_t - \eta(\pi_t(w_t)) \bigr|, 
\ 
\eta(i) := \frac{i-1}{|\mathcal{W}|-1}.\]
Under Theorem 4.1 of \cite{li2025statistical}, $\IE_{1}[1-Y] = \frac{2+\Delta}{3}$, and $\IE_0[1-Y]=\frac{2}{3}$. Therefore, here $d=\frac{\Delta}{3}$. 
\end{example}

To summarize, the pivot statistics $Y_t$ has a mean level $\mu_0$ when the token $\omega_t$ is unwatermarked; on the other hand, we expect the pivot statistics to take comparatively larger values inside the watermarked segments. Interestingly, this observation establishes a ready-made connection to the notion of ``epidemic change-points'', sporadically explored in the classical time-series literature for the past few decades. We discuss this novel perspective in the following section.

\subsection{Watermarked interval in the context of epidemic change-point}
We start with an epidemic change-point model with a single change. The simplest and yet the most popular formulation of a `mean-shift' epidemic model is as follows. Consider the time-series $X_i=\mu_i+Z_i,$ where $Z_i$ is mean-zero stationary process and 
\allowdisplaybreaks \begin{align}
  \mu_i=\mu \text{ if }i \in \{1,\cdots,p\} \cup \{q+1,\cdots,n\} \text{ and } \mu_i=\mu+d \text{ if } i \in \{p+1,\cdots,q\} \label{eq:usual-epidemic}
 \end{align}
With $K$ many true patches, this model reads as follows. For $1<p_1<q_1<p_2\cdots<q_k<n$, 
\begin{equation}\label{eq:multiple}
\mu_i =
\begin{cases}
\mu + d_k, & i \in \{p_k+1,\dots,q_k\} \text{ for some } k=1,\dots,K,\\[2pt]
\mu,       & \text{otherwise},
\end{cases}
\qquad i = 1,\dots,n.
\end{equation}
Epidemic change-point is not new by any means. This framework originated with \cite{levin1985cusum}, who studied the testing for the existence of such epidemic patches for epidemiology applications, with a more comprehensive discussion in \cite{yao1993tests, inclan1994use}. Later on, \cite{huvskova1995estimators, csorgo1997limit, chen2016detecting} have discussed consistency, asymptotic theory, as well as statistical powers of these epidemic estimators and accompanying tests. Other related papers discussing inference tailored to epidemic alternatives can be found in \cite{ravckauskas2004holder,ravckauskas2006testing, ning2012empirical}. Compared to the vast literature for usual change-point analysis, the epidemic change-point literature has been quite sparse, and even then, the focus has remained mostly on testing for the existence of such temporary departure rather than on locating these patches with provable statistical guarantees. In particular, the testing problem deals with the case $d_1=d_2=\cdots d_K=0$. On the other hand, our work concerns simultaneously estimating the number of true locations $K$ and the corresponding patches $(p_i, q_i)$. The literature on localizing multiple epidemic patches is even sparse \citep{zhao2021alternating, juodakis2023epidemic}, and seems to focus only on the much-restricted setting of independent Gaussian observations. Moreover, as discussed in the Introduction as well, due to the nature of pivot statistics, we suffer from a certain irregularity induced by the non-stationarity in the mean of the pivot-statistics for watermarked tokens. Therefore, any potential results or algorithms that might be obtained pertaining to model (\ref{eq:usual-epidemic}) or (\ref{eq:multiple}), are not directly applicable here. Instead, invoking Assumption \ref{ass:alt-mean}, we can only assume that the means of the pivot statistics are separated from the null by at least some margin. This puts us in a position to solve an epidemic mean-shift problem of a new kind, where we can solve the case of localizing multiple patches accounting for this non-stationary departure in the mean of the pivot statistic. 

Very recently \cite{kley2024change} proposed usual change-point detection under the presence of such irregular signals. Concretely, for noisy data of the form $X_t = \mu_t + Z_t, \ t = 1,\ldots,n$ where $\mu_t$ are means or signals and $(Z_t)_{t\in\mathbb{Z}}$ is a stationary mean-zero noise, they considered the following hypothesis testing problem with irregular `non-constant-mean' alternative:
\begin{equation*}
  H_0:\mu_1=\cdots=\mu_n \text{ vs. } H_1:\exists\, \tau \in \{2,\ldots,n\},\ d>0:\ 
  \mu_1=\cdots=\mu_{\tau-1},\quad 
  \mu_\tau,\ldots,\mu_n \ge \mu_1 + d .
\end{equation*}
They also proposed an estimation procedure for the location parameter $\tau$. In this work, we extend their estimators to the epidemic alternative with properties dictated by Assumption \ref{ass:alt-mean}, and provide guarantees of accurate localization. The analysis in \cite{kley2024change} is restricted to a single change-point, whereas the scenario of multiple patches with irregular signals comes naturally in our context. Moreover, the intrinsic dependence introduced by the context of how an LLM token sequence is generated also makes our premise for the error specification quite novel and thus brings out significant technical challenges. To address these challenges, we begin with a simpler problem segmenting of only one watermarked patch in \S\ref{se:theory}.


\section{Single watermarked patch}\label{se:theory}
In this section, we underlay the development of our algorithm by starting with the simpler case of localizing a single watermarked patch. In particular, we propose an estimator to localize a single watermarked segment inside a text, and establish its theoretical consistency with finite sample results. Building on this estimate, in \S\ref{se:multiple-wm} we will formally propose the \texttt{WISER}\ algorithm to detect multiple patches. 

We work with the pivot statistics \(X_t=h(Y_t)\). Recall Lemma \ref{lemma:pivot-iid}, the notation $\mu_0 =\IE_0 X_t$, and Assumption \ref{ass:alt-mean}. The pivot statistics are constructed so that under unwatermarked tokens, they behave like i.i.d. observations with a stable null mean \(\mu_0\). Under watermarking, however, the mean of \(X_t\) inside the true interval \(I_0\) is not assumed constant: token-by-token perturbations can make it vary arbitrarily, and all we rely on is an elevated–alternatives condition, as Assumption \ref{ass:alt-mean} describes, $\inf_{\IP \in \mathcal{P}} \IE_{1,\IP}[h(Y)] \geq \mu_0+ d$. Because of this irregularity, classical epidemic/CUSUM-type scans that presume a constant shift on the affected block are not directly applicable. Instead, we flip the viewpoint and search for an interval whose removal makes the remaining data look as close as possible to the null; this leads us to an initial interval estimator defined by minimizing a biased outside-of-interval surplus. 

In this spirit, let $\Tilde{d}$ be such that there exists $\rho \in (0,1)$ satisfying $d > 2\rho\Tilde{d}$. Based on our discussion above, we adapt the estimator from \cite{kley2024change} for our particular `epidemic' setting.
\allowdisplaybreaks \begin{align}\label{eq:single-estimate}
    \hat{I}= \argmin_{s,t \in [n]} \sum_{k \notin [s,t]} (X_k - \mu_0 - \rho \Tilde{d}).
\end{align}

The role of the bias term \(\rho\tilde d\) is crucial. By subtracting a positive buffer, we make each null token outside any candidate interval contribute a negative expected amount \(-\rho\tilde d\), while any missed watermarked token left outside contributes at least \(d-\rho\tilde d\), which is positive when the signal dominates the buffer. Consequently, underestimating the interval leaves elevated points outside and increases the objective, whereas overestimating it removes extra null points and loses many negative contributions, also increasing the objective; the minimizer is therefore driven toward the smallest interval that excises all elevated tokens. Since the true elevation \(d\) is unknown, we use a proxy \(\tilde d\) together with a tuning factor \(\rho\in(0,1)\) to remain conservative: \(\tilde d\) provides a scale for the elevation we expect, and \(\rho\) controls the tradeoff between being too permissive (small \(\rho\), risking overestimation from null fluctuations) and too strict (large \(\rho\), risking loss of separation if \(d\) is not sufficiently bigger than \(\rho\tilde d\)).

The following theorem analyzes its convergence properties for the case of a single, uninterrupted watermarked region. Subsequently, we discuss some of its connotations in successive remarks.

\begin{theorem}\label{thm:single-watermark}
    Let $\{X_t\}_{t=1}^n:= \{h(Y_t)\}_{t=1}^n$ be the pivot statistics based on the given input text, and assume that $I_0 \subset \{1, \ldots, n\}$ is the only watermarked interval. Grant Assumption \ref{ass:alt-mean}. Denote 
\[ \varepsilon_t=\begin{cases}
    &X_t-\mu_0, t\notin I_0, \\
    & X_t - \mu_t, \ \mu_t:=\IE_{1, \IP_t}[X_t], t\in I_0.
\end{cases}\]
    Suppose the class of distributions $\mathcal{P}$ is closed and compact, and there exists $\eta>0$ such that $\sup_{\IP\in \mathcal{P}}\IE_{1,\IP}[\exp(\eta|\varepsilon|)]<\infty$. Moreover, assume that $\min\{\Var_0(\varepsilon), \sup_{\IP}\Var_{1,P}(\varepsilon)\}>0$. 
    Consider the estimate \eqref{eq:single-estimate} with $\rho$ and $\tilde{d}$ satisfying $d > 2 \rho \tilde{d}$. If there exists a constant $c>0$ such that ${d}\geq c$, then $|\hat{I} \Delta I_0|= O_{\IP}((\rho \tilde{d})^{-1}\big).$
   Here $\Delta$ is the symmetric difference operator and $O_{\IP}$ hides constants independent of $n, \Tilde{d}, \rho,$ and $\mu_0$.
\end{theorem}
The $O(\tilde(\rho \tilde{d})^{-1})$ rate can further be sharpened to $O(\tilde(\rho \tilde{d})^{-2})$ under a local sub-Gaussianity condition (see Proposition \ref{prop:subG} in the Appendix \S\ref{se:proof} ). In fact, under very mild conditions, Theorem \ref{thm:single-watermark} already tackles a more general scenario compared to the only other theoretical result available in a similar context \citep{li2024segmenting}. In contrast to a general watermarked patch, \cite{li2024segmenting} considered a specialized scenario, where only the first half of the text till an arbitrary point is watermarked, reducing the problem to a classical change-point setting. 

The parameter $\tilde{d}$ serves as the \textit{signal strength} in the convergence diagnostics of $\hat{I}$. It allows $\hat{I}$ to look for intervals such that the $\tilde{d}$-biased mean outside that interval is minimized. However, due to the restriction $d>2\rho \tilde{d}$, since the minimum separation $d$ in Assumption \ref{ass:min-sep} is typically unknown, it cannot be used directly. In most cases (see examples in ~\S\ref{se:eg}), a distribution-dependent lower bound $d_{\mathrm{L}} \leq d$ may be available, but relying on $\tilde{d}=d_{\mathrm{L}}$ often sacrifices power, as $\inf_{t\in[n]} \IE_{1,\IP_t}[X_t-\mu_0]$ is usually much larger. Thus, a key step in practice is a data-driven yet valid choice of $\tilde{d}$, which we discuss in \S\ref{se:multiple-wm}. The tuning parameter $\rho$ adjusts the impact of $\tilde{d}$ and mitigates small errors in its selection. Choosing $\rho \approx 0$ is undesirable, as it causes $\hat{I}$ to overestimate $I$ due to fluctuations above $\mu_0$ under the null. Conversely, setting $\rho \approx 1$ can violate the requirement $d>2\rho\tilde{d}$ when $\tilde{d}$ is large. Empirically, $\rho \in [0.1,0.5]$ provides robust performance, and we revisit these choices in our discussion of \texttt{WISER}\ as well as the ablation studies in Appendix~\S\ref{se:ablation-study}.

\begin{remark}[Connection with other performance metric]
Even though Theorem \ref{thm:single-watermark} controls the estimation error in terms of symmetric difference between estimated and true watermarked patches $\hat{I}$ and $I$ respectively, it is straightforward to transform this result in terms of the more familiar Intersection-Over-Union metric $\operatorname{IOU}(I, \hat{I})= \vert I \cap \hat{I}\vert / \vert I \cup \hat{I}\vert$ as $1- \operatorname{IOU}(I, \hat{I}) = \frac{|I \Delta \hat{I}|}{|I \cup \hat{I}|} = O_{\IP}\Big(\frac{1}{|I|{\rho\tilde{d}}}\Big).$
As the text size increases ($n\to \infty$), if $\abs{I}=O(1)$, then the number of unwatermarked tokens is too large, overpowering the signal from the watermarked tokens. Under this ``heavy-edit" regime, no non-trivial test statistic can differentiate between $H_0:$ the entire text $\omega_{1:n}$ is unwatermarked (i.e., human-generated) and $H_1:$ the entire text $\omega_{1:n}$ is watermarked, with reasonable power~\citep{li2025optimal}. The estimation being a harder problem than testing, it is therefore reasonable to assume $|I|\to \infty$ as $n\to \infty$. Therefore, Theorem \ref{thm:single-watermark} essentially entails that $\operatorname{IOU}(I, \hat{I})\to 1$ as $n\to \infty$.
\end{remark}

\upd{The estimator $\hat{I}$ is not only theoretically attractive, but also can be implemented using \textit{Kadane's algorithm} \citep{bentley1984programming, kadane2023two} with a linear computational complexity of $O(n)$. However, there are a couple of practical roadblocks to deploying $\hat{I}$. Firstly, the choice of $\tilde{d}$ is yet unclear. Secondly, it is not straightforward as to how $\hat{I}$ can be generalized to localize multiple watermarked segments. We answer these questions in the next section.}
\section{\fancyname\ : segmenting multiple watermarked patches}\label{se:multiple-wm}

\upd{In the idealistic scenario in which the number of watermarked patches, $K$, is known, it is not hard to generalize the ideas of segmenting a single watermarked patch. In particular, the estimator \eqref{eq:single-estimate} can be directly generalized to the following estimator.}
\begin{align}
        \upd{(\hat{I}_1, \ldots, \hat{I}_K)= \argmin_{\{J_1, \ldots, J_K\} \in \mathcal{I}} \sum_{k \notin \cup_{i=1}^K J_i}(X_k - \mu_0 - \rho \tilde{d}), \label{eq:multiple-oracle-est}}
    \end{align}
\upd{where, for some $K\geq 1$, 
$\mathcal{I}:= \{ \{I_1, I_2, \ldots, I_K\}: I_j \subset \{ 1, \ldots, n\}, I_1 < I_2 < \ldots < I_K\}$ denotes the collection of all ordered sets of $K$ intervals. Here, given a set of disjoint intervals $I_1, \ldots, I_K\subset [n]$, we assume that they are ordered left to right, i.e. $I_{k, R} +1 < I_{k+1, L}$ for all $k$, and write $I_1 < I_2 < \ldots < I_K$ to express this. In Theorem \ref{thm:multiple-oracle}, presented in the Appendix, we show that the estimate \eqref{eq:multiple-oracle-est} leads to consistent segmentation with the same rate of $O_{\IP}(1/(\rho \tilde{d}))$ under fairly mild conditions, as in Theorem \ref{thm:single-watermark}. Theorem \ref{thm:multiple-oracle} may be of independent interest, establishing \textit{first-of-its-kind} theoretical consistency of oracle estimators in multiple epidemic change-point settings. Moreover, the estimator \eqref{eq:multiple-oracle-est} remains amenable to a dynamic implementation of Kadane's algorithm (presented in Algorithm \ref{algo:kadane}) with a run-time of $O(n\log n)$. Since \eqref{eq:multiple-oracle-est} requires the knowledge of $K$, we call this the \textit{oracle} estimate. Despite optimal theoretical performance, however, based on the simulation exercises in \S\ref{se:kadane}, the picture is wildly different in the more practical scenarios, when the knowledge of $K$ is not explicitly known, and when simply an upper-bound of $K$ is available. In these cases, Kadane's algorithm performs poorly, and rather erratically. It may output some spurious intervals, which then have to be detected and removed from the estimates - a process suffering from selective inference. Developing a valid methodology in this case constitutes a non-trivial research direction. Moreover, sometimes the Kadane algorithm, when faced with an incorrect specification of $K$, may club two distinct watermarked patches together, thereby incorrectly classifying the intermediate unwatermarked tokens as watermarked. This may have serious practical implications in many situations, such as academic fraud detection.}

\upd{In cognizance of these issues, we propose \fancyname, which does not require or use any previous knowledge of the number of watermarked patches, $K$.} The main motivation behind our proposed algorithm \fancyname\ is to use the estimator $\hat{I}$ on localized disjoint intervals that are more-or-less guaranteed to contain the true watermarked segments. Such intervals with guarantees are usually recovered as a consequence of some first-stage screening. For the convenience of readers, a schematic diagram of \fancyname\, containing the key steps, is illustrated in Figure~\ref{fig:algo}. The detailed algorithm can be found in Appendix \S\ref{se:algo}. 

\begin{figure}[htbp]
    \centering   
    \includegraphics[width=0.8\linewidth]{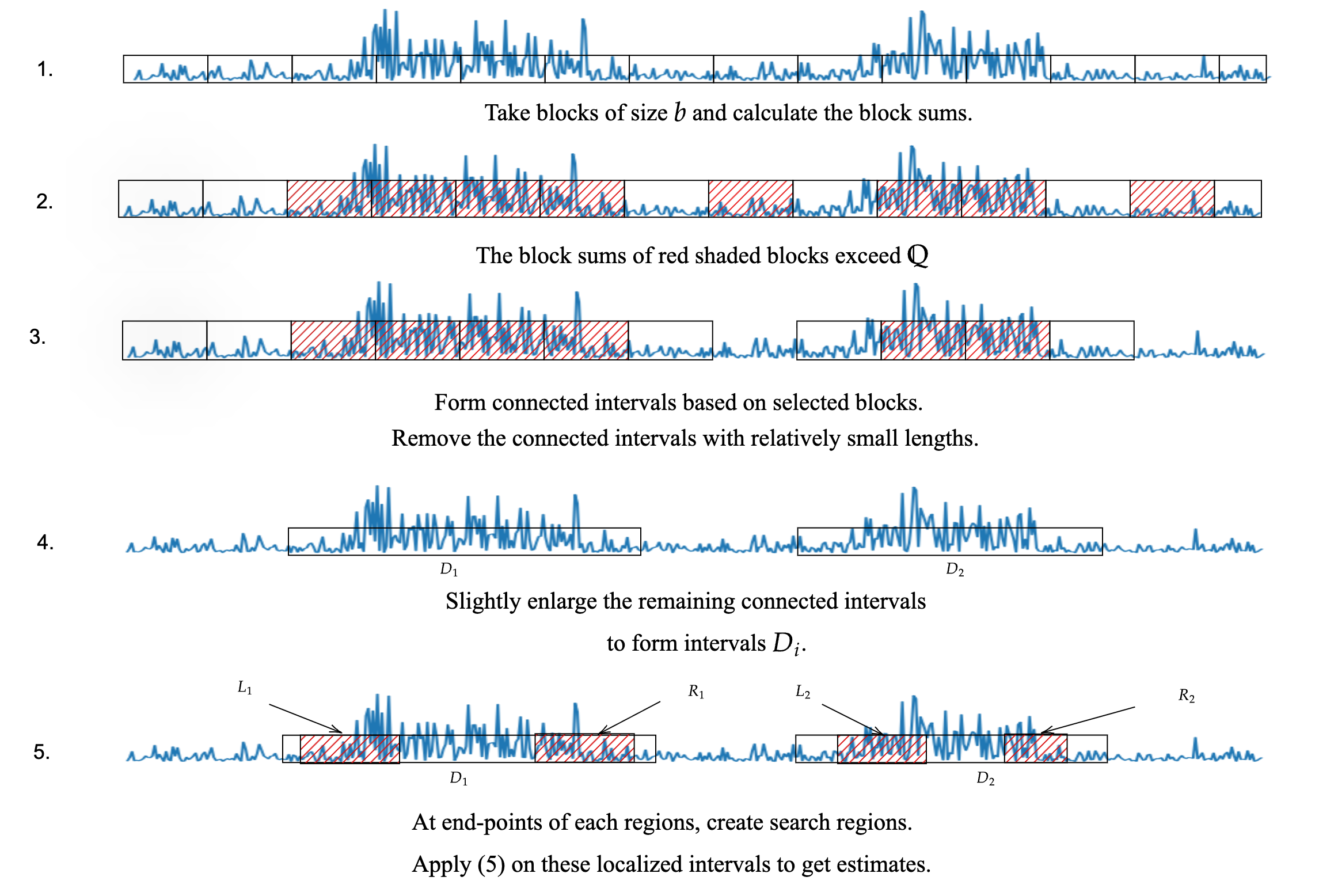}
    \caption{\fancyname\ in action with key steps.}
    \label{fig:algo}
\end{figure}

\upd{Subsequently, we impose a minimum-length and minimum-separation condition on the true
watermarked segments. This condition should be viewed as a sufficient asymptotic resolution
requirement for the first-stage screening step of \fancyname; its practical interpretation and
dependence on the tail behavior of the null pivot statistic are discussed in Remark~\ref{remark:ass-min-sep-effect}.} Formally, for two disjoint intervals $I_1 = (I_{1,L}, I_{1,R})$ and $I_2 = (I_{2,L}, I_{2,R})$, let $d(I_1, I_2):=\min\{ |I_{1, L} - I_{2, R}|, |I_{1, R} - I_{2,L}| \}$.

\begin{assumption}[Minimum separation] \label{ass:min-sep}
    Let $K$ be the number of true watermarked segments, with the segments themselves denoted by $I_j, j\in [K]$. Then there exists a constant $C_0>0$, such that $I^\star:=(\min_{k \in [K]} |I_k|) \wedge (\min_{j\in [K-1]} d(I_k, I_{k+1})) \geq C_0 n^{\upsilon}\log n$ for some $\upsilon>0$.  
\end{assumption}
\begin{remark}

In most practical scenarios, where a test for the existence of the watermark has sufficient power, the size of the watermarked patches will be significant, or should have high entropy. In fact, most of the theoretical literature in LLM watermarking~\citep{li2025statistical, cai2025statistical, w1-1, li2024segmenting, w3-0} assumes that either the entire text, or at least a constant proportion of the text, is watermarked. 
\end{remark}

\upd{Assumption \ref{ass:min-sep} allows for vanishing watermarked patches in the $[0,1]$ scale. Mathematically, we only require the minimum size of the watermarked patches (as well as the minimum separation) to grow only polynomially with the number of tokens $n$.  The oracle estimator in \eqref{eq:multiple-oracle-est} also requires $I^\star \gg \log n$ for its consistency (see Theorem \ref{thm:multiple-oracle}); our Assumption \ref{ass:min-sep} is only mildly stronger. These assumptions are ubiquitous in the analysis of multiple change-point; see Assumption 3.3 of \cite{wbs}, Assumption (B2) of \cite{choejs}, Assumption (H1$\dash$) of \cite{baijasa}, along with \citep{safikhanijasa, frick2014}, as well as in the relatively much sparser literature of analysis of multiple epidemic patches \citep{zhao2021alternating, juodakis2023epidemic}. However, we remark that the aforementioned separation conditions from change-point literature are often proposed under Gaussianity, or under specific dependency structures, none of which hold true for the watermarked interval in our setup.}

In what follows, we explain the step-by-step rationale behind the algorithm. For clarity, we ignore the niceties of $\lfloor \cdot \rfloor$'s and $\lceil \cdot \rceil$'s. Suppose, for convenience, that $\upsilon=1/2$.  

\begin{itemize}[noitemsep,topsep=0pt,leftmargin=*]
    \item \textbf{Blocking stage.} 
    For convenience let $b=\sqrt{n}$. In the first stage, we partition the data into $\sqrt{n}$ consecutive blocks of size $\sqrt{n}$. Let the threshold $\mathcal{Q}$ be given as some quantile of the distribution of the maximum of the block-sums of the pivot statistics under the null of no watermarking. Then, among the blocks, we retain only those blocks for which the corresponding realized sum of pivot statistics exceeds $\mathcal{Q}$. Typically, to avoid multiple testing issues, $\mathcal{Q}$ is chosen as the $(1-\alpha)$-quantile of the \textit{null} (i.e., when there is no watermarking in the entire text) distribution of the maximum block sum over all $n/b$ blocks. 
    
    \item \textbf{Discarding stage.} 
    If $\alpha$ is too small, we risk selecting many spurious blocks; if $\alpha$ is too large, we lose out on power in the first stage itself, failing to accurately identify even the number of watermarked segments. As a calibration step, we form connected components based on selected blocks, and then remove any of the intervals having length smaller than $c \sqrt{\log n}$, $c>0$. The intuition is as follows: the blocks corresponding to the unwatermarked region between them should not be selected; else we lose the localization we are aiming for before implementing $\hat{I}$ piece-meal. Moreover, under Assumption \ref{ass:min-sep}, by definition of $\mathcal{Q}$, $\sqrt{\log n}$ successive unwatermarked blocks will have sums exceeding $\mathcal{Q}$ \textit{only} with vanishing probability. Therefore, any connected interval of selected blocks from the first stage, with length at most $ c \sqrt{\log n}$, must necessarily be spurious. 


   \item \textbf{Enlargement stage.} The above two steps ensure $\hat{K}=K$ with probability approaching 1. Due to Assumption \ref{ass:min-sep}, each of the watermarked segments must correspond to exactly one of the remaining connected regions. Moreover, these intervals are almost accurate estimates of the true segments, but for some additional watermarked regions that might have had a non-null intersection with the discarded blocks. However, from the particular discarding procedure, we know that these additional regions must account for a size at most of the order of $\sqrt{n}$. Therefore, it makes sense to enlarge the connected intervals by $c\sqrt{n}$ for some constant $c>0$, so that now it covers the corresponding true watermarked segments with high probability. These enlarged intervals $D_j$'s remain disjoint with high probability due to Assumption \ref{ass:min-sep}, and are therefore each amenable to (\ref{eq:single-estimate}) to yield $\hat{I}_j$'s.

   
   \item \textbf{Estimating $\tilde{d}$.}  The crucial component behind $\hat{I}_j$ is $\tilde{d}$, which we estimate now. In fact, we plug in the sample mean of the pivot statistics over $\cup_{j=1}^{\hat{K}} D_j$ as $\tilde{d}$. Since $|D_j \Delta I_j| \ll |I_j|$ with high probability, hence $\tilde{d}$ is essentially equal to $(\sum_j |I_j|)^{-1}\sum_{j\in [K], t\in I_j}(X_t - \mu_0)$, which estimates $d$ with some positive bias. The $\rho$ parameter can be used to calibrate it so that $d> 2\rho \tilde{d}$. Typically we choose $\rho \in (0.1, 0.5)$. A smaller value of $\rho$ maintains validity of the procedure but sacrifices the detection accuracy. In Appendix \S\ref{se:ablation-study} we provide an ablation study to discuss the choices of both the parameters $b$ and $\rho$.

   \item \textbf{Reducing computational cost. } We alleviate the increased computational aspect of a naive implementation of (\ref{eq:single-estimate}) by leveraging additional information from the screening stage to reduce the search space. Indeed, due to our blocking and discarding steps, it can be guaranteed with high probability that, for each $j\in [K]$, $D_{j,L}$ is at most $\asymp \sqrt{n}$ distance apart from $I_{j,L}$; similarly $D_{j,R}$ is also at most $\asymp \sqrt{n}$ distance apart from $I_{j,R}$. Therefore, from $D_j$ we can produce search intervals $L_j$, $R_j$ of lengths $\asymp n^{1/2}$ such that $I_{j,L}\in L_j$ and $I_{j,R} \in R_j$ with high probability, and restrict the search to $s\in L_j, t\in R_j$. Consequently, now each implementation of this modified (\ref{eq:single-estimate}) (see Figure \ref{fig:algo}) takes $O((n^{1/2})^2)=O(n)$ amount of computational time, leading to a speed-up while maintaining theoretical validity. 
\end{itemize}



The following result summarizes \upd{various algorithmic} insights into a formal consistency guarantee.


\begin{theorem}\label{thm:multiple-watermark}
Assume that the null distribution of the pivot statistic is absolutely continuous with respect to the Lebesgue measure. Let the number of watermarked intervals $K$ be bounded, and Assumption \ref{ass:min-sep} be granted for the watermarked intervals $I_k, k\in [K]$.
Fix $\alpha \in (0,1)$, and recall the quantities defined in \texttt{WISER}\ described in Figure \ref{fig:algo}. Suppose that $\IE_0[|X-\mu_0|^{p}]<\infty$ for some $p \geq 2$, and let the block length $b=b_n$ satisfy $b_n = O(n^\upsilon)$, and $b_n / n^{1/p} \to \infty$, where $\upsilon>1/p$ is same as in Assumption \ref{ass:min-sep}. Moreover, suppose the threshold $\mathcal{Q}=\mathcal{Q}_n$ is selected so that $\IP_0(\max_{1\leq k \leq \lceil n/b\rceil} S_k> \mathcal{Q})= \alpha$. Finally, assume $d\geq c$ for some constant $c>0$, and $\sup_{\IP \in \mathcal{P}} \IE_{1,\IP}[X] < \infty$.
 Suppose $\varepsilon>0$. Then, under the assumptions of Theorem \ref{thm:single-watermark}, there exist $M_{\varepsilon} \in \R_+$, independent of $n, K,$ and $d$, and $\rho>0$, such that \fancyname\ applied with hyper-parameters $b$ and $\rho$ satisfies
\allowdisplaybreaks \begin{equation}\label{eq:multiple-consistency}
    \liminf_{n\to \infty}\IP\big(\hat{K}=K, \ \max_{k\in [K]}|\hat{I}_k \Delta I_k|< M_{\varepsilon} {d}^{-1} \big) \geq 1-\varepsilon.
\end{equation}
\end{theorem} 

\begin{remark}[\upd{Effect of Assumption~\ref{ass:min-sep}}] \label{remark:ass-min-sep-effect}
We interpret \upd{Assumption~\ref{ass:min-sep} as a sufficient separation condition for the first-stage block screening step of \fancyname, rather than as a sharp finite-sample boundary. The weak tail condition imposed in Theorem~\ref{thm:multiple-watermark} entails its dependence on the scale $n^{1/p}$. In particular, when only $\IE_0|X-\mu_0|^p<\infty$ is assumed, one can look at the details of the screening proof, that it controls the maximum of $O(n/b_n)$ null block sums. A union bound followed by the Fuk--Nagaev inequality yields, for fixed $\varepsilon>0$,}
\begin{align*}
    \upd{\IP_0\left(\max_{1\leq k\leq \lceil n/b_n\rceil}b_n^{-1}(S_k-b_n\mu_0)
    > \varepsilon\right)}
    & \upd{\lesssim \frac{n}{b_n^p} + \frac{n}{b_n}\exp(-c b_n).}
\end{align*}
\noindent \upd{Therefore, under only a finite $p$-th moment assumption, the polynomial term is small by requiring $n/b_n^p\to0$, or equivalently $b_n/n^{1/p}\to\infty$. Since Theorem~\ref{thm:multiple-watermark} also requires $b_n=O(n^\upsilon)$, this leads to the condition $\upsilon>1/p$ in Assumption~\ref{ass:min-sep}. In this sense, the $n^{1/p}$ scale can be seen as the cost of avoiding stronger tail assumptions such as Gaussianity, sub-Gaussianity, boundedness, or the existence of all moments. In fact, for most watermarking schemes in use, the pivot distribution under null will have infinitely many moments, enabling us to take as large a $p$ as possible.}

\upd{Under stronger null-tail assumptions, the same first-stage screening argument can be relaxed. For example, if the null pivot statistic is bounded or sub-Gaussian, the polynomial term above can be replaced by an exponential concentration term, and even logarithmic block lengths will be sufficient for controlling the maximum over the $O(n/b_n)$ null blocks. We state Theorem~\ref{thm:multiple-watermark} under the weaker finite-moment condition because this formulation is relatively distribution-free and does not require model-specific assumptions on the pivot statistics.}\\
\upd{\hspace{0.25cm} In finite samples, $b_n$ should be best viewed as the screening resolution of \fancyname. The watermarked stretches and the gaps between adjacent watermarked stretches should be large enough to be distinguishable at this block level. When an effective watermarked stretch or an inter-segment gap is shorter than this resolution, the first screening stage may merge neighboring watermarked regions or miss very short regions. This limitation is standard for a block-screening method and should not be interpreted as an empirical impossibility threshold. Our additional experiments in Appendix~\S\ref{se:simu-1}, including the tightly spaced setting~\ref{sim:setup3} with $n=1000$, segment length $60$, gap $30$, and $b \approx [\sqrt n]$, are designed to probe such finite-sample behavior.}
\end{remark}

We reiterate that with any choice of $b_n=O(\sqrt{n})$, \fancyname\ has a run-time only of approximately $O(n)$, ignoring log factors. This, to the best of our knowledge, is among the \textit{least computationally expensive} algorithms available in the literature. In view of its theoretical validity under very general conditions, this makes it a useful tool for practical applications. In general, for consistent segmentation, the block lengths only need to satisfy $n^{1/p}\ll b_n \ll  n^{\nu}$, where $\nu$ is as in Assumption \ref{ass:min-sep}, and the pivot statistics has at least $p$ finite moments. In Appendix \S \ref{se:ablation-study}, we undertake a detailed ablation study that deals with the practical choice of $b_n$. 
\vspace{-0.4 cm}
\section{Simulation Studies}\label{se:simulation}

Building on the theoretical results developed in the preceding sections, we now present a series of simulation experiments designed to stress-test the performance guarantee of the proposed \fancyname\ method, under deviations from the idealized assumptions. To this end, we consider three distinct simulation scenarios as follows. In \S\ref{se:s1}, we aim to uncover the effect of temporal dependence on the performance of the \fancyname\ algorithm. Moving on, in \S\ref{se:s3}, we evaluate \fancyname\ on its ability to detect multiple watermarked segments based on simulated text, which then serves as the closest approximation to the additional experiments on real-world scenarios discussed later in \S\ref{se:shortened-simu}. Finally, \S\ref{se:s2} examines the role of Assumption~\ref{ass:alt-mean} on the performance of the algorithm. For each of these experiments, we keep the vocabulary size fixed at $\vert \mathcal{W}\vert = 1000$ unless otherwise specified and perform $5000$ replications. 

\subsection{Effect of temporal dependence on \fancyname}\label{se:s1}

We begin by analyzing how temporal dependence in the underlying next-token prediction (NTP) distributions affects detection accuracy. For each of the $5000$ replications, we generate a sequence of NTP distributions $\{\IP_t\}_{t=1}^T$ according to 
\begin{equation*}
    \IP_t(w) = \dfrac{e^{z_t(w)}}{\sum_{w \in \mathcal{W}} e^{z_t(w)}}, \ 
    z_t(w) = \sqrt{\phi}z_{t-1}(w) + \sqrt{1 - \phi}\sigma_t(w), \ t = 1, 2, \dots, \ w \in \mathcal{W},
\end{equation*}
\noindent where $\phi$ is the auto-correlation coefficient ranging from $0$ to $1$, and $\sigma_t(w)$'s are independent, identically distributed logits generated from a spiked-probability distribution described by
\begin{equation*}
    \IP_t(w) = (1 - \Delta_t) \mathbf{1}_{\{w = w_t^\ast\}} + \dfrac{\Delta_t}{\vert \mathcal{W}\vert - 1} \mathbf{1}_{\{w \neq w_t^\ast\}}
\end{equation*}
\noindent where $w_t^\ast \sim \text{Uniform}(\mathcal{W})$ and $\Delta_t \sim U(10^{-3}, 0.5)$. Note that $\phi=0$ recovers the case where each NTP is generated as an independent spiked distribution, which is equivalent to the scenario considered by~\cite{li2025statistical}. 

For each replication, we generate a text of length $n = 500$, watermark the interval $(50, 300)$ under various schemes, and compute the resulting IOU of the intervals detected by \fancyname. The results are summarized in Figure~\ref{fig:simulation-s1}. Overall, \fancyname\ is remarkably stable across a broad range of temporal dependencies, except in two cases: (i) $\phi=1$, and (ii) red-green watermark scheme with smaller values of $\phi$. 

\begin{figure}[htbp]
    \centering
    \includegraphics[width=0.48\linewidth]{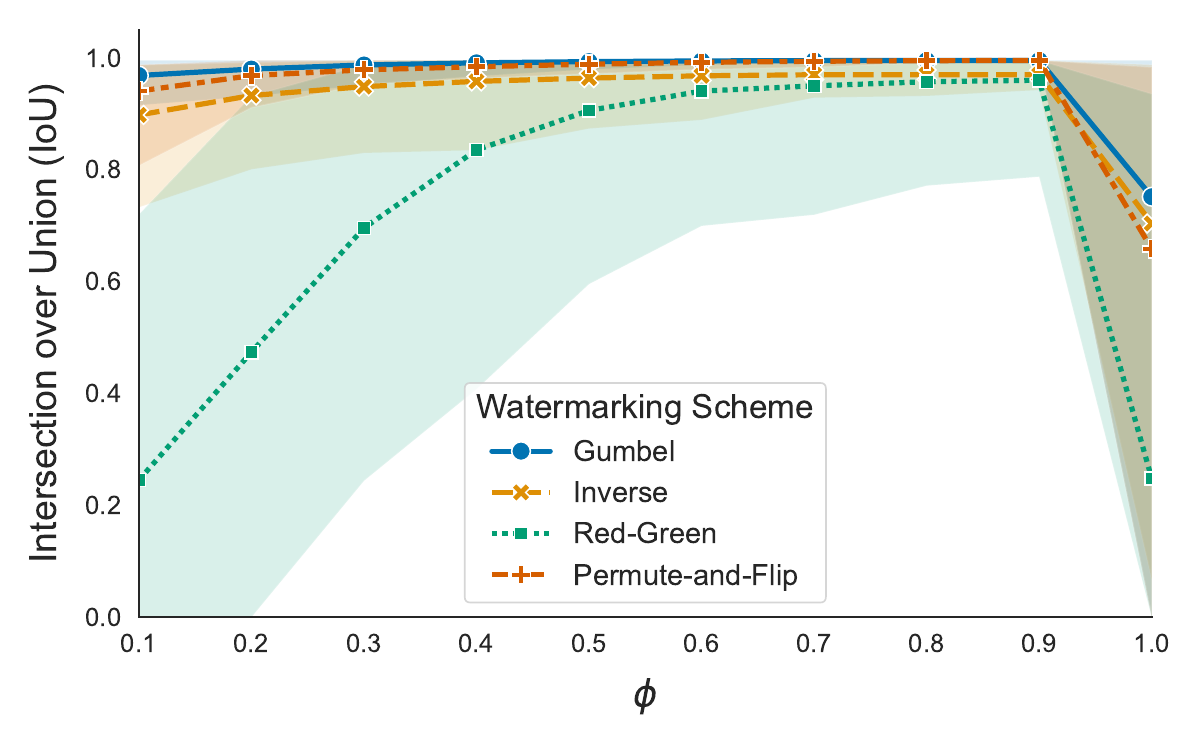}
    \includegraphics[width=0.48\linewidth]{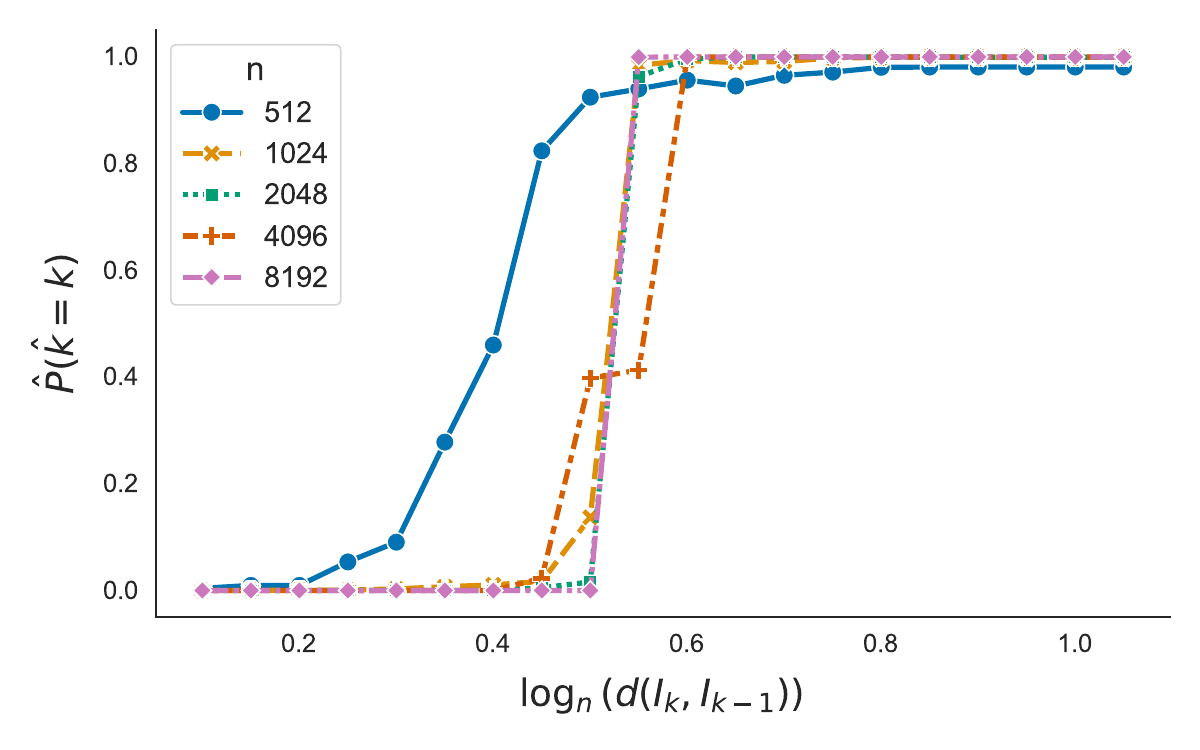}
    \caption{(Left) IOU of \texttt{WISER} algorithm for different levels of correlation in the token generation process given in simulation scenario of \S\ref{se:s1}. (Right) Empirical proportion of repetitions where \texttt{WISER} algorithm determines correct number of intervals in simulation scenario \S\ref{se:s3} as a function of the logarithm of gap $\log_n(d(I_k, I_{k-1}))$.}
    \label{fig:simulation-s1}
\end{figure}

In the extreme case $\phi = 1$, all NTP distributions $\{\IP_t\}$ collapse to the initial spiked distribution $P_0$. Because of the spiked-nature of $P_0$, a decoding (or sampling) from this model produces nearly deterministic text consisting almost entirely of the token $w_0^\ast$. Therefore, watermarking has virtually no opportunity to influence the output, because deviations from $w_0^\ast$ occur only through extremely low-probability events. This makes it extremely difficult for any detection algorithm to distinguish between watermarked and unwatermarked cases. Also, in this degenerate case, because of the deterministic nature of the generated text, the variance $\min\{ \var_0(\epsilon), \sup_{P} \var_{1,P}(\epsilon) \}$ reduces to almost zero, violating the assumptions of Theorem~\ref{thm:single-watermark}. 

On the other extreme, when $\phi = 0$, each of the spiked distributions is independently generated, producing the maximum probability at different tokens. The red-green scheme perturbs logits by adding a bias only on the tokens from the green list, but this influence is mitigated by the large and rapidly changing spikes. Thus, watermarked and unwatermarked texts become nearly indistinguishable.

\subsection{Experiments on multiple watermarked segment detection} \label{se:s3}

We next evaluate \fancyname\ in settings where multiple watermarked intervals appear within the text. Consider text of length $n$ containing three Gumbel-watermarked intervals: $(0.35n - g, 0.45n - g), (0.45n, 0.55n)$ and $(0.55n + g, 0.65n + g)$. Here $g$ denotes the gap between two successive intervals, and is selected as $g= (\tilde{g} \vee 2) \wedge 0.3n$, where $\tilde{g}\in \{n^{0.1}, n^{0.15}, \dots, n^{0.9}, n^{0.95}\}$.

Theorem~\ref{thm:multiple-watermark} assures successful detection of the intervals by \fancyname\ algorithm whenever the gaps satisfy $d(I_k, I_{k-1}) \asymp \sqrt{n}$. Figure~\ref{fig:simulation-s1} (right panel) displays, as a function of $\log_n(d(I_k, I_{k-1}))$ for different text lengths $n$, the empirical probability that \fancyname\ correctly detects three watermarked intervals. As expected, for each of the $n$ considered here, the empirical detection probability rises sharply from $0$ to $1$ around $\log_n(d(I_k, I_{k-1})) \approx 0.5$, consistent with the theoretical results.


\vspace{-0.05cm}
\subsection{Effect of Assumption \ref{ass:alt-mean} on \fancyname}\label{se:s2}

In order to properly characterize the importance of the elevated mean of pivot statistics for watermarked texts, we consider two different simulation experiments. 

In the first \upd{experiment}, we apply the red-green watermark on a single interval $(0.3n, 0.7n)$ for an $n$-length text, \upd{using the red-green scheme.} We vary the watermark strength by setting $\delta \in \{1.5, 2.0, \ldots, 3.5\}$, where $\delta$ is the bias added to the logit of the tokens from the green list. The corresponding IOU values for \fancyname\ are shown on the \upd{left} panel of Figure~\ref{fig:simulation-s2}. 

\begin{figure}[ht]
    \centering
    \includegraphics[width=0.48\linewidth]{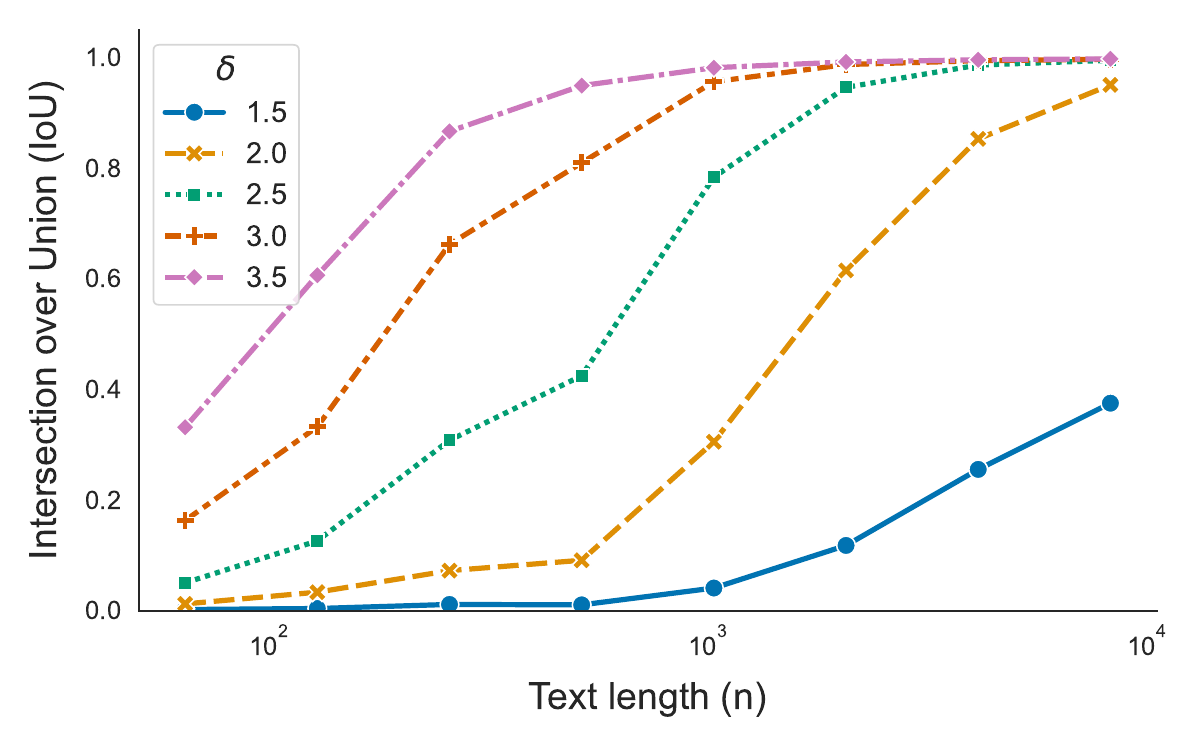}
    \includegraphics[width=0.48\linewidth]{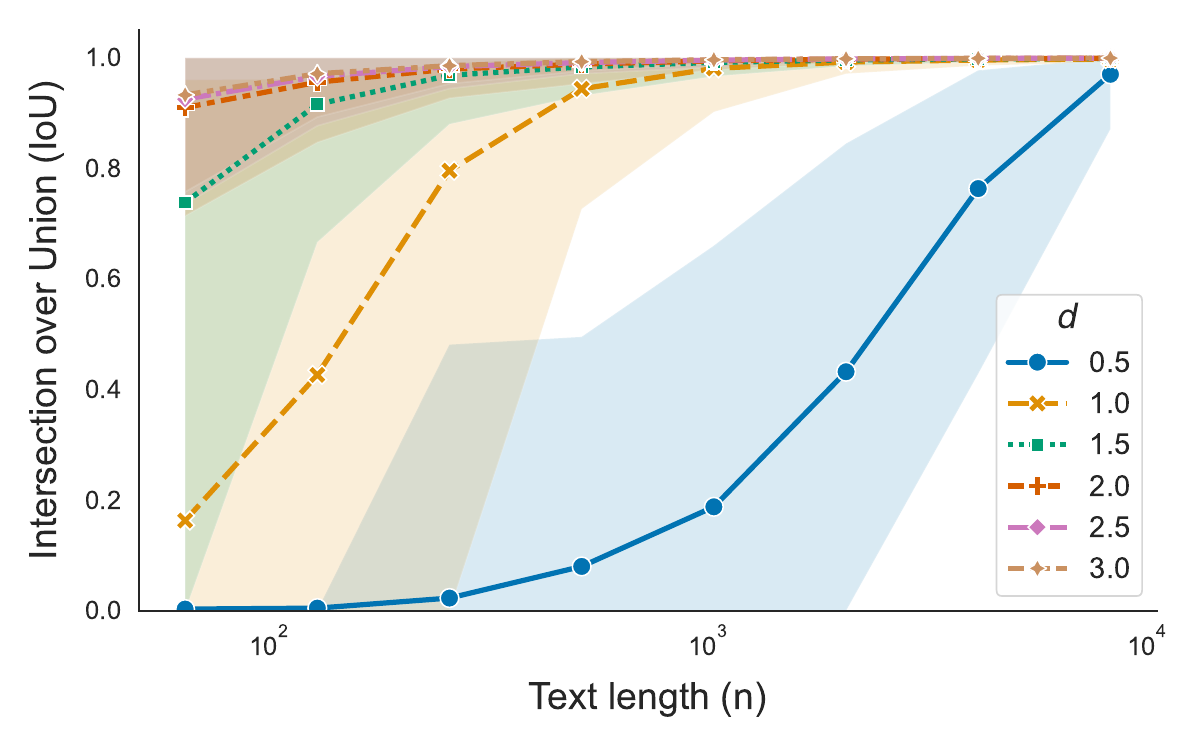}
    \caption{IOU of \texttt{WISER} algorithm across different length texts. (Left) for different levels of $\delta$ in Red-green watermarking; (Right) for different levels of $d$ as in Assumption~\ref{ass:alt-mean} for Gumbel watermarking scheme. X-axis is on a log scale for both plots.}
    \label{fig:simulation-s2}
\end{figure}

On the other hand, such control of the strength of the watermarking is not possible for the Gumbel watermarking scheme. Therefore, as an illustration, we perform another simulation study by directly generating the pivot statistic without any explicit choice of the underlying watermarking scheme. We choose the pivot statistic at the unwatermarked regions as i.i.d. exponentially distributed random variables with mean $1$ (note that, this is the distribution of the pivot statistic $Y_t$ corresponding to the Gumbel watermarking scheme for the unwatermarked tokens), and choose the watermarked regions as i.i.d. normal random variables with mean $(1+d)$ and variance $1$. The results for this simulation are illustrated on the \upd{right} panel of Figure~\ref{fig:simulation-s2}. Both the plots present in Figure~\ref{fig:simulation-s2} convey the asymptotic consistency of the algorithm, as well as establish an empirical validity of the error bound given in Theorem~\ref{thm:single-watermark}. \upd{Evidently from the left subplot, increasing $\delta$ increases the gap between null and the alternative, and IOU converges relatively faster to $1$, even with moderately large $n$. For the second experiment,} as $d$ decreases, watermark detection becomes more difficult, and a larger text length $n$ is required to achieve the same level of IOU.

\section{Numerical Experiments}\label{se:shortened-simu}

While the previous section corroborates the theoretical results with numerical simulations, this section aims to demonstrate the superiority of the proposed \texttt{WISER}\ method over existing state-of-the-art (SOTA) algorithms, when applied in a real-world scenario. In \S\ref{se:shortened-simu-1}, we compare its accuracy against competitive methods on a benchmark dataset across multiple watermarking schemes, and in \S\ref{se:shortened-simu-2}, we assess its computational efficiency. Due to space constraints, we provide additional numerical experiments in Appendix \S\ref{se:simu}. We encourage the readers to check it out for more practical insights, including, \textbf{(i)} a detailed explanation of the benefits of \texttt{WISER}\ over other SOTA algorithms (\S\ref{se:wiser-best-why}), \textbf{(ii)} experiments quantifying the effect of watermark intensity and length across different algorithms (\S\ref{se:effect-wm-intensity}), and \textbf{(iii)} an ablation study (\S\ref{se:ablation-study}) highlighting the stability of our method across tuning parameter choices. The datasets and the large language models were acquired from the open-source \href{https://huggingface.co/models}{Huggingface} library. All the relevant reproducible codes and figures can be found in the anonymous \href{https://github.com/tukeystats/WISER}{GitHub repository}.
\subsection{Comparative performance of \texttt{WISER}}\label{se:shortened-simu-1}

Within the relatively limited body of literature on the identification of watermarked segments from mixed-source texts, \texttt{Aligator}~\citep{zhao2024efficiently}, \texttt{SeedBS-NOT}~\citep{li2024segmenting}, and \texttt{Waterseeker}~\citep{pan2025waterseeker} algorithms have emerged as the leading methods, producing the most accurate results so far. For an extensive comparison, our experimental setup involves completion of randomly selected $200$ prompts from the Google C4 news dataset\footnote{\url{https://www.tensorflow.org/datasets/catalog/c4}}. We include language models spanning a wide range of scales: parameter sizes varying from $125$ million to $8$ billion, and vocabulary sizes ranging in $32$-$262$ thousands; for watermarking schemes, we consider Gumbel-max trick~\citep{w1-2}, Inverse transform~\citep{w2-2}, Red-green watermark~\citep{kirchenbauer2023watermark}, and Permute-and-Flip watermark~\citep{w2-7}. In each scenario, the first $50$ tokens of a news article have been provided as inputs to the language models, and \upd{$n = 2500$ output tokens are recorded. Among these output tokens, there are five watermarked segments with successively increasing lengths and increasing gaps; see Appendix~\S\ref{appendix:sim-settings} for further details.} The specific tuning parameter choices for \texttt{WISER}\ are provided in \S\ref{se:simu}. Table \ref{tab:results-gumbel-n2500} showcases the results for the Gumbel watermarking scheme. It is evident that \texttt{WISER}\ outperforms all the other algorithms across all the metrics for each model, \upd{achieving near optimal performance close to the oracle estimator}. The detailed discussions, including the specific metrics used and additional results and insights, are provided in Appendix~\S\ref{se:simu-1}.


\begin{table}[ht]
\centering
\resizebox{\textwidth}{!}{
\begin{tabular}{llrrrrrr}
\toprule
\textbf{Model Name} & \textbf{Method} & \textbf{IOU} & \textbf{Precision} & \textbf{Recall} & \textbf{F1} & \textbf{RI} & \textbf{MRI} \\
\midrule
\multirow{5}{*}{
        \makecell[l]{facebook \\ opt-1.3b \\ (50272)}
    } & \ohl{\texttt{Oracle}} & \ohl{0.965 (0.87-0.99)} & \ohl{0.997} & \ohl{0.997} & \ohl{0.997} & \ohl{0.995 (0.98-1.00)} & \ohl{0.995 (0.98-1.00)} \\
 & \texttt{WISER} & \textbf{0.966 (0.84-0.99)} & \textbf{0.992} & \textbf{0.994} & \textbf{0.993} & \textbf{0.991 (0.98-1.00)} & \textbf{0.991 (0.98-1.00)} \\
 & \texttt{Aligator} & 0.680 (0.25-0.92) & 0.168 & 0.958 & 0.285 & 0.957 (0.88-0.99) & 0.951 (0.86-0.99) \\
 & \texttt{SeedBS} & 0.435 (0.06-0.91) & 0.840 & 0.892 & 0.865 & 0.961 (0.91-0.99) & 0.933 (0.87-0.99) \\
 & \texttt{WaterSeeker} & 0.508 (0.31-0.64) & 0.812 & 0.763 & 0.787 & 0.906 (0.79-0.96) & 0.895 (0.77-0.95) \\
\midrule
\multirow{5}{*}{
        \makecell[l]{facebook \\ opt-125m \\ (50272)}
    } & \ohl{\texttt{Oracle}} & \ohl{0.953 (0.68-0.99)} & \ohl{0.996} & \ohl{0.996} & \ohl{0.996} & \ohl{0.992 (0.92-1.00)} & \ohl{0.991 (0.91-1.00)} \\
 & \texttt{WISER} & \textbf{0.950 (0.66-0.99)} & \textbf{0.996} & \textbf{0.988} & \textbf{0.992} & \textbf{0.988 (0.91-1.00)} & \textbf{0.986 (0.90-1.00)} \\
 & \texttt{Aligator} & 0.827 (0.24-0.96) & 0.343 & 0.953 & 0.504 & 0.957 (0.67-1.00) & 0.953 (0.64-1.00) \\
 & \texttt{SeedBS} & 0.457 (0.05-0.93) & 0.840 & 0.926 & 0.881 & 0.964 (0.85-0.99) & 0.937 (0.80-0.99) \\
 & \texttt{WaterSeeker} & 0.502 (0.24-0.64) & 0.820 & 0.760 & 0.789 & 0.888 (0.47-0.95) & 0.876 (0.44-0.95) \\
\midrule
\multirow{5}{*}{
        \makecell[l]{google \\ gemma-3-270m \\ (262144)}
    } & \ohl{\texttt{Oracle}} & \ohl{0.943 (0.50-0.99)} & \ohl{0.994} & \ohl{0.994} & \ohl{0.994} & \ohl{0.991 (0.94-1.00)} & \ohl{0.990 (0.93-1.00)} \\
 & \texttt{WISER} & \textbf{0.938 (0.40-0.99)} & \textbf{0.995} & \textbf{0.981} & \textbf{0.988} & \textbf{0.988 (0.88-1.00)} & \textbf{0.987 (0.87-1.00)} \\
 & \texttt{Aligator} & 0.649 (0.00-0.91) & 0.152 & 0.945 & 0.262 & 0.931 (0.11-0.99) & 0.923 (0.08-0.99) \\
 & \texttt{SeedBS} & 0.408 (0.00-0.91) & 0.826 & 0.859 & 0.842 & 0.938 (0.51-0.99) & 0.909 (0.48-0.99) \\
 & \texttt{WaterSeeker} & 0.515 (0.35-0.66) & 0.820 & 0.794 & 0.807 & 0.906 (0.76-0.96) & 0.895 (0.74-0.95) \\
\midrule
\multirow{5}{*}{
        \makecell[l]{meta-llama \\ Meta-Llama-3-8B \\ (128256)}
    } & \ohl{\texttt{Oracle}} & \ohl{0.977 (0.92-1.00)} & \ohl{1.000} & \ohl{1.000} & \ohl{1.000} & \ohl{0.998 (0.99-1.00)} & \ohl{0.998 (0.99-1.00)} \\
 & \texttt{WISER} & \textbf{0.971 (0.90-1.00)} & \textbf{0.998} & \textbf{1.000} & \textbf{0.999} & \textbf{0.997 (0.99-1.00)} & \textbf{0.997 (0.99-1.00)} \\
 & \texttt{Aligator} & 0.874 (0.62-0.96) & 0.536 & 0.976 & 0.692 & 0.988 (0.97-0.99) & 0.987 (0.96-0.99) \\
 & \texttt{SeedBS} & 0.450 (0.03-0.97) & 0.809 & 0.899 & 0.851 & 0.981 (0.95-1.00) & 0.953 (0.89-1.00) \\
 & \texttt{WaterSeeker} & 0.469 (0.15-0.71) & 0.846 & 0.564 & 0.677 & 0.854 (0.59-0.94) & 0.841 (0.57-0.93) \\
\midrule
\multirow{5}{*}{
        \makecell[l]{mistralai \\ Mistral-7B-v0.1 \\ (32000)}
    } & \ohl{\texttt{Oracle}} & \ohl{0.949 (0.71-0.99)} & \ohl{0.997} & \ohl{0.997} & \ohl{0.997} & \ohl{0.993 (0.96-1.00)} & \ohl{0.992 (0.95-1.00)} \\
 & \texttt{WISER} & \textbf{0.947 (0.73-0.99)} & \textbf{0.990} & \textbf{0.996} & \textbf{0.993} & \textbf{0.993 (0.96-1.00)} & \textbf{0.992 (0.95-1.00)} \\
 & \texttt{Aligator} & 0.649 (0.19-0.94) & 0.211 & 0.911 & 0.343 & 0.962 (0.86-0.99) & 0.955 (0.84-0.99) \\
 & \texttt{SeedBS} & 0.378 (0.06-0.87) & 0.829 & 0.846 & 0.838 & 0.955 (0.84-0.99) & 0.923 (0.81-0.98) \\
 & \texttt{WaterSeeker} & 0.469 (0.24-0.63) & 0.854 & 0.579 & 0.690 & 0.856 (0.59-0.95) & 0.844 (0.57-0.95) \\
\midrule
\multirow{5}{*}{
        \makecell[l]{princeton-nlp \\ Sheared-LLaMA-1.3B \\ (32000)}
    } & \ohl{\texttt{Oracle}} & \ohl{0.960 (0.78-0.99)} & \ohl{0.996} & \ohl{0.996} & \ohl{0.996} & \ohl{0.994 (0.97-1.00)} & \ohl{0.994 (0.96-1.00)} \\
 & \texttt{WISER} & \textbf{0.960 (0.79-0.99)} & \textbf{0.996} & \textbf{0.988} & \textbf{0.992} & \textbf{0.993 (0.95-1.00)} & \textbf{0.993 (0.95-1.00)} \\
 & \texttt{Aligator} & 0.725 (0.38-0.93) & 0.197 & 0.960 & 0.327 & 0.971 (0.93-0.99) & 0.966 (0.92-0.99) \\
 & \texttt{SeedBS} & 0.412 (0.08-0.89) & 0.848 & 0.893 & 0.870 & 0.968 (0.92-0.99) & 0.937 (0.88-0.99) \\
 & \texttt{WaterSeeker} & 0.501 (0.34-0.63) & 0.845 & 0.698 & 0.764 & 0.896 (0.75-0.95) & 0.885 (0.74-0.95) \\
\midrule
\bottomrule
\end{tabular}}
\caption{\upd{Results for Gumbel Watermarking in experimental setting with $n = 2500$ and $K = 5$.}}
\label{tab:results-gumbel-n2500}
\end{table}

\subsection{Time Comparison} \label{se:shortened-simu-2}

\begin{figure}[htbp]
\centering
\includegraphics[width=\linewidth]{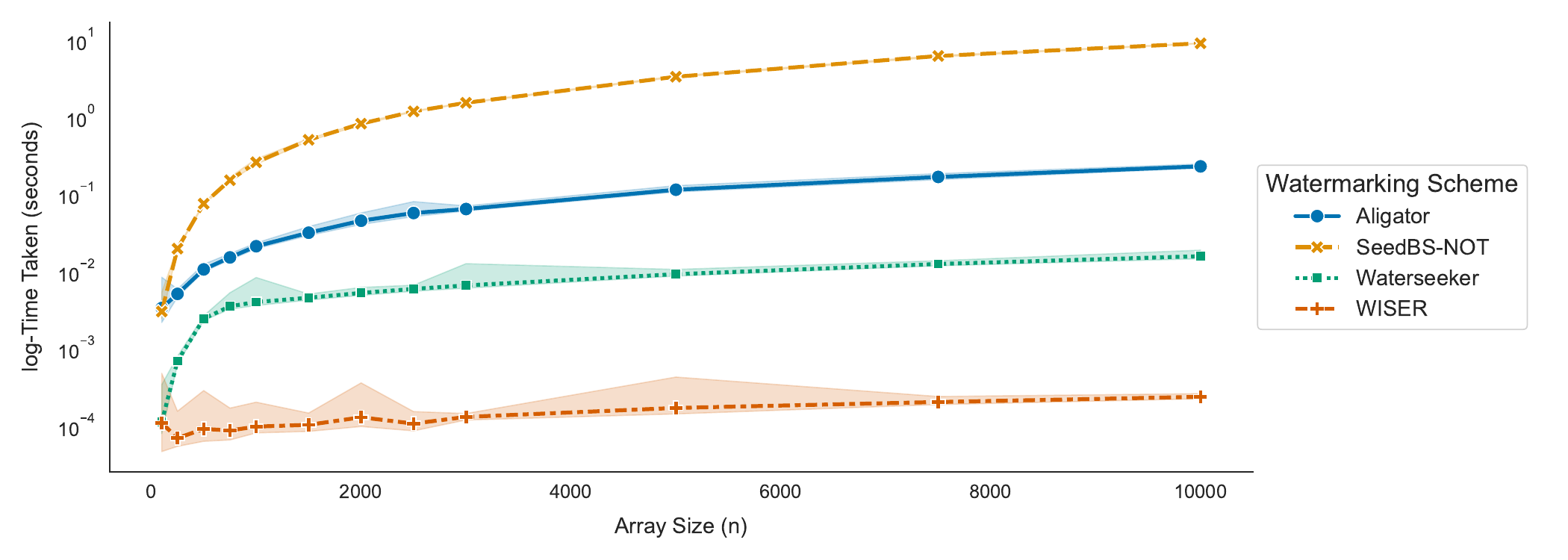} 
\caption{\small Time complexity (seconds) for various algorithms as a function of completion lengths ($n$). Y-axis is in log-scale, with $95\%$ confidence interval shown in shades.}
\label{fig:time-complexity} 
\end{figure}

As established in \S\ref{se:multiple-wm}, the proposed \texttt{WISER}\ algorithm achieves a computational complexity of $\approx O(n)$. Figure~\ref{fig:time-complexity} provides empirical evidence supporting this theoretical claim and, in addition, compares the runtime behavior of \texttt{WISER}\ with other state-of-the-art methods. For this experiment, we randomly create an $n/6$-length watermarked segment using the Gumbel-max trick with NTP generated by Google's Gemma-3 model; block size was taken as $\lceil \sqrt{n}\rceil$ and $\rho = 0.1$. The results clearly indicate that \texttt{WISER}\ consistently outperforms competing approaches in terms of computational efficiency, emerging as the fastest among all methods considered in this study.

\section{Conclusion}
In this paper, we introduced \texttt{WISER}, a first-of-its-kind algorithm for efficient and theoretically valid segmentation of watermarked intervals in mixed-source texts. By framing watermark localization as an epidemic change-point problem, we bridged a novel connection between classical statistical theory and a modern challenge in generative AI, and also designed a linear time algorithm with provable consistency guarantees, which were further confirmed by our extensive numerical experiments. Beyond the findings of this paper, it is also crucial to theoretically investigate the robustness of the proposed algorithm under human edits~\citep{w3-0}; as a roadmap, we have already included some relevant discussion in Appendix \S\ref{se:mixed-source}. Its applicability to multimodal (e.g., audio, image, video) settings \citep{qiu2024multimodalwatermark} also presents opportunities for future research.

\section{Data availability satement}\label{data-availability-statement}

The prompts used in this paper are publicly available as \href{https://www.tensorflow.org/datasets/catalog/c4}{C4 News dataset} with the ``tensorflow'' package. All the large language models used in Section~\ref{se:shortened-simu} are available in \href{https://huggingface.co/models}{HuggingFace} repository.

\vspace*{0.5cm}
\noindent
\textbf{AI disclosure statement}\footnote{\footnotesize \href{https://leidendeclaration.ai/}{\texttt{https://leidendeclaration.ai/}}}: ChatGPT 5 was used to simplify some proofs and correct grammatical errors; however, the authors take responsibility for the correctness of all claims in the paper.


\bibliographystyle{chicago}
\bibliography{references}

\newpage
\appendix  
\begin{center}
    \Large Appendix
\end{center}
This appendix is devoted to detailed proofs of our theoretical statements, and additional experimental evidence justifying \texttt{WISER}. \S\ref{se:algo} formally describes the \fancyname\ algorithm. In \S\ref{se:mixed-source}, we discuss how Assumption \ref{ass:ind-of-pseudo} can be implemented even in presence of human-edits. \upd{Appendix \S\ref{se:kadane} includes a detailed discussion on the oracle algorithm for multiple watermarked patch segmentation, along with its connection to Kadane's algorithm.} \S\ref{se:simu} complements the short experimental section in \S\ref{se:shortened-simu} by providing extensive numerical studies concerning the empirical behavior of \texttt{WISER}. Finally, in \S\ref{se:proof}, we provide the detailed mathematical arguments behind \texttt{WISER}.

\section{\fancyname\ algorithm}\label{se:algo}

\begin{algorithm}
    \caption{Subroutine \texttt{Block\_Thresholding} of \fancyname}\label{alg:watermark-block-thres}
    \KwInput{$(X_i)_{i \in [n]}$, block size $b$, threshold $\mathcal{Q}$}
    \For{$k \gets 1$ \textbf{to} $\lceil n/b \rceil$}{
      $B_k \;\gets\; [(k-1) b + 1,\;
             \min (kb,\;n\bigr)\bigr].$ \;
      $S_k \leftarrow \sum_{l \in B_k} X_l.$ \;
    }
    $\widetilde{\mathcal{M}} \gets \{ \}.$ \;
    \For{\upd{$k \gets 1$ \textbf{to} $\lceil n/b \rceil$}}{
      \If{$S_k> \mathcal{Q}$}{
        append $k$ to $\widetilde{\mathcal{M}}$ \;
      }
    }
    \KwRet{$\widetilde{\mathcal{M}}$}\;
\end{algorithm}

\begin{algorithm}
    \caption{\upd{Subroutine \texttt{Merging} of \fancyname}}\label{alg:watermark-merging}
    \KwInput{$\widetilde{\mathcal{M}} = \{ k_1, \dots, k_m\}$, tuning parameter $c$}
    $I \gets \{\}$, $s \gets \text{NULL}, e \gets \text{NULL}$\;
    \For{$j \gets 1$ to $m$}{
        \If{$k_j - 1 \notin \mathcal{M}$ \textbf{and} $k_j + 1 \in \mathcal{M}$}{
            $s \gets j$\;
        }
        \ElseIf{$k_j - 1 \in \mathcal{M}$ \textbf{ and } $k_j + 1 \in \mathcal{M}$}{
            $e \gets j$\;
        }
        \If{$s \neq \mathrm{NULL}$ \textbf{and} $e \neq \mathrm{NULL}$ \textbf{and} $k_e - k_s \geq \lceil c\sqrt{\log n}\rceil$}{
            append $s$ and $e$ to $I$\;
            $s \gets \text{NULL}, e \gets \text{NULL}$\;
        }
    }
    \KwRet{$I$}\;
\end{algorithm}


\begin{algorithm}[ht]
    \caption{Subroutine \texttt{Refined\_Local\_Search} of \fancyname}\label{alg:watermark-search}
    \KwInput{$\widetilde{\mathcal{M}} = \{ k_1, \dots, k_m\}$, $I = \{ i_1, \dots, i_{2\hat{K}} \}$, tuning parameters $\rho, C$}
    $M \gets \{ k_j : j \in I \}$\;
    Enumerate $M = \{ s_1, \dots, s_{2\hat{K}}\}$\;
    \For{$j \gets 1$ to $\hat{K}$}{
  ${D_j} \gets [\lfloor(s_{2j-1} - C \log n) b\rfloor \vee 1, \lfloor (s_{2j}  + C \log n)b \rfloor \wedge n ] $ . \;
}
$\Tilde{d} \gets \big(\sum_{j=1}^{\hat{K}}|D_j|\big)^{-1} \sum_{j=1}^{\hat{K}}\sum_{s \in D_j} (X_s - \mu_0).$ \;
\For{$j \gets 1$ to $\hat{K}$}{
  $L_j \gets [\lfloor(s_{2j-1} - C \log n) b \rfloor \vee 1, \lfloor (s_{2j-1} +C \log n)b\rfloor]$ \;
  $R_j \gets ( \lfloor(s_{2j}-C \log n)b\rfloor , \lfloor (s_{2j}  + C \log n)b \rfloor \wedge n].$ \;
  $\hat{I}_j \gets \argmin_{s \in L_j, t \in R_j} \sum_{k\in {D}_j \setminus [s, t] } (X_k -\mu_0 - \rho \tilde{d}).$ \;
}
\KwRet{\text{List of estimated watermarked intervals } $\hat{I}_j, j \in [\hat{K}]$.}
\end{algorithm}

\begin{algorithm}[ht]
    \caption{\fancyname}\label{alg:watermark}
    \KwInput{$(X_i)_{i \in [n]}$, block size $b$, threshold $\mathcal{Q}$, tuning parameters $c, C$ and $\rho$}
    $\widetilde{\mathcal{M}} \gets \texttt{Block\_Thresholding}((X_i)_{i \in [n]}, b, Q)$\;
    $I \gets \texttt{Merging}(\widetilde{\mathcal{M}}, c)$\;
    $L \gets \texttt{Refined\_Local\_Search}(\widetilde{\mathcal{M}}, I, \rho, C)$\;
    \KwRet{\text{List of estimated watermarked intervals } $L$}\;
\end{algorithm}

\section{Dealing with mixed-source texts}\label{se:mixed-source}

The assumption of knowledge of $\zeta_t$ can be too restrictive in most realistic scenarios where human edits are possible. In such cases, one assumes that the pseudo-random numbers $\zeta_t$ can also be reconstructed based on the available text and a \texttt{Key} with the help of a \textit{Hash function} $\mathcal{A}$:
\allowdisplaybreaks \begin{align}
    \zeta_t = \mathcal{A}(\omega_{(t-m):(t-1)}, \texttt{Key}). \label{eq:hash}
\end{align}
Suppose $\tilde{\omega}_1 \ldots \tilde{\omega}_n$ be a mixed-source text, with segments of uninterrupted watermarked texts punctuated by human-generated texts through substitution, insertion or deletion of LLM-generated texts. As a reference, we refer the readers to Procedure 1 of human edits in \cite{w3-0}. Note that it is impossible for any verifier to retrieve the exact pseudo-random numbers corresponding to each token in a mixed-source text. Nevertheless, with the knowledge of the hash function and \texttt{Key}, one can construct $\tilde{\zeta}_t= \mathcal{A}(\omega_{(t-m):(t-1)}, \texttt{Key})$. Once there is a stretch of uninterrupted watermarked interval with length at least $m\geq 1$, the pseudo-random numbers $\zeta_{t+m}, \zeta_{t+m+1}, \ldots$ can be reliably re-constructed through (\ref{eq:hash}) as the corresponding $\tilde{\zeta}$'s. On the other hand, if $\zeta_t$ is not the correct pseudo-random variable associated with $\omega_t$, then either 
\begin{enumerate}
    \item  $\tilde{\omega}_t$ is human-generated, in which case Working Hypothesis 2.2 of \cite{li2025statistical} applies to yield $\omega_t$ and $\tilde{\zeta}_t$ are independent conditional on $\IP_t$;
    \item $\tilde{\omega}_t$ is watermarked, which must mean if $\tilde{\omega}_t=\omega_{\tilde{t}}$ in the original watermarked text, then $\omega_{\tilde{t}} = S(P_{\tilde{t}}, \zeta_{\tilde{t}})$ for some true, unknown, pseudo-random number $\zeta_{\tilde{t}}$. In this case, we invoke the sensitive nature of the hash function to conclude that $\omega_t$ and $\tilde{\zeta}_t$ are independent.
\end{enumerate}
This argument appears in more detail in Section A.1 of \cite{w3-0}. In conclusion, the verifier can always obtain access to a sequence ${\zeta}_{m:n}$ corresponding to a given text $\omega_{1:n}$ such that 
(i) if $\omega_{(t-m):t}$ is NOT watermarked then $\omega_t$ and $\zeta_t$ are independent conditional on $\IP_t$; (ii), otherwise, $\zeta_t$ and $\omega_t$ may be intricately dependent on each other. This latter observation is crucial to our subsequent analysis and proposals, for it allows us to construct valid pivotal statistics. In light of the above discussion, we can be excused in making the Assumption \ref{ass:ind-of-pseudo}. 

Assumption \ref{ass:ind-of-pseudo} can be seen through the lens of constructing the $\tilde{\zeta}_t$ with $m=1$. We make this slightly simplistic assumption to avoid the unnecessary measure-theoretical niceties that might potentially cloud the novelty of our approach. Even with this assumption, proposing a computationally efficient algorithm and establishing its theoretical validity in a setting with multiple watermarked intervals, is an arguably non-trivial task in itself, and to the best of our knowledge, our paper is the first one to deal with this problem with full mathematical rigor.  Finally, for a general mixed-source text, we remark that the \texttt{WISER}\ algorithm can be trivially extended to the setting with general $m\in \N$ by padding an interval of length $m$ to the left of the watermarked segments located by \texttt{WISER}. In the following, we include a further discussion on mixed-source texts along with data misappropriation.

\subsection{Further discussion on Assumption \ref{ass:alt-mean}}\label{se:app-mixed}

A general way to deal with mixed-source texts has been proposed in \cite{cai2025statistical}. Therein, the authors allow for modifications to the watermarked tokens by allowing that $d_{\operatorname{TV}}(\omega_t, S(\IP_t, \zeta_t))>0$, as long as the distance is not too large. In light of this, Assumption \ref{ass:alt-mean} can be further appended by the following:

\textbf{(B)}: Let $\omega$ be the token generated by modifying $\omega':=S(P, \zeta)$, and let $h(Y), h(Y')$ be accordingly defined via $\omega, \omega'$. Then it holds that $$\inf_{\IP \in \mathcal{P}} |\IE_{1,\IP}[h(Y) - h(Y')]|\leq d \tau, \ \tau\in (0,1),$$ 
where $\mathcal{P}$ and $d$ are same as in Assumption \ref{ass:alt-mean}. 

Assumption \textbf{(B)} along with Assumption \ref{ass:alt-mean} ensure that even under potential modification, some degree of separation (characterized by $d(1-\tau)$) \upd{remains} between the means of the distribution of pivot statistics under null and under watermarking. It is conceivable that our algorithm is consistent even for this case, and the theoretical guarantees should follow more or less similarly; however, to keep the discussion focused, and to convey the key takeaways unhindered, we restrict ourselves to Assumption \ref{ass:alt-mean}.

\subsection{\upd{Performance results for mixed-source texts}}\label{se:mixed-source-results}

\upd{To assess the performance of \texttt{WISER} for robust watermark detection under human-edits, we consider the following setting. We consider a setting similar to~\ref{sim:setup1} for the Gumbel-max watermarking scheme, but with an additional postprocessing step that randomly selects tokens with probability $p$ and performs substitution with synonyms, or adds new tokens in those positions, or deletes the existing tokens. The human edit proportion $p$ is chosen at two different levels, \textit{Low} ($10\%$) and \textit{High} ($25\%$).  For a sequence of generated tokens $\omega_{1:t}$, we calculate the pseudo-random sequences as the skipgram hash function from the last $5$ tokens, i.e., $\zeta_t = \text{Hash}(w_{(t-5):t})$. The performance metrics for different detection methods under these human-edited texts are detailed in Table~\ref{tab:results-mixed}. For the different large language models, the methods and the error metric considered in Table \ref{tab:results-mixed}, we refer the readers to \S\ref{se:simu-1}. Increasing the edit proportion results in a drop in the performance metrics for both the models and all detection methods. The proposed WISER retains the best IOU metrics at all combinations, with the Waterseeker algorithm being a close second. However, the performance of WISER remains close to the Oracle estimator (see Algorithm~\ref{algo:kadane}), showcasing its resilience against edit attacks.}

\begin{table}[ht]
    \centering
    \resizebox{\textwidth}{!}{
        \begin{tabular}{lllrrrrrr}
            \toprule
            \textbf{Model Name} & \textbf{Regime} & \textbf{Method} & \textbf{IOU} & \textbf{Precision} & \textbf{Recall} & \textbf{F1} & \textbf{RI} & \textbf{MRI} \\
            \midrule
            
            \multirow{10}{*}{\makecell[l]{facebook-opt-125m \\ (50272)}} 
            & \multirow{5}{*}{Low}  
            & \ohl{\texttt{Oracle}}      & \ohl{0.773 (0.17-0.94)} & \ohl{0.985}          & \ohl{0.394} & \ohl{0.563} & \ohl{0.663 (0.42-0.83)}          & \ohl{0.635 (0.39-0.78)}          \\ 
            & & \texttt{WISER}       & \textbf{0.649 (0.00-0.99)}          & 0.955          & \textbf{0.376}          & \textbf{0.540}          & 0.719 (0.20-0.88)          & \textbf{0.661 (0.00-0.87)}          \\ 
            & & \texttt{Aligator}    & 0.000 (0.00-0.00)          & 0.038          & 0.011          & 0.017          & 0.283 (0.20-0.65)          & 0.080 (0.00-0.45)          \\ 
            & & \texttt{SeedBS}      & 0.017 (0.00-0.34)          & 0.043          & 0.010          & 0.016          & 0.288 (0.20-0.70)          & 0.089 (0.00-0.59)          \\ 
            & & \texttt{WaterSeeker} & 0.276 (0.14-0.45)          & \textbf{0.999} & 0.360          & 0.529          & \textbf{0.745} (0.47-0.89)          & 0.619 (0.30-0.79)          \\ 
            \cmidrule{2-9}
            & \multirow{5}{*}{High} 
            & \ohl{\texttt{Oracle}}      & \ohl{0.359 (0.16-0.59)}          & \ohl{0.990}          & \ohl{0.396}          & \ohl{0.566}          & \ohl{0.918 (0.72-0.98)}          & \ohl{0.794 (0.54-0.85)}          \\ 
            & & \texttt{WISER}       & \textbf{0.352 (0.00-0.61)}          & 0.955          & \textbf{0.355}          & \textbf{0.518}          & \textbf{0.870 (0.20-0.98)} & \textbf{0.745 (0.00-0.84)} \\ 
            & & \texttt{Aligator}    & 0.010 (0.00-0.08)          & 0.198          & 0.133          & 0.159          & 0.596 (0.20-0.88)          & 0.396 (0.00-0.69)          \\ 
            & & \texttt{SeedBS}      & 0.144 (0.00-0.73)          & 0.344          & 0.080          & 0.130          & 0.538 (0.20-0.86)          & 0.369 (0.00-0.74)          \\ 
            & & \texttt{WaterSeeker} & 0.247 (0.15-0.35)          & \textbf{0.980}          & 0.283          & 0.439          & 0.796 (0.57-0.97)          & 0.651 (0.42-0.83)          \\ 

            \midrule
            
            \multirow{10}{*}{\makecell[l]{princeton-nlp-\\Sheared-LLaMA-\\1-3B (32000)}} 
            & \multirow{5}{*}{Low} 
            & \ohl{\texttt{Oracle}}      & \ohl{0.894 (0.73-0.96)}          & \ohl{0.978}          & \ohl{0.391}          & \ohl{0.559}          & \ohl{0.581 (0.31-0.72)}          & \ohl{0.572 (0.31-0.71)}          \\ 
            & & \texttt{WISER}       & \textbf{0.825 (0.59-1.00)} & 0.995          & 0.298          & 0.459          & 0.453 (0.21-0.86)          & 0.444 (0.21-0.84)          \\ 
            & & \texttt{Aligator}    & 0.009 (0.00-0.05)          & 0.136          & 0.132          & 0.134          & 0.562 (0.20-0.86)          & 0.362 (0.00-0.67)          \\ 
            & & \texttt{SeedBS}      & 0.029 (0.00-0.41)          & 0.090          & 0.019          & 0.031          & 0.383 (0.20-0.76)          & 0.187 (0.00-0.65)          \\ 
            & & \texttt{WaterSeeker} & 0.282 (0.16-0.43)          & \textbf{1.000} & \textbf{0.357}          & \textbf{0.526}          & \textbf{0.748 (0.46-0.91)} & \textbf{0.623 (0.30-0.81)}          \\ 
            \cmidrule{2-9}
            
            & \multirow{5}{*}{High}  
            & \ohl{\texttt{Oracle}}      & \ohl{0.741 (0.37-0.93)}          & \ohl{0.993}          & \ohl{0.397} & \ohl{0.567} & \ohl{0.690 (0.45-0.85)}          & \ohl{0.657 (0.40-0.78)} \\ 
            & & \texttt{WISER}       & \textbf{0.494 (0.00-0.97)}         & 0.915          & 0.297          & 0.448          & 0.694 (0.20-0.87)          & \textbf{0.606 (0.00-0.85)}          \\ 
            & & \texttt{Aligator}    & 0.000 (0.00-0.00)          & 0.012          & 0.006          & 0.008          & 0.265 (0.20-0.65)          & 0.062 (0.00-0.45)          \\ 
            & & \texttt{SeedBS}      & 0.003 (0.00-0.00)          & 0.005          & 0.002          & 0.003          & 0.255 (0.20-0.66)          & 0.053 (0.00-0.45)          \\ 
            & & \texttt{WaterSeeker} & 0.267 (0.16-0.43)          & \textbf{0.992}          & \textbf{0.334}          & \textbf{0.500}          & \textbf{0.733 (0.50-0.90)}          & 0.603 (0.34-0.79)          \\ 
            \bottomrule
        \end{tabular}
    }
    \caption{\upd{Performance metrics for various watermarking detection algorithms under varying proportions of human edits for mixed source texts. The oracle refers to the algorithm solving~\eqref{eq:multiple-oracle-est} using Kadane's algorithm and the knowledge of the true number of watermarked patches; see Algorithm~\ref{algo:kadane}.}}
    \label{tab:results-mixed}
\end{table}

\section{\upd{Discussion on segmentation of multiple watermarked patches}}\label{se:kadane}

\upd{In the following, we first prove a general result about the consistency of an oracle estimator for watermarked patches. Given a set of disjoint intervals $I_1, \ldots, I_K\subset \{ 1, \ldots, n\}$, Recall $\mathcal{I}$:}
\begin{equation}
    \upd{
    \mathcal{I}:= \{ \{I_1, I_2, \ldots, I_K\}: I_j \subset \{ 1, \ldots, n\}, I_1 < I_2 < \ldots < I_K\}},    
    \label{eq:collection-of-intervals}
\end{equation}
\upd{where we assume that they are ordered left to right, i.e. $I_{k, R} +1 < I_{k+1, L}$ for all $k$, and write $I_1 < I_2 < \ldots < I_K$ to express this.}

\begin{theorem} \label{thm:multiple-oracle}
     \upd{Let $\{X_t\}_{t=1}^n:= \{h(Y_t)\}_{t=1}^n$ be the pivot statistics based on the given input text, and assume that $\{I_1, I_2, \ldots, I_K\} \in \mathcal{I}_K$ are the $K$ disjoint watermarked intervals with $K=O(1)$. Grant Assumption \ref{ass:alt-mean}. Denote}
    \begin{equation*}
        \upd{\varepsilon_t
        =} \begin{cases}
        \upd{X_t-\mu_0,} & \upd{t \notin I_0:= \cup_{j=1}^K I_j,}  \\
        \upd{X_t - \mu_t,} & \upd{\mu_t:=\IE_{1, \IP_t}[X_t], t\in I_0.}
    \end{cases}
    \end{equation*}
 \upd{Recall $I^\star$ from Assumption \ref{ass:min-sep}, and assume that $I^\star \gg \log n$ as $n\to \infty$. Suppose the class of distributions $\mathcal{P}$ is closed and compact, and there exists $\eta>0$ such that $\sup_{\IP\in \mathcal{P}}\IE_{1,\IP}[\exp(\eta|\varepsilon|)]<\infty$. Moreover, assume that $\min\{\Var_0(\varepsilon), \sup_{\IP}\Var_{1,P}(\varepsilon)\}>0$. Consider the oracle estimate in \eqref{eq:multiple-oracle-est} employing the knowledge of $K$,with $\rho$ and $\tilde{d}$ satisfying $d > 2 \rho \tilde{d}$. If there exists a constant $c>0$ such that ${d}\geq c$, then $\sum_{j\in [K]}|\hat{I}_j \Delta I_j|= O_{\IP}((\rho \tilde{d})^{-1}\big)$.}
\end{theorem}

\upd{Supplementing the theoretical validity of the oracle estimator in Theorem~\ref{thm:multiple-oracle}, we also indicate that the optimization problem of~\eqref{eq:multiple-oracle-est} can be efficiently solved by a dynamic programming approach motivated by Kadane's algorithm~\citep{kadane2023two}. The detailed steps of this algorithm are illustrated in Algorithm~\ref{algo:kadane}. We note that this oracle algorithm requires the knowledge of $K$, the true number of watermarked patches, which remains unidentified in practice, thereby limiting its applicability.}


\begin{algorithm}[htbp]
\caption{\upd{Oracle Estimator using Kadane's Algorithm}}\label{algo:kadane}
\KwIn{Pivot statistics $(X_i)_{i \in [n]}$, integer $K \ge 1$}
\lIf{$n = 0 \textbf{ or } K \le 0$}{\Return $0, \emptyset$}
$S \gets \text{cumulative\_sum}((X_i)_{i\in[n]})$\;
$P \gets $ matrix of size $(K+1) \times (n+1)$ initialized to $(-\infty, \emptyset)$\;
$P[0,0] \gets (0, \emptyset)$\;
\For{$i \gets 1$ \KwTo $n$}{
    $P[0,i] \gets (0, \emptyset)$\;
}
\For{$t \gets 1$ \KwTo $K$}{
    $P[t, 0] \gets (0, \emptyset)$\;
    $M \gets -\infty$, $L_{M} \gets \emptyset$,  $j^* \gets 1$\;
    \For{$i \gets 1$ \KwTo $n$}{
        $(S_{prev}, L_{prev}) \gets P[t-1, i-1]$\;
        $\omega \gets S_{prev} - S[i-1]$\;
        \If{$\omega > M$}{
            $M \gets \omega$, 
            $L_{M} \gets L_{prev}$,
            $j^* \gets i$\;
        }
        \eIf{$M = -\infty$}{
            $C_{extend} \gets (-\infty, \emptyset)$\;
        }{
            $S_{curr} \gets M + S[i]$\;
            $C_{extend} \gets (S_{curr}, L_{M} \cup \{(j^*, i)\})$\;
        }
        $C_{omit} \gets P[t, i-1]$\;
        \eIf{$C_{extend}.S \geq C_{omit}.S$}{
            $P[t,i] \gets C_{extend}$\; 
        }{
            $P[t,i] \gets C_{omit}$\;
        }
    }
    \Return{$P[K,n].L$}\;
}
\end{algorithm}

\upd{A conventional approach would be to execute Algorithm~\ref{algo:kadane} with a maximum possible value of $K$ (i.e., $K_{\mathrm{max}}$) and subsequently perform a selection procedure to prune some of the chosen intervals. This can be regarded as a search-then-filter approach in contrast to the filter-then-search approach adapted by \texttt{WISER}, allowing \texttt{WISER} to efficiently constrain the search space a priori to reduce false positive detections. Unfortunately, applying standard multiple testing corrections in the search-then-filter approach is statistically invalid due to inherent post-selection bias. Because the candidate intervals emitted by Kadane's algorithm are optimized to yield highly positive summands, they violate classical inference assumptions and are susceptible to not being rejected by the selection procedure thereafter. A snapshot of this phenomenon is illustrated in Tables~\ref{tab:results-gumbel-kadane}-\ref{tab:results-inverse-kadane} for Gumbel-max and inverse watermarking, respectively. In this exercise, we consider the scenario~\ref{sim:setup2} with the Gumbel-max watermarking scheme, and perform the selective inference by applying the e-BH procedure~\citep{wang2022eBH} for varying choices of $K_{\mathrm{max}}$. These heuristic configurations uniformly underperform relative to \texttt{WISER}. Choosing a higher value of $K_{\max}$ typically results in the selection of numerous spurious intervals alongside the genuinely watermarked regions. Because Kadane's algorithm selects only the intervals with higher sums, the e-BH procedure erroneously rejects the null hypotheses for these regions, failing to discard many of these false positives. While the substantial overlap between the true intervals and these detected intervals maintains high recall and IOU, this over-selection severely reduces the precision and F1-scores. Figure~\ref{fig:kadane-detections} visually shows an evidence for the setting~\ref{sim:setup1} under the Gumbel-max watermarking scheme. As $K_{\mathrm{max}}$ is continuously increased, the performances seem to plateau at a suboptimal level.}

\begin{figure}
    \centering
    \includegraphics[width=\linewidth]{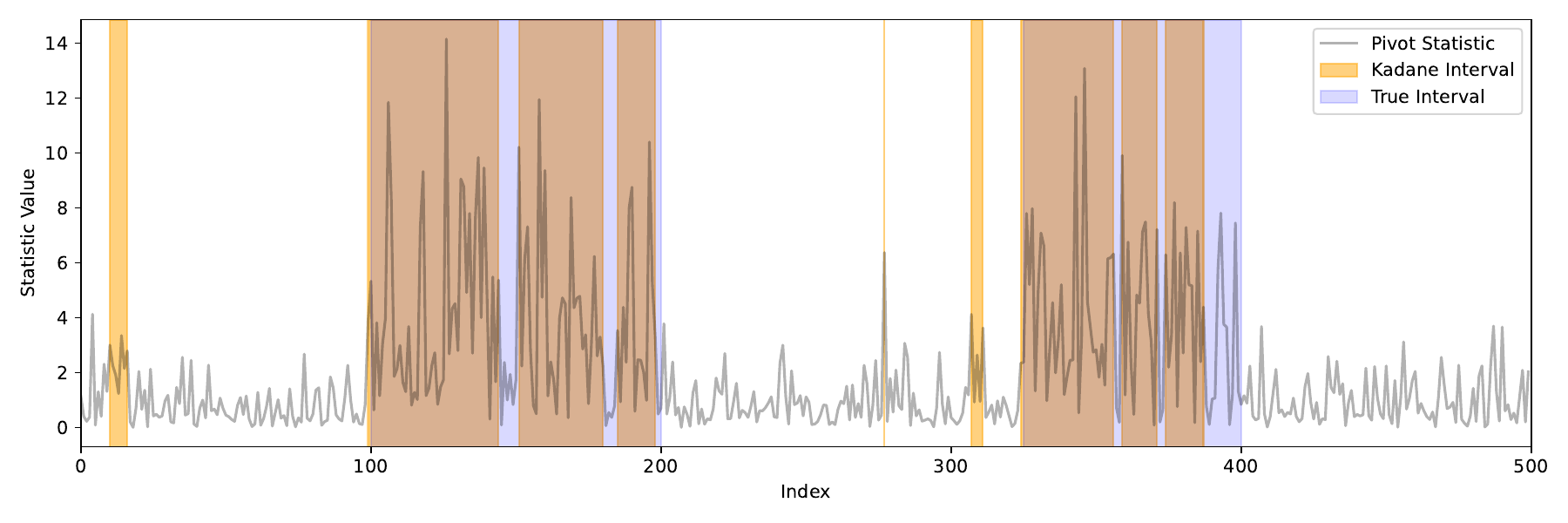}
    \caption{\upd{Detected intervals using the combination of Kadane's algorithm and e-BH thresholding procedure under setting~\ref{sim:setup1} with Gumbel-max watermarking scheme.}}
    \label{fig:kadane-detections}
\end{figure}


\begin{table}[ht]
\centering
\resizebox{\textwidth}{!}{
\begin{tabular}{llrrrrrr}
\toprule
\textbf{Model Name} & \textbf{Method} & \textbf{IOU} & \textbf{Precision} & \textbf{Recall} & \textbf{F1} & \textbf{RI} & \textbf{MRI} \\
\midrule
\multirow{6}{*}{
        \makecell[l]{facebook \\ opt-1.3b \\ (50272)}
    } & \texttt{WISER} & \textbf{0.951 (0.83-0.99)} & \textbf{0.897} & \textbf{0.992} & \textbf{0.942} & \textbf{0.988 (0.97-1.00)} & \textbf{0.988 (0.97-1.00)} \\
 & \texttt{Kadane + e-BH} ($K_{\mathrm{max}} = 6$) & 0.730 (0.58-0.99) & 0.852 & 0.852 & 0.852 & 0.906 (0.87-1.00) & 0.894 (0.85-1.00) \\
 & \texttt{Kadane + e-BH} ($K_{\mathrm{max}} = 8$) & 0.803 (0.61-0.98) & 0.659 & 0.922 & 0.768 & 0.943 (0.86-1.00) & 0.936 (0.85-1.00) \\
 & \texttt{Kadane + e-BH} ($K_{\mathrm{max}} = 10$) & 0.826 (0.60-0.98) & 0.524 & 0.943 & 0.674 & 0.950 (0.86-0.99) & 0.945 (0.84-0.99) \\
 & \texttt{Kadane + e-BH} ($K_{\mathrm{max}} = 15$) & 0.850 (0.62-0.96) & 0.346 & 0.970 & 0.511 & 0.952 (0.88-0.98) & 0.949 (0.87-0.98) \\
 & \texttt{Kadane + e-BH} ($K_{\mathrm{max}} = 20$) & 0.850 (0.65-0.95) & 0.256 & 0.972 & 0.405 & 0.946 (0.92-0.97) & 0.944 (0.92-0.97) \\
\midrule
\multirow{6}{*}{
        \makecell[l]{facebook \\ opt-125m \\ (50272)}
    } & \texttt{WISER} & \textbf{0.931 (0.54-0.99)} & \textbf{0.929} & \textbf{0.984} & \textbf{0.956} & \textbf{0.984 (0.86-1.00)} & \textbf{0.982 (0.85-1.00)} \\
 & \texttt{Kadane + e-BH} ($K_{\mathrm{max}} = 6$) & 0.743 (0.52-0.99) & 0.875 & 0.875 & 0.875 & 0.913 (0.86-1.00) & 0.903 (0.85-1.00) \\
 & \texttt{Kadane + e-BH} ($K_{\mathrm{max}} = 8$) & 0.802 (0.54-0.99) & 0.663 & 0.928 & 0.773 & 0.944 (0.86-1.00) & 0.937 (0.85-1.00) \\
 & \texttt{Kadane + e-BH} ($K_{\mathrm{max}} = 10$) & 0.825 (0.53-0.98) & 0.528 & 0.950 & 0.679 & 0.953 (0.86-0.99) & 0.948 (0.85-0.99) \\
 & \texttt{Kadane + e-BH} ($K_{\mathrm{max}} = 15$) & 0.842 (0.60-0.97) & 0.349 & 0.977 & 0.514 & 0.953 (0.87-0.98) & 0.949 (0.86-0.98) \\
 & \texttt{Kadane + e-BH} ($K_{\mathrm{max}} = 20$) & 0.842 (0.59-0.95) & 0.255 & 0.970 & 0.404 & 0.947 (0.91-0.97) & 0.944 (0.90-0.97) \\
\midrule
\multirow{6}{*}{
        \makecell[l]{google \\ gemma-3-270m \\ (262144)}
    } & \texttt{WISER} & \textbf{0.910 (0.00-0.99)} & \textbf{0.858} & 0.960 & \textbf{0.906} & \textbf{0.958 (0.11-1.00)} & \textbf{0.956 (0.08-1.00)} \\
 & \texttt{\texttt{Kadane + e-BH} ($K_{\mathrm{max}} = 6$)} & 0.718 (0.32-0.99) & 0.852 & 0.852 & 0.852 & 0.909 (0.85-1.00) & 0.898 (0.84-1.00) \\
 & \texttt{\texttt{Kadane + e-BH} ($K_{\mathrm{max}} = 8$)} & 0.775 (0.36-0.99) & 0.655 & 0.917 & 0.764 & 0.940 (0.86-0.99) & 0.932 (0.85-0.99) \\
 & \texttt{\texttt{Kadane + e-BH} ($K_{\mathrm{max}} = 10$)} & 0.789 (0.38-0.98) & 0.517 & 0.931 & 0.665 & 0.946 (0.86-0.99) & 0.939 (0.84-0.99) \\
 & \texttt{\texttt{Kadane + e-BH} ($K_{\mathrm{max}} = 15$)} & 0.809 (0.36-0.96) & 0.345 & 0.967 & 0.509 & 0.948 (0.88-0.98) & 0.943 (0.87-0.98) \\
 & \texttt{\texttt{Kadane + e-BH} ($K_{\mathrm{max}} = 20$)} & 0.831 (0.44-0.94) & 0.258 & \textbf{0.980} & 0.408 & 0.946 (0.92-0.96) & 0.944 (0.91-0.96) \\
\midrule
\multirow{6}{*}{
        \makecell[l]{meta-llama \\ Meta-Llama-3-8B \\ (128256)}
    } & \texttt{WISER} & \textbf{0.961 (0.88-0.99)} & \textbf{0.983} & \textbf{1.000} & \textbf{0.991} & \textbf{0.997 (0.99-1.00)} & \textbf{0.996 (0.99-1.00)} \\
 & \texttt{\texttt{Kadane + e-BH} ($K_{\mathrm{max}} = 6$)} & 0.683 (0.58-0.99) & 0.815 & 0.815 & 0.815 & 0.882 (0.87-1.00) & 0.868 (0.85-1.00) \\
 & \texttt{\texttt{Kadane + e-BH} ($K_{\mathrm{max}} = 8$)} & 0.698 (0.57-0.98) & 0.594 & 0.832 & 0.693 & 0.891 (0.86-1.00) & 0.878 (0.85-1.00) \\
 & \texttt{\texttt{Kadane + e-BH} ($K_{\mathrm{max}} = 10$)} & 0.711 (0.57-0.98) & 0.471 & 0.847 & 0.605 & 0.899 (0.86-1.00) & 0.888 (0.85-1.00) \\
 & \texttt{\texttt{Kadane + e-BH} ($K_{\mathrm{max}} = 15$)} & 0.731 (0.57-0.98) & 0.310 & 0.868 & 0.457 & 0.914 (0.86-0.99) & 0.904 (0.84-0.99) \\
 & \texttt{\texttt{Kadane + e-BH} ($K_{\mathrm{max}} = 20$)} & 0.860 (0.58-0.97) & 0.255 & 0.969 & 0.404 & 0.952 (0.86-0.98) & 0.949 (0.85-0.98) \\
\midrule
\multirow{6}{*}{
        \makecell[l]{mistralai \\ Mistral-7B-v0.1 \\ (32000)}
    } &  \texttt{WISER} & \textbf{0.913 (0.63-0.99)} & 0.836 & \textbf{0.994} & \textbf{0.908} & \textbf{0.988 (0.95-1.00)} & \textbf{0.987 (0.94-1.00)} \\
 & \texttt{\texttt{Kadane + e-BH} ($K_{\mathrm{max}} = 6$)} & 0.698 (0.51-0.98) & \textbf{0.856} & 0.856 & 0.856 & 0.899 (0.86-1.00) & 0.887 (0.85-1.00) \\
 & \texttt{\texttt{Kadane + e-BH} ($K_{\mathrm{max}} = 8$)} & 0.725 (0.52-0.98) & 0.639 & 0.895 & 0.746 & 0.921 (0.86-0.99) & 0.912 (0.84-0.99) \\
 & \texttt{\texttt{Kadane + e-BH} ($K_{\mathrm{max}} = 10$)} & 0.754 (0.52-0.97) & 0.513 & 0.924 & 0.660 & 0.936 (0.86-0.99) & 0.929 (0.84-0.99) \\
 & \texttt{\texttt{Kadane + e-BH} ($K_{\mathrm{max}} = 15$)} & 0.792 (0.51-0.96) & 0.341 & 0.954 & 0.502 & 0.949 (0.86-0.98) & 0.944 (0.84-0.98) \\
 & \texttt{\texttt{Kadane + e-BH} ($K_{\mathrm{max}} = 20$)} & 0.815 (0.55-0.94) & 0.258 & 0.979 & 0.408 & 0.951 (0.93-0.97) & 0.948 (0.92-0.97) \\
\midrule
\multirow{6}{*}{
        \makecell[l]{princeton-nlp \\ Sheared-LLaMA-1.3B \\ (32000)}
    } & \texttt{WISER} & \textbf{0.947 (0.77-0.99)} & \textbf{0.894} & \textbf{0.988} & \textbf{0.939} & \textbf{0.991 (0.95-1.00)} & \textbf{0.990 (0.95-1.00)} \\
 & \texttt{\texttt{Kadane + e-BH} ($K_{\mathrm{max}} = 6$)} & 0.721 (0.49-0.99) & 0.846 & 0.846 & 0.846 & 0.904 (0.86-1.00) & 0.892 (0.84-1.00) \\
 & \texttt{\texttt{Kadane + e-BH} ($K_{\mathrm{max}} = 8$)} & 0.787 (0.59-0.98) & 0.646 & 0.904 & 0.753 & 0.938 (0.86-1.00) & 0.930 (0.85-1.00) \\
 & \texttt{\texttt{Kadane + e-BH} ($K_{\mathrm{max}} = 10$)} & 0.809 (0.56-0.97) & 0.516 & 0.929 & 0.664 & 0.948 (0.86-0.99) & 0.942 (0.85-0.99) \\
 & \texttt{\texttt{Kadane + e-BH} ($K_{\mathrm{max}} = 15$)} & 0.826 (0.61-0.96) & 0.344 & 0.962 & 0.506 & 0.953 (0.90-0.98) & 0.948 (0.88-0.98) \\
 & \texttt{\texttt{Kadane + e-BH} ($K_{\mathrm{max}} = 20$)} & 0.841 (0.61-0.94) & 0.256 & 0.973 & 0.405 & 0.947 (0.92-0.96) & 0.944 (0.91-0.96) \\
\midrule
\bottomrule
\end{tabular}}
\caption{\upd{Comparison of performances between \texttt{WISER} and Kadane's algorithm followed by an e-BH procedure with different heuristic choices of maximum number of intervals $K_{\mathrm{max}}$, for the scenario~\ref{sim:setup2} under Gumbel-max watermarking scheme. Note that, in this table, true $K=5$. Performance of the oracle is provided in Table~\ref{tab:results-gumbel-n2500}.}}
\label{tab:results-gumbel-kadane}
\end{table}

\begin{table}[ht]
\centering
\resizebox{\textwidth}{!}{
\begin{tabular}{llrrrrrr}
\toprule
\textbf{Model Name} & \textbf{Method} & \textbf{IOU} & \textbf{Precision} & \textbf{Recall} & \textbf{F1} & \textbf{RI} & \textbf{MRI} \\
\midrule
\multirow{6}{*}{
        \makecell[l]{facebook \\ opt-1.3b \\ (50272)}
    } & \texttt{WISER} & \textbf{0.882 (0.38-0.97)} & \textbf{0.625} & \textbf{0.970} & \textbf{0.760} & \textbf{0.963 (0.80-0.99)} & \textbf{0.961 (0.79-0.99)} \\
 & \texttt{\texttt{Kadane + e-BH} ($K_{\mathrm{max}} = 6$)} & 0.331 (0.16-0.46) & 0.523 & 0.523 & 0.523 & 0.789 (0.59-0.87) & 0.732 (0.51-0.82) \\
 & \texttt{\texttt{Kadane + e-BH} ($K_{\mathrm{max}} = 8$)} & 0.352 (0.23-0.48) & 0.458 & 0.641 & 0.534 & 0.841 (0.73-0.91) & 0.792 (0.66-0.86) \\
 & \texttt{\texttt{Kadane + e-BH} ($K_{\mathrm{max}} = 10$)} & 0.359 (0.25-0.49) & 0.394 & 0.709 & 0.506 & 0.866 (0.78-0.91) & 0.824 (0.74-0.87) \\
 & \texttt{\texttt{Kadane + e-BH} ($K_{\mathrm{max}} = 15$)} & 0.370 (0.31-0.51) & 0.287 & 0.803 & 0.423 & 0.895 (0.84-0.92) & 0.865 (0.80-0.90) \\
 & \texttt{\texttt{Kadane + e-BH} ($K_{\mathrm{max}} = 20$)} & 0.403 (0.35-0.50) & 0.212 & 0.804 & 0.335 & 0.907 (0.89-0.92) & 0.886 (0.86-0.91) \\
\midrule
\multirow{6}{*}{
        \makecell[l]{facebook \\ opt-125m \\ (50272)}
    } & \texttt{WISER} & \textbf{0.878 (0.34-0.97)} & \textbf{0.726} & \textbf{0.965} & \textbf{0.828} & \textbf{0.960 (0.67-0.99)} & \textbf{0.958 (0.65-0.99)} \\
 & \texttt{\texttt{Kadane + e-BH} ($K_{\mathrm{max}} = 6$)} & 0.339 (0.16-0.47) & 0.529 & 0.529 & 0.529 & 0.798 (0.66-0.88) & 0.740 (0.59-0.82) \\
 & \texttt{\texttt{Kadane + e-BH} ($K_{\mathrm{max}} = 8$)} & 0.353 (0.24-0.48) & 0.452 & 0.633 & 0.527 & 0.841 (0.72-0.90) & 0.793 (0.68-0.85) \\
 & \texttt{\texttt{Kadane + e-BH} ($K_{\mathrm{max}} = 10$)} & 0.361 (0.25-0.49) & 0.395 & 0.711 & 0.508 & 0.869 (0.80-0.91) & 0.828 (0.74-0.87) \\
 & \texttt{\texttt{Kadane + e-BH} ($K_{\mathrm{max}} = 15$)} & 0.364 (0.30-0.51) & 0.281 & 0.787 & 0.414 & 0.896 (0.86-0.92) & 0.864 (0.81-0.89) \\
 & \texttt{\texttt{Kadane + e-BH} ($K_{\mathrm{max}} = 20$)} & 0.409 (0.33-0.51) & 0.214 & 0.812 & 0.338 & 0.908 (0.89-0.92) & 0.888 (0.86-0.91) \\
\midrule
\multirow{6}{*}{
        \makecell[l]{google \\ gemma-3-270m \\ (262144)}
    } & \texttt{WISER} & \textbf{0.682 (0.00-0.98)} & 0.580 & 0.860 & 0.693 & 0.870 (0.11-0.99) & 0.859 (0.08-0.99) \\
 & \texttt{\texttt{Kadane + e-BH} ($K_{\mathrm{max}} = 6$)} & 0.597 (0.05-0.93) & \textbf{0.819} & 0.819 & \textbf{0.819} & 0.896 (0.82-0.99) & 0.881 (0.81-0.99) \\
 & \texttt{\texttt{Kadane + e-BH} ($K_{\mathrm{max}} = 8$)} & 0.604 (0.06-0.92) & 0.612 & 0.857 & 0.714 & 0.918 (0.86-0.98) & 0.904 (0.84-0.98) \\
 & \texttt{\texttt{Kadane + e-BH} ($K_{\mathrm{max}} = 10$)} & 0.610 (0.06-0.90) & 0.493 & 0.888 & 0.634 & 0.931 (0.87-0.98) & 0.918 (0.84-0.97) \\
 & \texttt{\texttt{Kadane + e-BH} ($K_{\mathrm{max}} = 15$)} & 0.613 (0.07-0.86) & 0.330 & \textbf{0.923} & 0.486 & \textbf{0.935 (0.88-0.96)} & \textbf{0.924 (0.85-0.96)} \\
 & \texttt{\texttt{Kadane + e-BH} ($K_{\mathrm{max}} = 20$)} & 0.620 (0.09-0.84) & 0.237 & 0.899 & 0.375 & 0.930 (0.91-0.95) & 0.922 (0.88-0.95) \\
\midrule
\multirow{6}{*}{
        \makecell[l]{meta-llama \\ Meta-Llama-3-8B \\ (128256)}
    } &  \texttt{WISER} & \textbf{0.916 (0.39-0.99)} & \textbf{0.911} & \textbf{0.970} & \textbf{0.940} & \textbf{0.975 (0.69-1.00)} & \textbf{0.973 (0.67-1.00)} \\
 & \texttt{\texttt{Kadane + e-BH} ($K_{\mathrm{max}} = 6$)} & 0.675 (0.22-0.99) & 0.817 & 0.817 & 0.817 & 0.901 (0.82-1.00) & 0.887 (0.80-1.00) \\
 & \texttt{\texttt{Kadane + e-BH} ($K_{\mathrm{max}} = 8$)} & 0.706 (0.34-0.97) & 0.607 & 0.850 & 0.708 & 0.938 (0.86-0.99) & 0.926 (0.84-0.99) \\
 & \texttt{\texttt{Kadane + e-BH} ($K_{\mathrm{max}} = 10$)} & 0.723 (0.34-0.96) & 0.486 & 0.874 & 0.624 & 0.947 (0.87-0.99) & 0.937 (0.85-0.99) \\
 & \texttt{\texttt{Kadane + e-BH} ($K_{\mathrm{max}} = 15$)} & 0.724 (0.33-0.93) & 0.317 & 0.889 & 0.468 & 0.946 (0.91-0.97) & 0.937 (0.89-0.97) \\
 & \texttt{\texttt{Kadane + e-BH} ($K_{\mathrm{max}} = 20$)} & 0.758 (0.38-0.91) & 0.236 & 0.895 & 0.373 & 0.940 (0.91-0.96) & 0.936 (0.89-0.96) \\
\midrule
\multirow{6}{*}{
        \makecell[l]{mistralai \\ Mistral-7B-v0.1 \\ (32000)}
    } & \texttt{WISER} & \textbf{0.908 (0.76-0.97)} & 0.820 & \textbf{0.976} & \textbf{0.891} & \textbf{0.980 (0.96-1.00)} & \textbf{0.979 (0.95-1.00)} \\
 & \texttt{\texttt{Kadane + e-BH} ($K_{\mathrm{max}} = 6$)} & 0.671 (0.47-0.92) & \textbf{0.859} & 0.859 & 0.859 & 0.921 (0.86-0.99) & 0.909 (0.84-0.99) \\
 & \texttt{\texttt{Kadane + e-BH} ($K_{\mathrm{max}} = 8$)} & 0.656 (0.48-0.88) & 0.625 & 0.875 & 0.729 & 0.935 (0.86-0.98) & 0.924 (0.84-0.98) \\
 & \texttt{\texttt{Kadane + e-BH} ($K_{\mathrm{max}} = 10$)} & 0.646 (0.46-0.86) & 0.491 & 0.883 & 0.631 & 0.936 (0.90-0.97) & 0.926 (0.88-0.97) \\
 & \texttt{\texttt{Kadane + e-BH} ($K_{\mathrm{max}} = 15$)} & 0.630 (0.45-0.81) & 0.321 & 0.899 & 0.473 & 0.932 (0.90-0.95) & 0.923 (0.89-0.95) \\
 & \texttt{\texttt{Kadane + e-BH} ($K_{\mathrm{max}} = 20$)} & 0.606 (0.47-0.76) & 0.221 & 0.840 & 0.350 & 0.925 (0.91-0.94) & 0.918 (0.90-0.94) \\
\midrule
\multirow{6}{*}{
        \makecell[l]{princeton-nlp \\ Sheared-LLaMA-1.3B \\ (32000)}
    } & \texttt{WISER} & \textbf{0.844 (0.54-0.96)} & 0.482 & \textbf{0.956} & 0.641 & \textbf{0.961 (0.87-0.99)} & \textbf{0.959 (0.86-0.99)} \\
 & \texttt{\texttt{Kadane + e-BH} ($K_{\mathrm{max}} = 6$)} & 0.560 (0.40-0.70) & \textbf{0.781} & 0.781 & \textbf{0.781} & 0.886 (0.84-0.94) & 0.868 (0.82-0.92) \\
 & \texttt{\texttt{Kadane + e-BH} ($K_{\mathrm{max}} = 8$)} & 0.538 (0.39-0.75) & 0.571 & 0.800 & 0.667 & 0.908 (0.86-0.96) & 0.891 (0.84-0.95) \\
 & \texttt{\texttt{Kadane + e-BH} ($K_{\mathrm{max}} = 10$)} & 0.531 (0.40-0.76) & 0.453 & 0.815 & 0.582 & 0.917 (0.87-0.96) & 0.900 (0.85-0.96) \\
 & \texttt{\texttt{Kadane + e-BH} ($K_{\mathrm{max}} = 15$)} & 0.513 (0.37-0.74) & 0.297 & 0.832 & 0.438 & 0.921 (0.90-0.94) & 0.905 (0.88-0.94) \\
 & \texttt{\texttt{Kadane + e-BH} ($K_{\mathrm{max}} = 20$)} & 0.570 (0.43-0.69) & 0.219 & 0.832 & 0.347 & 0.924 (0.91-0.94) & 0.916 (0.90-0.94) \\
\midrule
\bottomrule
\end{tabular}}
\caption{\upd{Comparison of performances between \texttt{WISER} and Kadane's algorithm followed by an e-BH procedure with different heuristic choices of maximum number of intervals $K_{\mathrm{max}}$, for the scenario~\ref{sim:setup2} under inverse watermarking scheme. Note that, in this table, true $K=5$. Performance of the oracle is provided in Table~\ref{tab:results-inverse-n2500}.}}
\label{tab:results-inverse-kadane}
\end{table}

\section{Extended numerical experiments}\label{se:simu}

In this section, we provide additional numerical experiments complementing those in \S\ref{se:shortened-simu}. In \S\ref{se:simu-1}, we compare the accuracy of \texttt{WISER}\ with other competitive methods in the literature, on various benchmark datasets on myriad standard watermarking schemes. Moving on to \S\ref{se:effect-wm-intensity}, we investigate the effect of watermark intensity as well as the watermarked length on the performance of the algorithms. Finally, in \S\ref{se:ablation-study}, we provide some ablation studies corresponding to the hyperparameters in \texttt{WISER}. 

\subsection{Comparative performance of \texttt{WISER}}\label{se:simu-1}

\subsubsection{Experimental Settings}\label{appendix:sim-settings}

\upd{To comprehensively understand the behavior of the \texttt{WISER} algorithm under different conditions, we consider three different experimental scenarios, as enumerated below.
\begin{enumerate}[label=(S\arabic*)]
    \item\label{sim:setup1} Short length output tokens ($n = 500$) with $K = 2$, i.e., two watermarked patches on $100$-$200$ and $325$-$400$.
    \item\label{sim:setup2} Long length outputs ($n = 2500$ tokens) with $K = 5$ watermarked patches, with successively increasing lengths and increasing gaps, i.e., on $100$-$200$, $350$-$500$, $700$-$900$, $1150$-$1400$, and $1700$-$2000$.
    \item\label{sim:setup3} Moderately long outputs ($n = 1000$) with $K = 5$ watermarked patches, but tightly clustered in the middle with each watermarked region spanning over $60$ tokens and a gap of $30$ tokens between consecutive intervals.
\end{enumerate}}

\upd{From \S\ref{se:shortened-simu-1}, we recall the SOTA benchmark algorithms for comparison. For each of the experiments, we implement \texttt{WISER}\ with block size equal to $b=[\sqrt{n}]$, $\rho = 0.5$, $\alpha = 0.05$ and $\gamma = 0.1$. Before we provide detailed comparison studies, we elaborate on the performance metrics used.}


\subsubsection{Performance metric}\label{se:performance-metric}

To ensure consistency with the prior works, we primarily treat the intersection-over-union (IOU) as a performance measure. Let, $\bb{I} := (I_1, \dots, I_K)$ denotes the true watermarked intervals and $\bbhat{I} := (\widehat{I}_1, \dots, \widehat{I}_{\hat{K}})$ be the estimated watermarked segments. Then, the intersection-over-union metric is given by
\begin{equation*}
    \text{IOU}(\bb{I}, \bbhat{I}) = \dfrac{\vert (\cup_{i=1}^K I_i) \cap (\cup_{j=1}^{\hat{K}} \hat{I}_i)\vert }{\vert (\cup_{i=1}^K I_i) \cup (\cup_{j=1}^{\hat{K}} \hat{I}_i)\vert}.
\end{equation*}
\noindent Owing to Theorem~\ref{thm:single-watermark}, it is obvious that the IOU measure is expected to be close to $1$ for the \texttt{WISER}\ method. Following the definition of~\cite{pan2025waterseeker}, we also compute the precision, recall, and F1-score based on whether any of the estimated intervals have a nonempty intersection with any of the true intervals, i.e.,
\begin{align*}
    \text{Precision } = \frac{\vert \{ i: 1 \leq i \leq {\hat{K}}, \hat{I}_i \cap (\cup_{j=1}^K I_j) \neq \phi  \} \vert}{{\hat{K}}},
    \text{ Recall} = \frac{\vert \{ i: 1 \leq i \leq {\hat{K}}, \hat{I}_i \cap (\cup_{j=1}^K I_j) \neq \phi  \} \vert}{K}.
\end{align*}
\textbf{Rand Index and asymmetry of the watermark segmentation.}

In addition to these metrics, the Rand Index (RI) is also usually used to measure coherence between the estimated and true watermarked segments, using the algorithm illustrated in~\cite{prates2021more}. For the standard definition of Rand Index, see Equation (2) of~\cite{prates2021more}. However, the Rand Index may depict a wrong picture of the performance of a watermarked segment identification algorithm.  Although watermark segmentation closely resembles epidemic change-point detection, a crucial difference arises in algorithm evaluation. Before proceeding, we briefly deliberate on these issues.

Standard change-point problems are symmetric; under model (\ref{eq:usual-epidemic}), the edge cases ${p=1,q=n}$ and ${p=q}$ are equivalent. On the other hand, watermarking problems exhibit asymmetry; the edge cases (i) ``the entire sequence is unwatermarked'' and (ii) ``the entire sequence is watermarked'',  differ due to irregular means of the pivot statistics under watermarking. Rand Index (RI) - despite being used in watermark segmentation \citep{li2024segmenting, pan2025waterseeker} - fails to capture this distinction. 

As an illustration, consider the situation where most of the tokens (say $90\%$) are watermarked, while the watermark detection algorithm fails to detect any watermarked segment. While the performance of such an algorithm should reflect poorly, the standard Rand Index fails to capture this due to the exchangeability of the watermarked segment and the non-watermarked segment: any pair of indices $(i,j)$ that is truly watermarked trivially is also part of the estimated non-watermarked segment and considered as a concordant pair.

To circumvent these limitations described there, we consider a Modified Rand Index (MRI) given as 
\begin{multline*}
    \text{MRI}(\bb{I}, \bbhat{I}) := \text{RI}(\bb{I}, \bbhat{I}) \\
    - \dfrac{ \sum_{i\neq j} \left( \sum_{k=1}^K \mathbf{1}\{ \{ i, j \} \subseteq I_k \cap (\cup_{l=1}^{\hat{K}} \hat{I}_l)^c \} + \sum_{l=1}^{\hat{K}} \mathbf{1}\{ \{ i, j \} \subseteq \hat{I}_l \cap (\cup_{k=1}^{K} I_k)^c \} \right) }{\binom{n}{2}},
\end{multline*}
\noindent where $\mathbf{1}\{ \cdot \}$ is the indicator function, and $n$ is the number of tokens. The MRI simply adjusts the RI by restricting its exchangeability only within each of the watermarked or non-watermarked intervals, but not in between. Intuitively, the MRI removes the specific pairs of indices $(i, j)$ from the calculation of RI for which both the indices $i$ and $j$ lie either in a true watermarked interval but are estimated to be in the non-watermarked region, or are estimated to be in a watermarked interval but actually lie in a non-watermarked region.

\subsubsection{Experimental results and explanation}

\upd{The comparison results are summarized in Tables~\ref{tab:results-gumbel}-\ref{tab:results-pf-n1000}, corresponding to each of the watermarking schemes and experimental settings considered. We also include the performance of the Oracle algorithm~\ref{algo:kadane} by passing the knowledge of the true number of watermarked patches, $K$, to the estimator \eqref{eq:multiple-oracle-est}. The corresponding rows, designated as \texttt{Oracle} and highlighted in \ohl{pink}, serve as a baseline to all the other methods, including \fancyname, which do not require the knowledge of $K$.}

Across all settings, \texttt{WISER}\ consistently delivers the strongest performance across every model and metric. In the Gumbel case, it achieves near-perfect results with IOU scores above 0.90, precision of 1.0, and recall above 0.98 across both small and large models. Competing methods like \texttt{Aligator} and \texttt{SeedBS-NOT} often fail to balance recall and precision, either collapsing to very low precision (\texttt{Aligator}) or producing weaker recall (\texttt{SeedBS-NOT}), while \texttt{Waterseeker} attains moderate balance but still lags well behind \texttt{WISER}.

The trend is even more pronounced in the cases of Inverse and Red-Green setups, where the pivot statistics remain uniformly bounded. In these cases, \texttt{Aligator} fail to detect any watermarked intervals, while both \texttt{SeedBS-NOT} and \texttt{Waterseeker} suffer a significant decline in performance. In contrast, \texttt{WISER}\ maintains F1-scores in the range of $0.95$ - $0.99$ with stable IOU values across model sizes, showing robustness to different architectures and vocabulary sizes. \texttt{Waterseeker} provides the next best alternative, but with noticeable drops in IOU and F1, especially for larger models. \upd{All methods deteriorate under the more stringent scenario~\ref{sim:setup3}, owing to the shorter and more concentrated watermarked patches. Nevertheless, \texttt{WISER} continues to achieve a reasonably strong IOU.} These findings clearly demonstrate that \texttt{WISER}\ not only generalizes across watermarking schemes but also offers substantial gains in both segmentation accuracy and reliability, marking a clear benefit over existing baselines.



\begin{table}[ht]
\centering
\resizebox{\textwidth}{!}{
\begin{tabular}{llrrrrrr}
    \toprule
    \textbf{Model Name} & \textbf{Method} & \textbf{IOU} & \textbf{Precision} & \textbf{Recall} & \textbf{F1} & \textbf{RI} & \textbf{MRI} \\
    \midrule
    \multirow{5}{*}{
            \makecell[l]{facebook \\ opt-1.3b \\ (50272)}
        } & \ohl{\texttt{Oracle}} & \ohl{0.944 (0.71-0.99)} & \ohl{0.998} & \ohl{0.998} & \ohl{0.998} & \ohl{0.985 (0.92-1.00)} & \ohl{0.984 (0.92-1.00)} \\
     & \texttt{WISER} & \textbf{0.921 (0.72-0.99)} & \textbf{1.000} & \textbf{0.995} & \textbf{0.997} & \textbf{0.976 (0.92-1.00)} & \textbf{0.974 (0.90-1.00)} \\
     & \texttt{Aligator} & 0.514 (0.03-0.88) & 0.245 & 0.930 & 0.388 & 0.894 (0.67-0.97) & 0.874 (0.63-0.97) \\
     & \texttt{SeedBS} & 0.327 (0.00-0.86) & 0.623 & 0.445 & 0.519 & 0.761 (0.20-0.96) & 0.710 (0.14-0.96) \\
     & \texttt{WaterSeeker} & 0.657 (0.37-0.92) & 1.000 & 0.807 & 0.893 & 0.860 (0.70-0.98) & 0.846 (0.66-0.98) \\
    \midrule
    \multirow{5}{*}{
            \makecell[l]{facebook \\ opt-125m \\ (50272)}
        } & \ohl{\texttt{Oracle}} & \ohl{0.951 (0.83-0.99)} & \ohl{1.000} & \ohl{1.000} & \ohl{1.000} & \ohl{0.987 (0.96-1.00)} & \ohl{0.987 (0.95-1.00)} \\
     & \texttt{WISER} & \textbf{0.934 (0.81-0.99)} & \textbf{1.000} & \textbf{0.995} & \textbf{0.997} & \textbf{0.980 (0.94-1.00)} & \textbf{0.979 (0.94-1.00)} \\
\end{tabular}}
    \caption{\upd{Performance metrics for Gumbel-max watermarking in experimental setting~\ref{sim:setup1}.}}
    \label{tab:results-gumbel}
\end{table}


\begin{table}[ht]
    \centering
    \resizebox{\textwidth}{!}{
    \begin{tabular}{llrrrrrr}
    \toprule
    \textbf{Model Name} & \textbf{Method} & \textbf{IOU} & \textbf{Precision} & \textbf{Recall} & \textbf{F1} & \textbf{RI} & \textbf{MRI} \\
    \midrule
    \multirow{5}{*}{
            \makecell[l]{facebook \\ opt-1.3b \\ (50272)}
        } & \ohl{\texttt{Oracle}} & \ohl{0.883 (0.50-0.98)} & \ohl{0.993} & \ohl{0.993} & \ohl{0.993} & \ohl{0.967 (0.85-1.00)} & \ohl{0.964 (0.82-1.00)} \\
     & \texttt{WISER} & \textbf{0.784 (0.38-0.97)} & 0.980 & \textbf{0.927} & \textbf{0.953} & \textbf{0.912 (0.70-0.99)} & \textbf{0.904 (0.66-0.99)} \\
     & \texttt{Aligator} & 0.000 (0.00-0.00) & 0.000 & 0.000 & 0.000 & 0.203 (0.20-0.20) & 0.141 (0.14-0.14) \\
     & \texttt{SeedBS} & 0.166 (0.00-0.82) & 0.274 & 0.210 & 0.238 & 0.494 (0.20-0.95) & 0.438 (0.14-0.94) \\
     & \texttt{WaterSeeker} & 0.555 (0.29-0.91) & \textbf{0.985} & 0.698 & 0.817 & 0.802 (0.65-0.97) & 0.781 (0.62-0.97) \\
    \midrule
    \multirow{5}{*}{
            \makecell[l]{facebook \\ opt-125m \\ (50272)}
        } & \ohl{\texttt{Oracle}} & \ohl{0.919 (0.67-0.99)} & \ohl{0.998} & \ohl{0.998} & \ohl{0.998} & \ohl{0.978 (0.92-1.00)} & \ohl{0.977 (0.91-1.00)} \\
     & \texttt{WISER} & \textbf{0.818 (0.52-0.97)} & \textbf{0.995} & \textbf{0.980} & \textbf{0.987} & \textbf{0.937 (0.76-0.99)} & \textbf{0.931 (0.74-0.99)} \\
\bottomrule
    \end{tabular}}
        \caption{\upd{Performance metrics for Inverse watermarking in experimental setting~\ref{sim:setup1}.}}
    \label{tab:results-inverse}
\end{table}


\begin{table}[ht]
\centering
\resizebox{\textwidth}{!}{
\begin{tabular}{llrrrrrr}
\toprule
\textbf{Model Name} & \textbf{Method} & \textbf{IOU} & \textbf{Precision} & \textbf{Recall} & \textbf{F1} & \textbf{RI} & \textbf{MRI} \\
\midrule
\multirow{5}{*}{
    \makecell[l]{facebook \\ opt-1.3b \\ (50272)}
    } & \ohl{\texttt{Oracle}} & \ohl{0.861 (0.66-0.97)} & \ohl{1.000} & \ohl{1.000} & \ohl{1.000} & \ohl{0.961 (0.90-0.99)} & \ohl{0.959 (0.89-0.99)} \\
     & \texttt{WISER} & \textbf{0.766 (0.54-0.96)} & \textbf{0.993} & \textbf{0.990} & \textbf{0.992} & \textbf{0.923 (0.83-0.99)} & \textbf{0.913 (0.77-0.99)} \\
     & \texttt{Aligator} & 0.000 (0.00-0.00) & 0.000 & 0.000 & 0.000 & 0.203 (0.20-0.20) & 0.141 (0.14-0.14) \\
     & \texttt{SeedBS} & 0.591 (0.00-0.95) & 0.769 & 0.740 & 0.754 & 0.902 (0.70-0.99) & 0.870 (0.64-0.99) \\
     & \texttt{WaterSeeker} & 0.623 (0.34-0.92) & 0.990 & 0.815 & 0.894 & 0.851 (0.66-0.98) & 0.836 (0.63-0.98) \\
\midrule
\multirow{5}{*}{
        \makecell[l]{facebook \\ opt-125m \\ (50272)}
    } & \ohl{\texttt{Oracle}} & \ohl{0.869 (0.62-0.97)} & \ohl{0.995} & \ohl{0.995} & \ohl{0.995} & \ohl{0.966 (0.89-0.99)} & \ohl{0.965 (0.88-0.99)} \\
 & \texttt{WISER} & \textbf{0.742 (0.50-0.94)} & \textbf{1.000} & \textbf{0.975} & \textbf{0.987} & \textbf{0.914 (0.76-0.99)} & \textbf{0.903 (0.74-0.99)} \\
\bottomrule
\end{tabular}}
\caption{\upd{Performance metrics for Red-Green watermarking in experimental setting~\ref{sim:setup1}.}}
\label{tab:results-redgreen}
\end{table}


\begin{table}[ht]
\centering
\resizebox{\textwidth}{!}{
\begin{tabular}{llrrrrrr}
\toprule
\textbf{Model Name} & \textbf{Method} & \textbf{IOU} & \textbf{Precision} & \textbf{Recall} & \textbf{F1} & \textbf{RI} & \textbf{MRI} \\
\midrule
\multirow{5}{*}{
        \makecell[l]{facebook \\ opt-1.3b \\ (50272)}
    } & \ohl{\texttt{Oracle}} & \ohl{0.921 (0.69-0.98)} & \ohl{0.995} & \ohl{0.995} & \ohl{0.995} & \ohl{0.978 (0.92-1.00)} & \ohl{0.977 (0.92-1.00)} \\
 & \texttt{WISER} & \textbf{0.904 (0.68-0.98)} & \textbf{1.000} & \textbf{0.990} & \textbf{0.995} & \textbf{0.971 (0.91-1.00)} & \textbf{0.969 (0.89-1.00)} \\
 & \texttt{Aligator} & 0.436 (0.00-0.81) & 0.206 & 0.917 & 0.336 & 0.881 (0.59-0.96) & 0.856 (0.53-0.95) \\
 & \texttt{SeedBS} & 0.383 (0.00-0.91) & 0.643 & 0.515 & 0.572 & 0.801 (0.20-0.98) & 0.757 (0.14-0.97) \\
 & \texttt{WaterSeeker} & 0.677 (0.38-0.90) & 1.000 & 0.840 & 0.913 & 0.873 (0.70-0.97) & 0.861 (0.67-0.97) \\
\midrule
\multirow{5}{*}{
        \makecell[l]{facebook \\ opt-125m \\ (50272)}
    } & \ohl{\texttt{Oracle}} & \ohl{0.935 (0.75-0.99)} & \ohl{0.998} & \ohl{0.998} & \ohl{0.998} & \ohl{0.982 (0.95-1.00)} & \ohl{0.980 (0.95-1.00)} \\
 & \texttt{WISER} & \textbf{0.926 (0.73-0.99)} & 0.998 & \textbf{0.998} & \textbf{0.998} & \textbf{0.980 (0.92-1.00)} & \textbf{0.979 (0.92-1.00)} \\
\bottomrule
\end{tabular}}
    \caption{\upd{Performance metrics for Permute and Flip watermarking in experimental setting~\ref{sim:setup1}.}}
    \label{tab:results-pf}
\end{table}

\begin{table}[ht]
\centering
\resizebox{\textwidth}{!}{
\begin{tabular}{llrrrrrr}
\toprule
\textbf{Model Name} & \textbf{Method} & \textbf{IOU} & \textbf{Precision} & \textbf{Recall} & \textbf{F1} & \textbf{RI} & \textbf{MRI} \\
\midrule
\multirow{5}{*}{
        \makecell[l]{facebook \\ opt-1.3b \\ (50272)}
    } & \ohl{\texttt{Oracle}} & \ohl{0.924 (0.49-0.99)} & \ohl{0.988} & \ohl{0.988} & \ohl{0.988} & \ohl{0.988 (0.93-1.00)} & \ohl{0.987 (0.91-1.00)} \\
 & \texttt{WISER} & \textbf{0.787 (0.48-0.92)} & \textbf{0.979} & \textbf{0.978} & \textbf{0.978} & \textbf{0.959 (0.85-0.99)} & \textbf{0.955 (0.83-0.99)} \\
 & \texttt{Aligator} & 0.000 (0.00-0.00) & 0.002 & 0.001 & 0.001 & 0.114 (0.11-0.11) & 0.078 (0.08-0.08) \\
 & \texttt{SeedBS} & 0.264 (0.00-0.84) & 0.603 & 0.407 & 0.486 & 0.767 (0.11-0.98) & 0.727 (0.08-0.97) \\
 & \texttt{WaterSeeker} & 0.399 (0.18-0.55) & 0.813 & 0.596 & 0.688 & 0.848 (0.70-0.94) & 0.832 (0.68-0.93) \\
\midrule
\multirow{5}{*}{
        \makecell[l]{facebook \\ opt-125m \\ (50272)}
    } & \ohl{\texttt{Oracle}} & \ohl{0.918 (0.54-0.99)} & \ohl{0.984} & \ohl{0.984} & \ohl{0.984} & \ohl{0.986 (0.90-1.00)} & \ohl{0.983 (0.85-1.00)} \\
 & \texttt{WISER} & \textbf{0.786 (0.37-0.93)} & \textbf{0.991} & \textbf{0.967} & \textbf{0.979} & \textbf{0.952 (0.66-0.99)} & \textbf{0.947 (0.64-0.99)} \\
\bottomrule
\end{tabular}}
\caption{\upd{Performance metrics for Inverse watermarking in experimental setting~\ref{sim:setup2}.}}
\label{tab:results-inverse-n2500}
\end{table}


\begin{table}[ht]
\centering
\resizebox{\textwidth}{!}{
\begin{tabular}{llrrrrrr}
\toprule
\textbf{Model Name} & \textbf{Method} & \textbf{IOU} & \textbf{Precision} & \textbf{Recall} & \textbf{F1} & \textbf{RI} & \textbf{MRI} \\
\midrule
\multirow{5}{*}{
        \makecell[l]{facebook \\ opt-1.3b \\ (50272)}
    } & \ohl{\texttt{Oracle}} & \ohl{0.906 (0.56-0.97)} & \ohl{0.991} & \ohl{0.991} & \ohl{0.991} & \ohl{0.986 (0.92-1.00)} & \ohl{0.985 (0.91-1.00)} \\
 & \texttt{WISER} & \textbf{0.874 (0.51-0.97)} & \textbf{0.976} & \textbf{0.974} & \textbf{0.975} & \textbf{0.972 (0.82-1.00)} & \textbf{0.970 (0.80-1.00)} \\
 & \texttt{Aligator} & 0.000 (0.00-0.00) & 0.000 & 0.000 & 0.000 & 0.112 (0.11-0.11) & 0.076 (0.08-0.08) \\
 & \texttt{SeedBS} & 0.428 (0.08-0.93) & 0.701 & 0.905 & 0.790 & 0.965 (0.92-0.99) & 0.935 (0.88-0.98) \\
 & \texttt{WaterSeeker} & 0.505 (0.35-0.66) & 0.768 & 0.782 & 0.775 & 0.909 (0.79-0.96) & 0.898 (0.78-0.95) \\
\midrule
\multirow{5}{*}{
        \makecell[l]{facebook \\ opt-125m \\ (50272)}
    } & \ohl{\texttt{Oracle}} & \ohl{0.928 (0.65-0.98)} & \ohl{0.994} & \ohl{0.994} & \ohl{0.994} & \ohl{0.990 (0.96-1.00)} & \ohl{0.989 (0.95-1.00)} \\
 & \texttt{WISER} & \textbf{0.897 (0.65-0.97)} & \textbf{0.989} & \textbf{0.986} & \textbf{0.987} & \textbf{0.980 (0.92-1.00)} & \textbf{0.979 (0.91-1.00)} \\
\bottomrule
\end{tabular}}
\caption{\upd{Performance metrics for Red-Green watermarking in experimental setting~\ref{sim:setup2}.}}
\label{tab:results-redgreen-n2500}
\end{table}


\begin{table}[ht]
\centering
\resizebox{\textwidth}{!}{
\begin{tabular}{llrrrrrr}
\toprule
\textbf{Model Name} & \textbf{Method} & \textbf{IOU} & \textbf{Precision} & \textbf{Recall} & \textbf{F1} & \textbf{RI} & \textbf{MRI} \\
\midrule
\multirow{5}{*}{
        \makecell[l]{facebook \\ opt-1.3b \\ (50272)}
    } & \ohl{\texttt{Oracle}} & \ohl{0.938 (0.64-0.99)} & \ohl{0.994} & \ohl{0.994} & \ohl{0.994} & \ohl{0.990 (0.94-1.00)} & \ohl{0.989 (0.93-1.00)} \\
 & \texttt{WISER} & \textbf{0.927 (0.54-0.99)} & \textbf{0.992} & \textbf{0.990} & \textbf{0.991} & \textbf{0.987 (0.90-1.00)} & \textbf{0.985 (0.88-1.00)} \\
 & \texttt{Aligator} & 0.550 (0.19-0.87) & 0.137 & 0.917 & 0.239 & 0.927 (0.63-0.98) & 0.916 (0.60-0.98) \\
 & \texttt{SeedBS} & 0.398 (0.07-0.89) & 0.836 & 0.849 & 0.842 & 0.941 (0.68-0.99) & 0.910 (0.64-0.98) \\
 & \texttt{WaterSeeker} & 0.513 (0.36-0.65) & 0.849 & 0.789 & 0.818 & 0.889 (0.64-0.96) & 0.877 (0.61-0.95) \\
\midrule
\multirow{5}{*}{
        \makecell[l]{facebook \\ opt-125m \\ (50272)}
    } & \ohl{\texttt{Oracle}} & \ohl{0.963 (0.80-0.99)} & \ohl{0.999} & \ohl{0.999} & \ohl{0.999} & \ohl{0.995 (0.97-1.00)} & \ohl{0.995 (0.97-1.00)} \\
 & \texttt{WISER} & \textbf{0.959 (0.77-0.99)} & \textbf{0.994} & \textbf{0.999} & \textbf{0.996} & \textbf{0.995 (0.97-1.00)} & \textbf{0.994 (0.96-1.00)} \\
\bottomrule
\end{tabular}}
\caption{\upd{Performance metrics for Permute-and-Flip watermarking in experimental setting~\ref{sim:setup2}.}}
\label{tab:results-pf-n2500}
\end{table}


\begin{table}[ht]
\centering
\resizebox{\textwidth}{!}{
\begin{tabular}{llrrrrrr}
\toprule
\textbf{Model Name} & \textbf{Method} & \textbf{IOU} & \textbf{Precision} & \textbf{Recall} & \textbf{F1} & \textbf{RI} & \textbf{MRI} \\
\midrule
\multirow{5}{*}{
        \makecell[l]{facebook \\ opt-1.3b \\ (50272)}
    } & \ohl{\texttt{Oracle}} & \ohl{0.914 (0.63-0.98)} & \ohl{0.990} & \ohl{0.990} & \ohl{0.990} & \ohl{0.991 (0.94-1.00)} & \ohl{0.991 (0.94-1.00)} \\
 & \texttt{WISER} & \textbf{0.872 (0.63-0.97)} & 0.976 & \textbf{0.980} & \textbf{0.978} & \textbf{0.982 (0.94-1.00)} & \textbf{0.981 (0.93-1.00)} \\
 & \texttt{Aligator} & 0.456 (0.04-0.83) & 0.261 & 0.842 & 0.398 & 0.946 (0.82-0.99) & 0.939 (0.81-0.99) \\
 & \texttt{SeedBS} & 0.466 (0.00-0.71) & 0.794 & 0.247 & 0.377 & 0.863 (0.63-0.96) & 0.848 (0.61-0.95) \\
 & \texttt{WaterSeeker} & 0.645 (0.34-0.84) & \textbf{0.998} & 0.753 & 0.858 & 0.946 (0.81-0.99) & 0.940 (0.80-0.99) \\
\midrule
\multirow{5}{*}{
        \makecell[l]{facebook \\ opt-125m \\ (50272)}
    } & \ohl{\texttt{Oracle}} & \ohl{0.925 (0.58-0.98)} & \ohl{0.991} & \ohl{0.991} & \ohl{0.991} & \ohl{0.993 (0.95-1.00)} & \ohl{0.992 (0.95-1.00)} \\
 & \texttt{WISER} & \textbf{0.885 (0.59-0.98)} & 0.993 & \textbf{0.988} & \textbf{0.990} & \textbf{0.989 (0.94-1.00)} & \textbf{0.988 (0.93-1.00)} \\
\bottomrule
\end{tabular}}
\caption{\upd{Performance metrics for Gumbel watermarking in experimental setting~\ref{sim:setup3}.}}
\label{tab:results-gumbel-n1000}
\end{table}


\begin{table}[ht]
\centering
\resizebox{\textwidth}{!}{
\begin{tabular}{llrrrrrr}
\toprule
\textbf{Model Name} & \textbf{Method} & \textbf{IOU} & \textbf{Precision} & \textbf{Recall} & \textbf{F1} & \textbf{RI} & \textbf{MRI} \\
\midrule
\multirow{5}{*}{
        \makecell[l]{facebook \\ opt-1.3b \\ (50272)}
    } & \ohl{\texttt{Oracle}} & \ohl{0.853 (0.57-0.96)} & \ohl{0.947} & \ohl{0.947} & \ohl{0.947} & \ohl{0.976 (0.92-1.00)} & \ohl{0.974 (0.90-1.00)} \\
 & \texttt{WISER} & \textbf{0.744 (0.30-0.95)} & 0.949 & \textbf{0.850} & \textbf{0.897} & 0.949 (0.79-0.99) & \textbf{0.945 (0.78-0.99)} \\
 & \texttt{Aligator} & 0.000 (0.00-0.00) & 0.000 & 0.000 & 0.000 & 0.189 (0.19-0.19) & 0.171 (0.17-0.17) \\
 & \texttt{SeedBS} & 0.203 (0.00-0.66) & 0.377 & 0.138 & 0.202 & 0.599 (0.19-0.93) & 0.570 (0.17-0.92) \\
 & \texttt{WaterSeeker} & 0.689 (0.34-0.87) & \textbf{0.996} & 0.810 & 0.894 & \textbf{0.949 (0.82-0.99)} & 0.945 (0.80-0.99) \\
\midrule
\multirow{5}{*}{
        \makecell[l]{facebook \\ opt-125m \\ (50272)}
    } & \ohl{\texttt{Oracle}} & \ohl{0.886 (0.65-0.96)} & \ohl{0.974} & \ohl{0.974} & \ohl{0.974} & \ohl{0.985 (0.93-1.00)} & \ohl{0.983 (0.93-1.00)} \\
 & \texttt{WISER} & \textbf{0.853 (0.54-0.96)} & 0.965 & \textbf{0.948} & \textbf{0.956} & \textbf{0.976 (0.87-1.00)} & \textbf{0.975 (0.86-1.00)} \\
\bottomrule
\end{tabular}}
\caption{\upd{Performance metrics for Inverse watermarking in experimental setting~\ref{sim:setup3}.}}
\label{tab:results-inverse-n1000}
\end{table}


\begin{table}[ht]
\centering
\resizebox{\textwidth}{!}{
\begin{tabular}{llrrrrrr}
\toprule
\textbf{Model Name} & \textbf{Method} & \textbf{IOU} & \textbf{Precision} & \textbf{Recall} & \textbf{F1} & \textbf{RI} & \textbf{MRI} \\
\midrule
\multirow{5}{*}{
        \makecell[l]{facebook \\ opt-1.3b \\ (50272)}
    } & \ohl{\texttt{Oracle}} & \ohl{0.785 (0.52-0.92)} & \ohl{0.959} & \ohl{0.959} & \ohl{0.959} & \ohl{0.974 (0.91-0.99)} & \ohl{0.972 (0.90-0.99)} \\
 & \texttt{WISER} & \textbf{0.672 (0.32-0.89)} & 0.914 & 0.823 & \textbf{0.866} & 0.940 (0.78-0.99) & 0.936 (0.77-0.99) \\
 & \texttt{Aligator} & 0.000 (0.00-0.00) & 0.000 & 0.000 & 0.000 & 0.189 (0.19-0.19) & 0.171 (0.17-0.17) \\
 & \texttt{SeedBS} & 0.373 (0.06-0.70) & 0.717 & 0.401 & 0.514 & 0.896 (0.80-0.97) & 0.872 (0.76-0.96) \\
 & \texttt{WaterSeeker} & 0.660 (0.36-0.90) & \textbf{0.998} & 0.746 & 0.854 & \textbf{0.945 (0.84-0.99)} & \textbf{0.940 (0.83-0.99)} \\
\midrule
\multirow{5}{*}{
        \makecell[l]{facebook \\ opt-125m \\ (50272)}
    } & \ohl{\texttt{Oracle}} & \ohl{0.820 (0.56-0.94)} & \ohl{0.966} & \ohl{0.966} & \ohl{0.966} & \ohl{0.978 (0.91-1.00)} & \ohl{0.977 (0.90-1.00)} \\
 & \texttt{WISER} & \textbf{0.755 (0.42-0.92)} & 0.935 & 0.908 & \textbf{0.921} & \textbf{0.965 (0.86-0.99)} & \textbf{0.963 (0.85-0.99)} \\
\bottomrule
\end{tabular}}
\caption{\upd{Performance metrics for Red-Green watermarking in experimental setting~\ref{sim:setup3}.}}
\label{tab:results-redgreen-n1000}
\end{table}


\begin{table}[ht]
\centering
\resizebox{\textwidth}{!}{
\begin{tabular}{llrrrrrr}
\toprule
\textbf{Model Name} & \textbf{Method} & \textbf{IOU} & \textbf{Precision} & \textbf{Recall} & \textbf{F1} & \textbf{RI} & \textbf{MRI} \\
\midrule
\multirow{5}{*}{
        \makecell[l]{facebook \\ opt-1.3b \\ (50272)}
    } & \ohl{\texttt{Oracle}} & \ohl{0.911 (0.68-0.97)} & \ohl{0.995} & \ohl{0.995} & \ohl{0.995} & \ohl{0.992 (0.98-1.00)} & \ohl{0.991 (0.97-1.00)} \\
 & \texttt{WISER} & \textbf{0.879 (0.68-0.97)} & 0.986 & \textbf{0.995} & \textbf{0.990} & \textbf{0.989 (0.95-1.00)} & \textbf{0.988 (0.95-1.00)} \\
 & \texttt{Aligator} & 0.394 (0.07-0.81) & 0.235 & 0.841 & 0.367 & 0.951 (0.85-0.99) & 0.943 (0.84-0.98) \\
 & \texttt{SeedBS} & 0.408 (0.00-0.70) & 0.705 & 0.252 & 0.371 & 0.861 (0.60-0.96) & 0.847 (0.58-0.95) \\
 & \texttt{WaterSeeker} & 0.640 (0.46-0.83) & \textbf{1.000} & 0.761 & 0.864 & 0.956 (0.87-0.99) & 0.950 (0.86-0.99) \\
\midrule
\multirow{5}{*}{
        \makecell[l]{facebook \\ opt-125m \\ (50272)}
    } & \ohl{\texttt{Oracle}} & \ohl{0.923 (0.74-0.98)} & \ohl{0.997} & \ohl{0.997} & \ohl{0.997} & \ohl{0.993 (0.99-1.00)} & \ohl{0.993 (0.98-1.00)} \\
 & \texttt{WISER} & \textbf{0.888 (0.71-0.97)} & 0.993 & \textbf{0.994} & \textbf{0.994} & \textbf{0.991 (0.95-1.00)} & \textbf{0.990 (0.95-1.00)} \\
\bottomrule
\end{tabular}}
\caption{\upd{Performance metrics for Permute-and-Flip watermarking in experimental setting~\ref{sim:setup3}.}}
\label{tab:results-pf-n1000}
\end{table}

\subsubsection{Why \texttt{WISER}\ outperforms other methods} \label{se:wiser-best-why}

The enhanced performance of \texttt{WISER}\ does not come out-of-the-blue, rather, we argue that it is a byproduct of our unique, epidemic change-point perspective that marries theoretical validity with practical insights. While these methods—\texttt{SeedBS-NOT}, \texttt{Aligator}, and \texttt{Waterseeker}— each contribute useful perspectives, they also exhibit important limitations that the generality of our method usually overcomes.

\textbf{Limitations of \texttt{SeedBS-NOT}:}
The limitations of \texttt{SeedBS-NOT} primarily arise from its reliance on a permutation-based change-point detection framework, which is inherently computationally expensive. Moreover, nowhere they restrict their attention to the specific scenario of watermarked segments, which consigns the change-points to occur in pairs, corresponding to the start and end of a watermarked segment. This is automatically alleviated by \texttt{WISER}\ through its adoption of a natural epidemic change-point formulation. This structural assumption substantially reduces the search space, yielding both computational efficiency and improved statistical stability. Additionally, \texttt{SeedBS-NOT} works with the sequence of p-values that are computed from a single observation of the pivot statistic at that location. Due to the complicated nature of the dependence between these p-values, they are difficult to combine to increase the statistical power. Our approach circumvents this by aggregating the pivot statistic at the block level (Step 7 in Figure~\ref{fig:algo}), enhancing the effective sample size and increasing the power of the detection.

\textbf{Limitations of \texttt{Aligator}:}
The \texttt{Aligator} algorithm frames the task as a reinforcement learning problem, producing a smoothed estimate of the underlying generative process and subsequently applying token-level hypothesis tests with a p-value threshold. While this strategy can capture localized deviations, it often results in a large number of short and fragmented detections, many of which might be spurious due to possible multiple testing. Consequently, the method tends to produce many disjoint intervals, which severely diminishes its precision. By contrast, the discarding stage of \texttt{WISER}\ enforces structural coherence at the segment level, before returning fine-grained estimate by applying (\ref{eq:single-estimate}). This ensures that localized intervals correspond more closely to contiguous watermark insertions.

\textbf{Limitations of \texttt{Waterseeker}:}
The \texttt{Waterseeker} algorithm may seem structurally similar to the proposed \texttt{WISER}\ method, in that it also employs a two-stage detection framework. However, \texttt{Waterseeker} considers a sliding window-based testing mechanism in its first stage, which has a crucial limitation. Consider a very realistic scenario when one of the pivot statistics corresponding to an unwatermarked token is high simply due to random chance. In \texttt{Waterseeker}, this will push the score up for $W$ consecutive windows, usually resulting in a false positive in the first stage. On the other hand, for \texttt{WISER}\, such anomalous pivot statistics will affect only one block, which, being usually part of a connected interval with small length, can potentially be discarded with a very high probability in our \textit{discarding stage}. For larger models, this scenario is extremely likely, making this reduction in precision much more pronounced (see Models google/gemma-3-270m and meta-Ilama/Meta-Llama-3-8B in Tables \ref{tab:results-gumbel}, \ref{tab:results-inverse} - \ref{tab:results-pf}). Moreover, \cite{pan2025waterseeker} provide only limited theoretical validation of their approach, making the optimal tuning of hyperparameters difficult to justify. This lack of statistical guarantees limits its reliability across watermarking schemes and model sizes, in contrast to the rigorous and general guarantees underlying \texttt{WISER}.

\subsection{Effect of Watermark Intensity} \label{se:effect-wm-intensity}

\begin{table}[htbp]
    \centering
    \begin{tabular}{llrrrr}
    \toprule
    \textbf{Type} & \textbf{Method} & \textbf{IOU} & \textbf{F1} & \textbf{RI} & \textbf{MRI} \\
    \midrule
    \multirow{4}{*}{Strong but short} & \texttt{WISER} & 0.794 & 0.984 & 0.933 & 0.925 \\
    & \texttt{SeedBS-NOT} & 0.639 & 0.785 & 0.919 & 0.900 \\
    & \texttt{Waterseeker} & 0.878 & 0.997 & 0.969 & 0.967 \\
    \midrule
    \multirow{4}{*}{Weak but long} & \texttt{WISER} & 0.745 & 0.779 & 0.628 & 0.551 \\
    & \texttt{SeedBS-NOT} & 0.172 & 0.321 & 0.675 & 0.187 \\
    & \texttt{Waterseeker} & 0.268 & 0.847 & 0.519 & 0.172 \\
    \bottomrule
    \end{tabular}
    \caption{Effect on watermarking signal strength}
    \label{tab:wm-strength}
\end{table}

Following the experimental design of \cite{pan2025waterseeker}, we evaluate the comparative performance of the proposed \texttt{WISER}\ algorithm under varying levels of watermark intensity. As a demonstration, we choose Google's Gemma-3 series model (270 million) to generate a completion of $500$ tokens for each input prompt. The watermark strength is modulated through the bias parameter $\delta$ of the Red-Green watermarking scheme~\citep{kirchenbauer2023watermark}, while another parameter $m$ specifies the length of the watermarked region by applying the decoding strategy to the middle $m$ tokens within the $500$-token output.

In the ``strong but short'' configuration ($\delta = 2.0, m = 100$), as shown in Table~\ref{tab:wm-strength}, all methods perform well, achieving a Rand Index exceeding $0.9$. Although \texttt{WISER}\ is not the best-performing method in this particular case, it remains competitive with \texttt{Waterseeker}, which achieves the highest score. By contrast, in the ``weak but long'' configuration ($\delta = 1.0, m = 400$), only \texttt{WISER}\ maintains robust performance. While \texttt{SeedBS-NOT} appears to achieve a higher Rand Index, this outcome is primarily attributed to the issues described in \S\ref{se:performance-metric}. The Modified Rand Index (MRI) offers a more reliable assessment, highlighting the superiority of \texttt{WISER}\ in this setting.

\subsection{Ablation studies} \label{se:ablation-study}

We also perform an ablation study to understand the effectiveness of the hyperparameters (e.g., block size and $\rho$) of \texttt{WISER}. Our results are arguably quite stable across wide choices of the tuning parameters; nevertheless, we provide more informed choices along with additional insights.

For this study, we consider a single watermarked segment from token index $100$ to $200$, fix $\rho = 0.25$, and vary the tuning parameter $b$ of the \texttt{WISER}\ algorithm. As one would have hoped, increasing the block size too much decreases the performance, as the smaller watermarked segments get subsumed in the noise of unwatermarked segments when block sizes are too large. On the other hand, decreasing the block size would reduce the statistical power of the detection algorithm in the first stage itself. Therefore, one requires a judicious choice of the block size to optimally balance these two aspects, which is empirically observed through the upper plot of Figure~\ref{fig:ablation-study}. Based on empirical evidence, we recommend the choice $b \in (\lceil \sqrt{n} \rceil, 3\lceil \sqrt{n} \rceil)$, which works quite well in various settings that we have experimented with, while being also theoretically supported.

\begin{figure}[ht]
    \centering
    \includegraphics[width=0.9\linewidth]{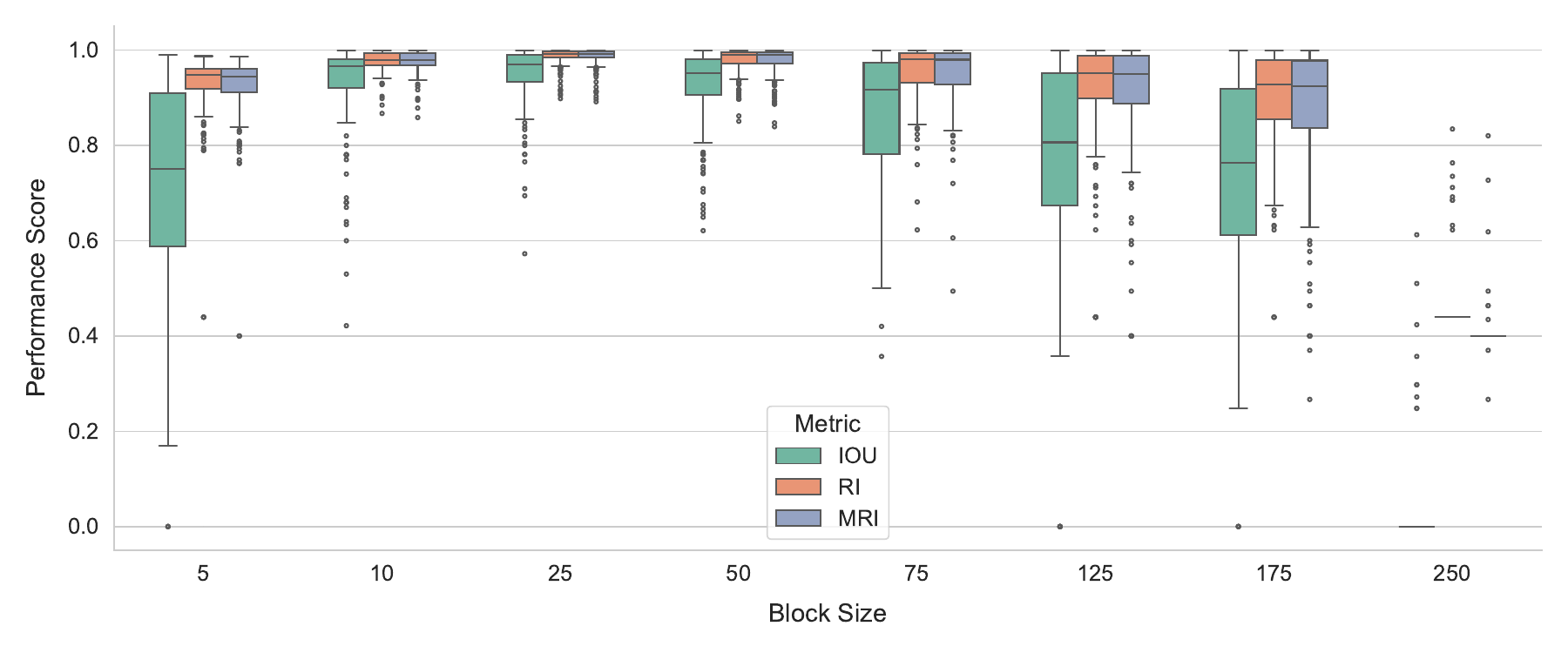}
    \includegraphics[width=0.9\linewidth]{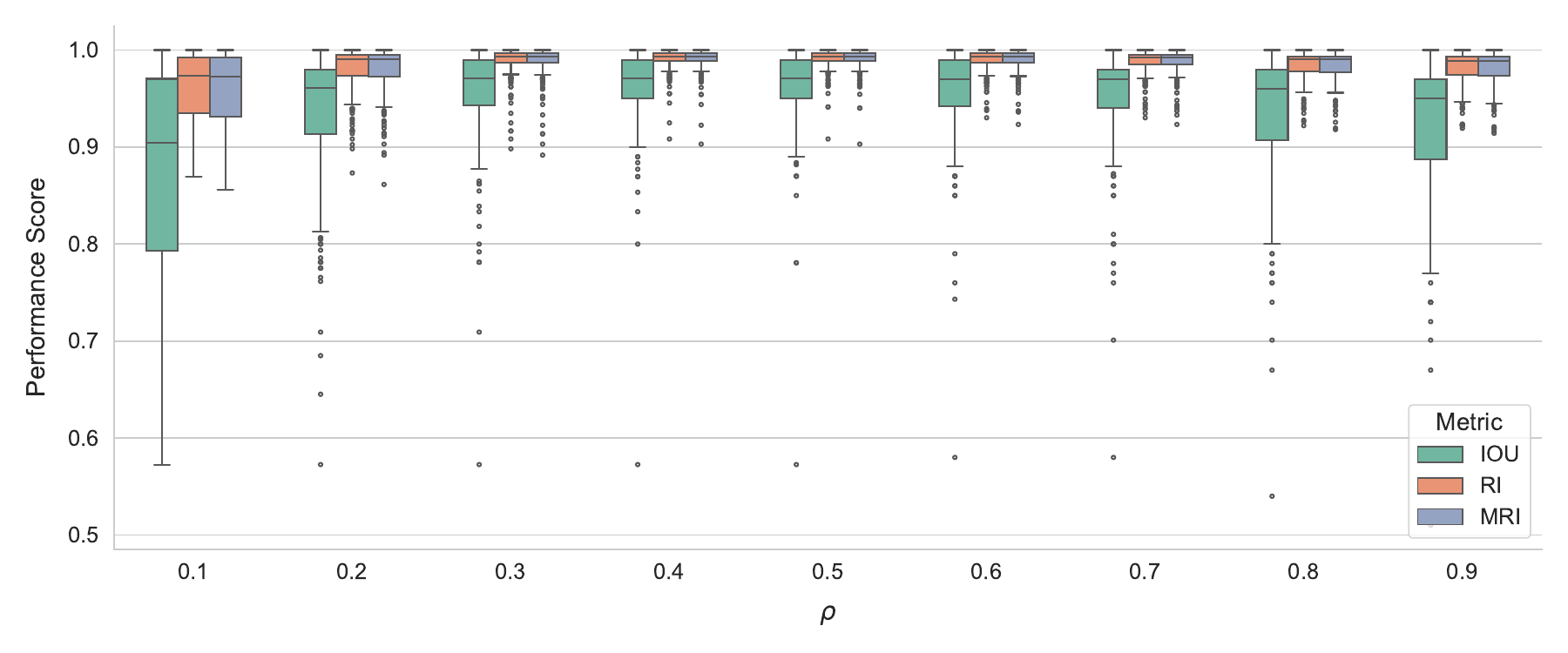}
    \caption{Effect on performance metrics (IOU and Rand Index) due to modification of the hyper-parameters of the \texttt{WISER}\ algorithm, namely block size (Top) and $\rho$ (Bottom).}
    \label{fig:ablation-study}
\end{figure}

A similar conclusion also holds for the choice of $\rho$, for which we fix the block size as $b=25$ and vary the tuning parameter $\rho$. As the choice of $d$ in Assumption~\ref{ass:alt-mean} is exogenously determined based on the language model and watermarking scheme, a large value of $\rho$ would imply a smaller $\tilde{d}$ and, by virtue of Theorem~\ref{thm:single-watermark}, would imply a larger error. The lower plot of Figure~\ref{fig:ablation-study} demonstrates this empirically. However, any value of $\rho$ between $0.1$ and $0.5$ provides reasonable and relatively stable estimates.

\section{Proof of Theoretical Results}\label{se:proof}
In this section, we collect the proofs of theoretical results in the \S\ref{se:theory}. Before we proceed further, we establish some notations. In the following, we write $a_n \lesssim b_n$ if $a_n \le C b_n$ for some constant $C > 0$, and $a_n \asymp b_n$ if $C_1 b_n \le a_n \le C_2 b_n$ for some constants $C_1, C_2 > 0$. Often we denote $a_n \lesssim b_n$ by $a_n=O(b_n)$. Additionally, if $a_n/b_n\to 0$, we write $a_n=o(b_n)$. For a function $f: \R^n \otimes \R^m \to \R$, let $f^{(1)} (\theta, w)= \frac{\partial}{\partial \theta} f(\theta, w)$, $\theta\in \R^n, w\in \R^m$, $n,m\geq 1$, be the partial derivative function with respect to $\theta$. 

\subsection{Proof of Theorem \ref{thm:single-watermark}}
In the following, we first state and prove a more generalized version of Theorem \ref{thm:single-watermark}. 

\begin{theorem}\label{thm:single-watermark-app}
    Let $\{X_t\}_{t=1}^n:= \{h(Y_t)\}_{t=1}^n$ be the pivot statistics based on the given input text, and assume that $I_0 \subset \{1, \ldots, n\}$ is the watermarked interval. Grant Assumption \ref{ass:alt-mean}. Let us also denote 
\[ \varepsilon_t=\begin{cases}
    &X_t-\mu_0, t\notin I_0, \\
    & X_t - \mu_t, \ \mu_t:=\IE_{1, \IP_t}[X_t], t\in I_0.
\end{cases}\]
    Suppose the class of distributions $\mathcal{P}$ is closed and compact, and there exists $\eta>0$ such that $\sup_{\IP\in \mathcal{P}}\IE_{1,\IP}[\exp(\eta|\varepsilon|)]<\infty$. Moreover, assume that $\min\{\Var_0(\varepsilon), \sup_{\IP}\Var_{1,P}(\varepsilon)\}>0$. Then it holds that 
    \[ |\hat{I} \Delta I_0|= O_{\IP}\big((\sup_{\theta\geq 0} \{\theta \rho \Tilde{d}- \Psi(\theta)\})^{-1}\big), \]
    where $\Delta$ denotes the symmetric difference operator, $O_{\IP}$ hides constants independent of $n$ and $ \Tilde{d}$, and
    \[\Psi(\theta)=\log \IE_0[\exp(\theta\varepsilon)] + 2^{-1} \log \sup_{\IP} \IE_{1,\IP}[\exp(2\theta \varepsilon)] + 2^{-1} \log \sup_{\IP} \IE_{1,\IP}[\exp(-2\theta \varepsilon)].\]
\end{theorem}

Theorem \ref{thm:single-watermark-app} is proved by showing that the probability $\IP(|\hat{I} \Delta I_0|> M)$ is small for all sufficiently large $M$. This probability is controlled by considering the objective function $V_I = S_{I^c} - (\mu_0 + \rho \Tilde{d}) |I^c|$, where $S_I = \sum_{k\in I} X_k$ and $S_{I^c}= \sum_{k=1}^n X_k - S_I$, and noting that, by construction of $\hat{I}$, $\IP(|\hat{I} \Delta I_0|> M) \leq \IP(\inf_{I: |I \Delta I_0|>M} V_I - V_{I_0} \leq 0)$. Usually, in change-point literature, one controls terms such as $\inf_{I: |I \Delta I_0|>M} V_I - V_{I_0}$ through Hàjek-Rényi type inequality \cite{hajekreyni1955}; see \cite{bai1994, bonnerjee2025}. Such inequalities are usually derived by dividing the domain, on which infimum is taken, into smaller intervals, and applying Doob's inequality or Rosenthal's inequality piece-meal. However, the main bottleneck in this particular setting is the potentially strong dependence between the pivot statistics in watermarked patches. We develop novel arguments that exploit $\sup_{\IP\in \mathcal{P}}\IE_{1,\IP}[\exp(\eta|\varepsilon|)]<\infty$ to provide an extended version of the Hajek-Renyi theory through the lens of the cumulant generating function. The proof is provided below.

\begin{proof}[Proof of Theorem \ref{thm:single-watermark-app}]
For a candidate watermarked interval $I$, let $A_1 (I) = I \cap I_0^c$, $A_2(I)= I \cap I_0$, $A_3(I) = I^c \cap I_0$, $A_4(I) =(I \cup I_0)^c$, and correspondingly $x_i(I)=|A_i(I)|$, $i=1(1)4$. Subsequently, we omit the argument $I$ when it is clear from the context. Note that $|I_0|= x_2+x_3$, $|I|=x_1+x_2$, and $|\hat{I} \Delta I_0|=x_1+x_3$. Note that, by definition of $\hat{I}$ it follows that $V_{\hat{I}} \leq V_{I_0}$. 
Finally, denote $S_i = \sum_{k \in A_i} X_k$, and $S_i^{\varepsilon}= \sum_{k \in A_i} \varepsilon_k$. With these notations established, we proceed through the following series of implications. 
\allowdisplaybreaks \begin{align}
    V_I  - V_{I_0} &= (S_{I^c}- S_{I_0^c}) -  (|I^c| - |I_0^c|)(\mu_0 + \rho \Tilde{d}) \nonumber\\
    &= S_3 - S_1 -  (x_3- x_1) (\mu_0 + \rho \Tilde{d}) \nonumber\\
    &= (S_3^{\varepsilon} + \sum_{t\in A_3} \mu_t) - (S_1^{\varepsilon} + x_1 \mu_0) -  (x_3- x_1) (\mu_0 + \rho \Tilde{d}) \nonumber\\
    &= S_3^\varepsilon - S_1^\varepsilon + \sum_{t\in A_3} (\mu_t - \mu_0 )+(x_1-x_3)\rho \Tilde{d}\nonumber\\
    &\geq S_3^\varepsilon - S_1^\varepsilon + x_3(d - \rho\Tilde{d})+ x_1 \rho \Tilde{d} \label{eq:d-bound-lb}\\
    & \geq S_3^\varepsilon - S_1^\varepsilon + (x_1+x_3)\rho \Tilde{d},
    \label{eq:lb}
\end{align}
where (\ref{eq:d-bound-lb}) follows from Assumption \ref{ass:alt-mean} and (\ref{eq:lb}) uses $d \geq 2\rho \Tilde{d}$. For some $M >0$, let $D_M:=\{I: |I\Delta I_0|>M\}$. Let $I_0 = [L,R]$. Note that, a candidate interval $I$ can belong to any of the following five sub-classes:
\begin{itemize}
        \item $\mathcal{P}_1:=\{I:I \subseteq I_0,  I \in D_M\}$. 
        \item  $\mathcal{P}_2:=\{I:I \supseteq I_0, I \in D_M\}$.
        \item  $\mathcal{P}_3:=\{I:I \cap I_0=\phi, I \in D_M\}$.
        \item  $\mathcal{P}_4:=\{ (a, b) :a< L < b < R, I = (a, b) \in D_M\}$.
        \item  $\mathcal{P}_5:=\{ (a, b) :L <a < R <b, I = (a, b) \in D_M\}$. 
    \end{itemize}
Subsequently, we detail the analysis for the relatively harder case $\mathcal{P}_4$. The arguments for the other cases are similar. Observe that:
\allowdisplaybreaks \begin{align}
    \IP(\hat{I} \in \mathcal{P}_4) &\leq \IP(\min_{I: I\in \mathcal{P}_4} V_I - V_{I_0}\leq 0) \nonumber\\
   &\leq \IP\big(\max_{I: I \in \mathcal{P}_4, x_1+x_3>M} \frac{S_1^\varepsilon - S_3^\varepsilon}{x_1+x_3} \geq  \rho \Tilde{d}) \nonumber\\
 & \leq \sum_{j=M+1}^{\infty} \inf_{\theta \geq 0}\IP\big(\max_{a,b: a < L <b< R: L-a+R-b=j} \exp(\theta (S_{[a,L]}^\varepsilon - S_{[b,R]}^\varepsilon)) \geq  \exp(\theta\rho \Tilde{d}j)) \nonumber\\
 &\leq \sum_{j=M+1}^{\infty} \inf_{\theta \geq 0} \exp(-\theta\rho \Tilde{d}j)\IE\big[\max_{a,b: a < L <b< R: L-a+R-b=j} \exp(\theta (S_{[a,L]}^\varepsilon - S_{[b,R]}^\varepsilon))\big]\nonumber\\
 &\leq \sum_{j=M+1}^{\infty} \inf_{\theta \geq 0} \exp(-\theta\rho \Tilde{d}j)\IE\big[\max_{a,b: a \in \{L-j+1, \cdots, L\}, b\in\{(R-j+1) \vee L, \cdots, R\}} \exp(\theta (S_{[a,L]}^\varepsilon - S_{[b,R]}^\varepsilon))\big] 
 \label{eq:exp-over-max}
\end{align}
For $j \in [n]$, let $\mathcal{F}_j:=\sigma( \{ (\omega_{s-1}, \zeta_s): s<j \})$.  Write 
\allowdisplaybreaks \begin{align}
  & \IE\big[\max_{a,b: a \in \{L-j+1, \cdots, L\}, b\in\{(R-j+1) \vee L, \cdots, R\}} \exp(\theta (S_{[a,L]}^\varepsilon - S_{[b,R]}^\varepsilon))\big] \nonumber\\
  ={}& \IE\Big[\max_{a: a \in \{L-j+1, \cdots, L\}} \exp(\theta S_{[a,L]}^\varepsilon) \IE\big[\max_{b\in\{(R-j+1) \vee L, \cdots, R\}} \exp(-\theta S_{[b,R]}^\varepsilon)\ | \ \mathcal{F}_{(R-j) \vee L} \big]\Big] \nonumber\\
  \leq{} & \IE\Big[\max_{a: a \in \{L-j+1, \cdots, L\}} \exp(\theta S_{[a,L]}^\varepsilon) \sqrt{\IE[\exp(-2\theta S^{\varepsilon}_{[(R-j+1)\vee L, R]}) \ | \ \mathcal{F}_{(R-j) \vee L}]} \nonumber\\ &\hspace*{2cm} \sqrt{\IE\big[\max_{b\in\{(R-j+1) \vee L, \cdots, R\}} \exp(2\theta S_{[(R-j+1)\vee L,b]}^\varepsilon)\ | \ \mathcal{F}_{(R-j) \vee L} \big]}\Big] 
  \label{eq: first-conditional},
\end{align}
where, (\ref{eq: first-conditional}) follows from Cauchy-Schwartz inequality. Now, note that, by construction of $\varepsilon_t$, conditional on $\mathcal{F}_{(R-j) \vee L}$, $\varepsilon_t$ is a martingale difference sequence adapted to $\sigma( \{ (\omega_{s-1}, \zeta_s): (R-j+1)\vee L \leq s \leq t \} )$. Since $x \mapsto \exp(2\theta x)$ is convex, hence $\exp(2\theta S_{[(R-j+1)\vee L,b]}^\varepsilon), b \in \{(R-j+1)\vee L, \ldots, R\}$ is a sub-martingale sequence. Consequently, Doob's maximal inequality~\citep{hallhydebook} applies. Further sequential conditioning yields the following series of inequalities.
\allowdisplaybreaks \begin{align}
    &\IE\big[\max_{b\in\{(R-j+1) \vee L, \cdots, R\}} \exp(2\theta S_{[(R-j+1)\vee L,b]}^\varepsilon)\ | \ \mathcal{F}_{(R-j) \vee L} \big] \nonumber\\ 
    \leq{} & 4\IE[\exp(2\theta S_{[(R-j+1)\vee L,R]}^\varepsilon) \ | \ \mathcal{F}_{(R-j) \vee L} ] \nonumber\\
    \leq{} & 4\IE[\exp(2\theta S_{[(R-j+1) \vee L, R-1]}^\varepsilon) \IE[\exp(2\theta \varepsilon_R)| \mathcal{F}_{R-1}] \ | \ \mathcal{F}_{(R-j) \vee L} ] \nonumber\\
    \leq{} & 4\sup_{\IP} \IE_{1,\IP}[\exp(2\theta \varepsilon)]\IE[\exp(2\theta S_{[(R-j+1) \vee L, R-1]}^\varepsilon) \ | \ \mathcal{F}_{(R-j) \vee L} ] \nonumber\\
    \leq{} &  4\big(\sup_{\IP} \IE_{1,\IP}[\exp(2\theta \varepsilon)] \big)^j. \label{eq:inner-condn-exp-1}
\end{align}
 Proceeding along similar lines, we obtain 
\allowdisplaybreaks \begin{align}
    \IE[\exp(-2\theta S^{\varepsilon}_{[(R-j+1)\vee L, R]}) \ | \ \mathcal{F}_{(R-j) \vee L}] \leq 4\big(\sup_{\IP} \IE_{1,\IP}[\exp(-2\theta \varepsilon)] \big)^j, \label{eq:inner-condn-exp-2}
\end{align}
and 
\allowdisplaybreaks \begin{align}
    \IE[\max_{a: a \in \{L-j+1, \cdots, L\}} \exp(\theta S_{[a,L]}^\varepsilon) ] \leq 4\big(\IE_{0}[\exp(\theta \varepsilon)] \big)^j. \label{eq:inner-condn-exp-3}
\end{align}
Combining (\ref{eq:inner-condn-exp-1})-(\ref{eq:inner-condn-exp-3})  and plugging them in (\ref{eq: first-conditional}) and (\ref{eq:exp-over-max}), one obtains
\allowdisplaybreaks \begin{align}
    \IP(I \in \mathcal{P}_4) \leq 16\sum_{j=M+1}^{\infty} \inf_{\theta\geq 0} \bigg(\exp(-\theta\rho \Tilde{d}) \IE_{0}[\exp(\theta \varepsilon)] \sqrt{\sup_{\IP} \IE_{1,\IP}[\exp(2\theta \varepsilon)] \sup_{\IP} \IE_{1,\IP}[\exp(-2\theta \varepsilon)] }\bigg)^j. \label{eq:big-ub}
\end{align}
To deliver the coup de grâce of our argument, we are required to bound (\ref{eq:big-ub}). To that end, define $\phi:\R_+ \to \R$ as
\[ \phi (\theta) = -\theta\rho\Tilde{d} + \log \IE_0[\exp(\theta\varepsilon)] + 2^{-1} \log \sup_{\IP} \IE_{1,\IP}[\exp(2\theta \varepsilon)] + 2^{-1} \log \sup_{\IP} \IE_{1,\IP}[\exp(-2\theta \varepsilon)].\]
By definition of $\phi$, $\IP(I\in \mathcal{P}_4) \leq \sum_{j=M+1}^{\infty} \inf_{\theta\geq 0} \exp(j\phi(\theta)).$ Moreover, for $\lambda\in(0,1),\  \theta_1, \theta_2\in \R_+$, Hölder's inequality produces 
\allowdisplaybreaks \begin{align}\label{eq:convexity}
    \log \sup_{\IP} \IE_{1,\IP}[\exp(2 (\lambda\theta_1+ (1-\lambda)\theta_2) \varepsilon)] \leq \sup_{\IP} \Big(\lambda\log \IE_{1,\IP}[\exp(2\theta_1 \varepsilon)] + (1-\lambda)\log \IE_{1,\IP}[\exp(2\theta_2 \varepsilon)] \Big). 
\end{align}
Similar arguments for $ \log \IE_0[\exp(\theta\varepsilon)]$ and $\log \sup_{\IP} \IE_{1,\IP}[\exp(-2\theta \varepsilon)]$ show that $\phi$, being a linear combination of convex functions with non-negative weights (note that $-\theta\rho\Tilde{d}$ is linear) , is itself convex.

Let $f:\R \otimes \R^{|W|} \mapsto \R$ be given by $f(\theta, P)=\log  \IE_{0}[\exp(2\theta \varepsilon)] + \log  \IE_{1,\IP}[\exp(2\theta \varepsilon)]$. Recalling that $f^{(1)} (\theta, w)= \frac{\partial}{\partial \theta} f(\theta, w)$, observe that
\allowdisplaybreaks \begin{align}
    f^{(1)}(0, P) =0 \ \text{for any $P\in \mathcal{P}$}, \label{eq:fixed-P}
\end{align} 
since $\IE_0[\varepsilon]=\IE_{1,\IP}[\varepsilon]=0$. Therefore, noting that $\mathcal{P}$ is a compact subset of the $|W|$-dimensional simplex, in light of $\sup_{\IP \in \mathcal{P}} \IE_{1, \IP}[\exp(-\eta|\varepsilon|)] \leq \sup_{\IP \in \mathcal{P}} \IE_{1,\IP}[\exp(\eta|\varepsilon|)]<\infty$, Danskin's Theorem~\citep{danskin1967} entails
\begin{align}
    \frac{\partial}{\partial \theta}\sup_{\IP\in \mathcal{P}} f(\theta, P) \Big|_{\theta \downarrow0}= \sup_{\IP \in \mathcal{P}} f^{(1)}(0, P)= 0, \label{eq:19}
\end{align}
where in the second equality we use that $f(0,P)=0$ for any $P\in \mathcal{P}$, and the third equality follows from (\ref{eq:fixed-P}). Similarly, $\frac{\partial}{\partial \theta}\sup_{\IP\in \mathcal{P}} f(\theta, P) \Big|_{\theta \uparrow0}= -\inf_{\IP \in \mathcal{P}} f^{(1)}(0, P)= 0$. Therefore, $\phi'(0)=-\rho\Tilde{d} <0$. On the other hand, since $\min\{\Var_0(\varepsilon), \sup_{\IP}\Var_{1,P}(\varepsilon)\}>0$, hence, in conjunction with $\phi$ being convex, there must exist $\kappa \in (0,1)$ such that $\log \kappa:= \inf_{\theta \geq 0}\phi(\theta)$. Consequently, from (\ref{eq:big-ub}), one obtains, 
\begin{align}
    \IP(I \in \mathcal{P}_4) \leq 16\sum_{j=M+1}^{\infty} \kappa^j = O(\kappa^M). \label{eq:choice-of-delta}
\end{align}
Suppose $\delta\in (0,1)$ be given. A choice of $M > \frac{\log 1/\delta}{\log 1/\kappa}$ ensures that $\IP(I \in \mathcal{P}_4) <\delta$. This completes the proof. 
\end{proof}

Finally, Theorem \ref{thm:single-watermark} is proved by invoking Theorem \ref{thm:single-watermark-app} and Proposition \ref{prop:corollary-to-Theorem-3}. 

We can further sharpen the $O((\rho \tilde{d})^{-1})$ rate in Theorem \ref{thm:single-watermark} to $O((\rho \tilde{d})^{-2})$ by assuming a mild condition: local sub-Gaussianity of the pivot statistics. The following result also trivially follows from Theorem \ref{thm:single-watermark-app} and Proposition \ref{prop:corollary-to-Theorem-3}, but is separately stated to highlight its importance.
\begin{lemma}\label{prop:subG}
   Grant the assumptions of Theorem \ref{thm:single-watermark}. If
        \begin{equation*}
            \max\{ \IE_0[\exp(r|\varepsilon|)] ,\sup_{\IP\in \mathcal{P}}\IE_{1,\IP}[\exp(r|\varepsilon|)] \} \leq \exp(r^2/2),
        \end{equation*}
        for all $r\in [0, \eta]$, then choosing $\rho>0$ such that $\rho\tilde{d}< \frac{5}{2}\eta$, then $|\hat{I} \Delta I_0|= O_{\IP}\big((\rho \tilde{d})^{-2}\big).$
\end{lemma}

\subsection{Proof of Theorem \ref{thm:multiple-watermark}}\label{se:proof-thm-multiple}

For convenience, we first restate the theorem.

\begin{theorem}\label{thm:multiple-watermark-app}
Assume that the null distribution of the pivot statistics is absolutely continuous with respect to the Lebesgue measure. Let the number of watermarked intervals $K$ be bounded, and Assumption \ref{ass:min-sep} be granted for the watermarked intervals $I_k, k\in [K]$.
Fix $\alpha \in (0,1)$, and recall the quantities defined in \texttt{WISER}\ described in Figure \ref{fig:algo}. Suppose that $\IE_0[|X-\mu_0|^{p}]<\infty$ for some $p \geq 2$, and let the block length $b=b_n$ satisfy $b_n = O(n^\upsilon)$, and $b_n / n^{1/p} \to \infty$, where $\upsilon>1/p$ is same as in Assumption \ref{ass:min-sep}. Moreover, suppose the threshold $\mathcal{Q}=\mathcal{Q}_n$ is selected so that $\IP_0(\max_{1\leq k \leq \lceil n/b\rceil} S_k> \mathcal{Q})= \alpha.$ Finally, assume $d\geq c$ for some constant $c>0$, and 
\begin{align}\label{eq:alt-mean-ub}
    \sup_{\IP \in \mathcal{P}} \IE_{1,\IP}[X] < \infty.
\end{align}
 Then, given $\varepsilon>0$, under the assumptions of Theorem \ref{thm:single-watermark}, there exists $M_{\varepsilon} \in \R_+$, independent of $n, K,$ and $d$, and $\rho>0$, such that \fancyname\ applied with hyper-parameters $b$ and $\rho$ satisfies
\allowdisplaybreaks \begin{align}
    \liminf_{n\to \infty}\IP\big(\hat{K}=K, \ \max_{k\in [K]}|\hat{I}_k \Delta I_k|< M_{\varepsilon} {d}^{-1} \big) \geq 1-\varepsilon.
\end{align}

\end{theorem}

Let $\widetilde{\mathcal{B}}=\{1\leq k \leq \lceil n/b \rceil: B_k \subseteq I_j \text{ for some } j\in [K]\}$. Our proof proceeds through a series of arguments, each carefully orchestrated to establish the validity of the corresponding steps of our algorithm. We comment that subsequently, all statements involving $n$ but without a limit attached to it are meant to be considered for all sufficiently large values of $n$. 

\textbf{Step 1: Validity of first stage thresholding.}

In this step, we show that 
\allowdisplaybreaks \begin{align}
    \IP(\min_{k\in \widetilde{\mathcal{B}}} S_k > \mathcal{Q}) \to 1, \text{ as $n\to \infty$.} \label{eq:first-stage-thresholding}
\end{align}
To begin with, let $\tau$ be a constant as defined in Proposition \ref{lem:eqdv-condn}, and let
\begin{align}\label{eq:dv-condn-appendix}
   \kappa:= \inf_{\theta \geq 0} \theta (\mu_0+ \tau d) + \log \sup_{\IP} \IE_{1,\IP}[\exp(-\theta X)] <0,
\end{align}
where the inequality follows from Proposition \ref{lem:eqdv-condn}.
Note that
\allowdisplaybreaks \begin{align}
    \limsup_{n\to \infty}\max_{k\in \widetilde{\mathcal{B}}}\IP(S_k \leq \mathcal{Q}_n)^{1/b} &\leq \limsup_{n\to \infty}\inf_{\theta\geq 0} \exp(\theta \mathcal{Q}_nb_n^{-1} + \log \sup_{\IP \in \mathcal{P}} \IE_{1,\IP}[\exp(-\theta X)]) \nonumber \\
    &\leq \inf_{\theta\geq 0} \exp(\theta \mu_0 + \log \sup_{\IP \in \mathcal{P}} \IE_{1,\IP}[\exp(-\theta X)]) \label{eq:appl-of-prop-1}\\
    &\leq \exp(\kappa)<1, \label{eq:less-than-1}
\end{align}
where (\ref{eq:appl-of-prop-1}) is obtained through an application of Proposition \ref{prop:1}, and (\ref{eq:less-than-1}) follows from (\ref{eq:dv-condn-appendix}). Since $\kappa<0$, one has $\frac{n}{b}\exp(\kappa b)\to 0$ as $n\to \infty$, and consequently
\[\IP(\min_{k\in \widetilde{\mathcal{B}}} S_k \leq \mathcal{Q}) \leq \frac{n}{b}\max_{k\in \widetilde{\mathcal{B}}}\IP(S_k \leq \mathcal{Q}_n) \to 0, \text{ as $n\to \infty$},\]
thereby establishing (\ref{eq:first-stage-thresholding}).

\textbf{Step 2. Estimation of the number of watermarked regions through the set $M$.}

Recall $M$ from the Step 1 of Subroutine \texttt{Refined\_Local\_Search} in Algorithm \ref{alg:watermark-search}. In this step of our proof, we will prove $\IP(\hat{K}=K)\to 1,$ which will also imply that $|M|$ is even with probability approaching $1$. Therefore, we may be excused for assuming that $|M|$ is even. 

Let $C_1,\ldots, C_{\hat{K}}$ be the disjoint set of intervals in $M$, with $C_j= [(s_{2j-1}-1)b+1, s_{2j}b]$. Note that for each $k\in \mathcal{B}$ such that $S_k > \mathcal{Q}$, $B_k \subseteq C_j$ for some $j$. Let $\widetilde{\mathcal{B}}_j = \{k \in \widetilde{\mathcal{B}}: B_k\subseteq I_j\}$, $j\in [K]$. We remark that 

Clearly, $\widetilde{\mathcal{B}}=\cup_{j=1}^K \widetilde{\mathcal{B}}_j$, and $\widetilde{\mathcal{B}}_j$ are disjoint. Therefore, in light of the construction of $M$ from blocks surpassing the threshold $\mathcal{Q}$, it follows,
\begin{align*}
    \IP(\min_{k\in \widetilde{\mathcal{B}}} S_k > \mathcal{Q}) = \IP(\min_{j\in [K]} \min_{k \in \widetilde{\mathcal{B}}_j} S_k > \mathcal{Q})
    &\leq \IP( \text{for each $j\in [K]$, there exists $i_j \in [\hat{K}]$ such that } \widetilde{\mathcal{B}}_j \subseteq C_{i_j}),
\end{align*}
which implies, in light of (\ref{eq:first-stage-thresholding}),
\allowdisplaybreaks \begin{align}
    \IP(A_n) \to 1, \text{ as $n\to \infty$, where, } A_n:=\{\text{for each $j\in [K]$, there exists $i_j \in [\hat{K}]$ such that } \widetilde{\mathcal{B}}_j \subseteq C_{i_j}\}. \label{eq:prob-A_n}
\end{align}
It is crucial to note that since both $\widetilde{\mathcal{B}}_j$ and $C_j$'s are defined to occur from left-to-right and since $C_j$'s are connected intervals, under the event $A_n$ it also holds that $i_1\leq i_2\leq\ldots \leq i_K$.
At this stage, the relationship between $\hat{K}$ and $K$ is still not entirely clear. Subsequently, we will show that under the event $A_n$, the mapping $j\mapsto i_j$ is injective, establishing that $\hat{K}\geq K$ with high probability. To that end, suppose there exists $k_1 < k_2\in [K]$ such that $i_{k_1}=i_{k_2}$. Since $C_{i_{k_1}}$ is a connected interval, $i_{k_1}=i_{k_2}$ implies that that $i_{k_1} = i_{k_1+1}$. Let $\IP_{E,F}(\cdot)=\IP(\cdot \ \cap E \ \cap F )$ for any events $E$, $F$. Consider the following series of inequalities.
\allowdisplaybreaks \begin{align}
    & \IP_{A_n}(\text{There exists $k\in[K-1]$ such that } C_{i_k}=C_{i_{k+1}}) \nonumber\\
    \leq{} & \IP_{A_n}(\text{There exists $k\in [K-1]$ such that } (I_{k,R}, I_{k+1, L}) \subseteq C_{i_k} )\nonumber\\
    \leq{} & \IP_{A_n}(\text{There exists $k$ such that } \min_{l\in (\lceil I_{k,R}/b \rceil, \lfloor I_{k+1, L}/b \rfloor )} S_l > \mathcal{Q})\nonumber\\
    \leq{} & \IP_0( \sum_{k=1}^{n/b} I\{S_k > \mathcal{Q}\} \geq C_0 \sqrt{\log n} ), \label{eq:A_n-intersect-ub}
\end{align}
where the $\IP_0$ in final inequality appears since for $l\in (\lceil I_{k,R}/b \rceil, \lfloor I_{k+1, L}/b \rfloor )$, the region $B_l$ is unwatermarked; the $\sqrt{\log n}$ appears by invoking Assumption \ref{ass:min-sep} and noting that $b^{-1}(I_{k+1, L} - I_{k,R})\geq C_0 \sqrt{\log n}$. An application of Proposition \ref{prop:2} to (\ref{eq:A_n-intersect-ub}) entails, in view of (\ref{eq:prob-A_n}), that,
\begin{align*}
    \IP_{A_n}(\bar{B}_n) \to 1, \text{ as $n\to \infty$, where $\bar{B}_n=\{\text{$C_{i_k}$ and $C_{i_s}$ are disjoint if $i_k\neq i_s$}\}$.}
\end{align*}
Clearly, this implies that $\IP_{A_n}(\hat{K} \geq K)\to 1$ as $n\to \infty$, which also produces $\IP(\hat{K} \geq K)\to 1$ as $n\to \infty$. On the other hand, if $\hat{K}>K$, then under the event $A_n\cap \bar{B}_n$, there exists $j\in [\hat{K}]$ such that $C_j$ and $ \cup_{s \in \mathcal{B}} B_s$ are disjoint. Consequently, it must be true that $|C_j \cap (\cup_{k=1}^K I_j)| \leq b.$ Note that, by construction of $C_j$'s in \texttt{WISER}, $|C_j| \geq c b\sqrt{\log n}$. Therefore it must be true that there are at least $2^{-1}c \sqrt{\log n}$ many $s$'s such that $B_s \cap C_j \cap (\cup_{k=1}^K I_j)=\phi$, and $S_s > \mathcal{Q}$. Hence it follows from Proposition \ref{prop:2} that
\[ \IP_{A_n,\bar{B}_n}(\hat{K}>K) \to 0, \text{ as $n\to \infty$}, \]
which immediately implies that
\allowdisplaybreaks \begin{align}
    \IP(\hat{K}=K)\to 1\text{ as $n\to \infty$}. \label{eq:correct-est-K}
\end{align}

\textbf{Step 3. Choice of $\Tilde{d}$ and $\rho$.}

Recall $\Tilde{d}$ from Step 5 of Subroutine \texttt{Refined\_Local\_Search} in Algorithm \ref{alg:watermark-search}. In this step, we establish that there exists $\rho>0$, such that $d> 2\rho \Tilde{d}$ with high probability. In conjunction to $\Tilde{d}$, also define 
\[ d^\dagger = \frac{\sum_{j=1}^K\sum_{s\in I_j} (X_s- \mu_0)}{\sum_{j=1}^K |I_j|}.  \]
Let the event $\{\hat{K}=K\}$ be denoted as $E_n$. Under $E_n$, by construction of $D_j$, $\IP_{A_n, B_n, E_n}(I_j \subseteq D_j \text{ for all $j\in [K]$})\to 1$ as $n\to \infty$. Call the latter event as $F_n$. Observe that under $E_n \cap F_n$, it holds 
\allowdisplaybreaks \begin{align}\label{eq:bound-on-Dj}
    \sum_{j=1}^K |I_j| + 2C b \log n \geq \sum_{j=1}^{\hat{K}} |D_j| \geq \sum_{j=1}^K |I_j| + C b \log n
\end{align}for some $C>0$. Therefore, under the same event, it follows
\allowdisplaybreaks \begin{align}
    \Tilde{d} \leq  d^\dagger \frac{\sum_{j=1}^K |I_j| }{\sum_{j=1}^K |I_j| + C b \log n} + \frac{\sum_{s \in \cup_j ( I_j^c \cap D_j)} (X_s - \mu_0)}{\sum_{j=1}^K |I_j| + C b \log n}. \label{eq:ub-d-tilde}
\end{align}
We first tackle the second term in the upper-bound in (\ref{eq:ub-d-tilde}). Let $D^\dagger_j = [(I_{j,L}- \lfloor C b \log^{3/2} n \rfloor) \vee 1 , (I_{j,R}+ \lfloor C b \log^{3/2} n \rfloor) \wedge n ]$.  Again, by construction of $D_j$ as well as from Assumption \ref{ass:min-sep}, for all sufficiently large $n$ it follows 
\[ \IP_{A_n, \bar{B}_n, E_n}(D_j \subseteq D_j^{\dagger}, D_i^\dagger \cap D_j^\dagger =\phi \text{ for $i\neq j$})\to 1. \] 
Call the above event as $G_n$. Fix $\varepsilon>0$, and consider the following implications.
\allowdisplaybreaks \begin{align}
    &\IP_{A_n, \bar{B}_n, E_n}\bigg(\frac{\sum_{s \in \cup_j ( I_j^c \cap D_j)} (X_s - \mu_0)}{\sum_{j=1}^K |I_j| + C b \log n}>\varepsilon\bigg)\nonumber\\
    \leq{} & \IP_{A_n, \bar{B}_n, E_n, G_n}\bigg(\frac{\sum_{s \in \cup_j ( I_j^c \cap D_j^{\dagger})} |X_s - \mu_0|}{\sum_{j=1}^K |I_j| + C b \log n}>\varepsilon\bigg) + o(1) \nonumber\\
    \leq{} & \IP\bigg(\frac{\sum_{s \in \cup_j ( I_j^c \cap D_j^{\dagger})} |X_s - \mu_0|}{\sum_{j=1}^K |I_j| + C b \log n}>\varepsilon\bigg) + o(1) \nonumber\\
    \leq{} & \frac{O(b \log^{3/2}n)}{\varepsilon^2 n\log^2 n} + o(1)=o(1), \label{eq:d-second-term}
\end{align}
where the inequality in the final assertion follows from $|\cup_{j=1}^K (I_j^c \cap D_j^\dagger)|\lesssim b \log^{3/2}n$. Therefore, (\ref{eq:ub-d-tilde}) and (\ref{eq:d-second-term}) jointly yield
\allowdisplaybreaks \begin{align}\label{eq:first-ub-on-Tilde_d}
    \IP_{A_n, \bar{B}_n, E_n, F_n}(\Tilde{d} \leq 2d^\dagger) \to 1, \text{ as $n\to \infty$}.
\end{align}
Next, we focus on controlling $d^\dagger$ by $d$. To that end, we resort to an argument through moment generating functions. On one hand, (\ref{eq:dv-condn-appendix}) entails
\allowdisplaybreaks \begin{align}
    \IP(d^\dagger \leq \tau d)\leq \inf_{\theta\geq 0} \big(\exp(\theta(\mu_0+\tau d) + \log \sup_{\IP\in \mathcal{P}} \IE_{1,\IP}[\exp(-\theta X)])\big)^{\sum_{j=1}^K |I_j|}\leq \exp(\kappa\sum_{j=1}^K |I_j|)\to 0.\label{eq:d-dagger-vs-d}
\end{align}
On the other hand, in light of (\ref{eq:alt-mean-ub}) and $d\geq c$,  choose $$\nu>\frac{\sup_{\IP\in \mathcal{P}} \IE_{1,\IP} [X] - \mu_0 }{c} \vee \frac{\tau}{4},$$ and write:
\allowdisplaybreaks \begin{align}
    \IP(d^{\dagger} \geq 2\nu d)\leq & \inf_{\theta\geq 0} \Big(\exp(-2\theta \nu c + \log \sup_{\IP\in \mathcal{P}}\IE_{1,\IP}[\exp(\theta(X-\mu_0))]\Big)^{\sum_{j=1}^K |I_j|}.  \label{eq:bound on d-dagger}
\end{align}
Echoing the argument in the proof of Theorem \ref{thm:single-watermark}, define 
\[ g(\theta, P \ ; c) =-2\theta \nu c + \log \IE_{1,\IP}[\exp(\theta(X-\mu_0))], \ \widetilde{g}(\theta \ ; c)=\sup_{\IP\in \mathcal{P}} g(\theta, P).\]
Since $\mathcal{P}$ is compact and $\sup_{\IP\in \mathcal{P}} \IE_{1,\IP}[\exp(\eta|X-\mu_0|]< \infty$, Danskin's Theorem~\citep{danskin1967} applies and produces
\allowdisplaybreaks \begin{align}
    \widetilde{g}^{(1)}_+(0, P\ ; c)=\frac{\partial}{\partial \theta} \sup_{\IP\in \mathcal{P}} g(\theta, P\ ; c) \Big |_{\theta \downarrow 0} = \sup_{\IP \in \mathcal{P}} g^{(1)}(0,P\ ; c) = -2\nu d + \sup_{\IP\in \mathcal{P}}  \IE_{1,\IP}[X-\mu_0] \leq -\nu d<0, \label{eq:right-derivative-ub}
\end{align}
where the final inequality is derived via (\ref{eq:alt-mean-ub}). Moreover, similar to (\ref{eq:convexity}) it can be argued that $\widetilde{g}(\theta)$ is convex in $\theta$. Finally, since $\widetilde{g}(0\ ; c)=0$, (\ref{eq:right-derivative-ub}) coupled with its convexity implies that $ \varphi(c):= \inf_{\theta\geq 0} \widetilde{g}(\theta\ ; c)<0$. In view of this, (\ref{eq:bound on d-dagger}) results in 
\allowdisplaybreaks 
\begin{equation}
    \upd{\IP(d^\dagger \geq 2\nu d)\leq \exp(\varphi(c)\sum_{j=1}^K |I_j|)\to 0 \text{ as $n\to \infty$,}} \label{eq:d-dagger-more-than-d}
\end{equation}
where the limiting assertion is due to $\sum_{j=1}^K |I_j| \geq c\sqrt{n}.$ Finally, (\ref{eq:first-ub-on-Tilde_d}) and (\ref{eq:d-dagger-more-than-d}) jointly indicate that
\allowdisplaybreaks \begin{align}
    \IP_{A_n, B_n, E_n, F_n}(\Tilde{d}\leq 4\nu d)\to 1, \text{ as $n\to \infty$.}
\end{align}
Subsequently, we choose $\rho=(8\nu)^{-1}$. In conclusion to this step, (\ref{eq:d-second-term}) along with (\ref{eq:d-dagger-vs-d}) establishes 
\[ \IP_{A_n, B_n, E_n, F_n}(G_n) \to 1 \text{ as $n\to \infty$}, \  G_n:=\{\tau d\leq \Tilde{d}\leq 4\nu d\}. \]

\textbf{Step 4. Localization of watermarked intervals.}

In this step, we establish the validity of our localized estimates $\hat{I}_j$. In Step 3, we argued that 
\[ \IP_{A_n, B_n, E_n}(I_j \subseteq D_j \subseteq D_j^\dagger \text{ for each $j\in [K]$}) \to 1 \text{ as $n\to \infty$.} \]
Call the above event as $\Tilde{F}_n$. Under $\Tilde{F}_n$, it is immediate that 
\[ \hat{I}_j(\tilde{d})=\argmin_{s\in L_j, t\in R_j} \sum_{k\in {D}_j \setminus [s, t] } (X_k -\mu_0 - \rho \tilde{d})= \argmin_{s\in L_j, t\in R_j} \sum_{k\in {D}_j^\dagger \setminus [s, t] } (X_k -\mu_0 - \rho \tilde{d}), \]
since the operator $\sum_{k\in {D}_j^\dagger \setminus [s, t] }$ can be decomposed into $\sum_{k\in {D}_j \setminus [s, t] }+ \sum_{k\in {D}_j^\dagger \setminus D_j }$.

We proceed towards applying Theorem \ref{thm:single-watermark} to $\hat{I}_j(\tilde{d})$. However, note that $\tilde{d}$ is a random quantity, so special care must be accorded to its treatment. To that end, define
\[ \hat{I}_j(\sigma) = \argmin_{s\in L_j, t\in R_j} \sum_{k\in {D}_j^\dagger \setminus [s, t] } (X_k -\mu_0 - \rho \sigma), \ \sigma \in [\tau d, 4\nu d].  \]
Fix $j\in [K]$. For $M>0$, let $D_M:=\{I: |I \Delta I_j|>M\}$. For a candidate interval $I=[s,t]$, let 
\[ \widetilde{V}_I(\sigma) =  \sum_{k\in {D}_j^\dagger \setminus [s, t] } (X_k -\mu_0 - \rho \sigma). \]Clearly, by definition of $G_n$,
\allowdisplaybreaks \begin{align}
    &\IP_{A_n, B_n, E_n, \tilde{F}_n, G_n}(|\hat{I}_j(\tilde{d}) \ \Delta \ I_n|> M)\nonumber\\ 
    \leq{} & \IP_{A_n, B_n, E_n, \tilde{F}_n, G_n}\Big(\sup_{\sigma \in [\tau d, 4\nu d] }|\hat{I}_j(\sigma) \ \Delta \ I_n|> M\Big) \nonumber\\
    \leq{} &\IP_{A_n, B_n, E_n, \tilde{F}_n, G_n}\Big(\text{ There exists $\sigma\in [\tau d, 4\nu d]$ such that} \inf_{s \in L_j, t\in R_j, I\in D_M} \widetilde{V}_I(\sigma) < \widetilde{V}_{I_j}(\sigma) \Big) \nonumber\\
    \leq{} & \IP_{A_n, B_n, E_n, \tilde{F}_n, G_n}\Big(\text{ There exists $\sigma\in [\tau d, 4\nu d]$ such that} \inf_{I\in D_M} \widetilde{V}_I(\sigma) < \widetilde{V}_{I_j}(\sigma) \Big) \nonumber\\
    \leq{} & \IP_{A_n, B_n, E_n, \tilde{F}_n, G_n}\big(\max_{I: x_1+x_3>M} \frac{S_1^{\varepsilon} -S_3^\varepsilon}{x_1+x_3} > (\frac{1}{2} \wedge \frac{\tau}{8\nu}) d\big) \label{eq:appl-of-single-thm}\\
    \leq{} & \IP_{A_n, B_n, E_n, \tilde{F}_n, G_n}\big(\max_{I: x_1+x_3>M} \frac{S_1^{\varepsilon} -S_3^\varepsilon}{x_1+x_3} >  \frac{\tau}{8\nu} d\big) \label{eq:use-of-G_n} \\
    \leq{} & \IP(\max_{I: x_1+x_3>M} \frac{S_1^{\varepsilon} -S_3^\varepsilon}{x_1+x_3} > \frac{\tau}{8\nu} d).\label{eq:to-be-shown-small}
\end{align}
Here, (\ref{eq:appl-of-single-thm}) follows by recalling the notations in the proof of Theorem \ref{thm:single-watermark} and following the arguments (\ref{eq:d-bound-lb})-(\ref{eq:lb}) after observing $\sigma\in[\tau d, 4\nu d]$ implies $d- (8\nu)^{-1}\sigma \geq \frac{d}{2}$. Moreover, (\ref{eq:use-of-G_n}) also follows from $4\nu d\geq \sigma\geq \tau d$. Finally, (\ref{eq:to-be-shown-small}) is derived from $\IP(A\cap B)\leq \IP(A)$; in particular, arguments of Theorem \ref{thm:single-watermark} can be followed verbatim to obtain that 
\[\IP(\max_{I: x_1+x_3>M} \frac{S_1^{\varepsilon} -S_3^\varepsilon}{x_1+x_3} >  \frac{\tau}{8\nu} d) \leq \xi^M \text{ for some $\xi<1$.} \]
Note that in the above assertion we have used the fact that $d\geq c$ to decouple $\xi$ from $d$. Given arbitrary $\varepsilon>0$, $M_{\varepsilon}$ can be chosen to ensure $\xi^{M_{\varepsilon}}<\varepsilon$, and through $\kappa$, this choice of $M_{\varepsilon}$ solely depends on the constants $\nu$, $\tau$, $c$, and $\mu_0$, apart from the quantity $\varepsilon$. Therefore, in view of the number of watermarked intervals $K=O(1)$, we obtain that there exists $M_{\varepsilon}$  independent of $n, K$ and $d$ such that 
\allowdisplaybreaks \begin{align}
    &\IP_{A_n, B_n, E_n, \tilde{F}_n, G_n}(|\hat{I}_j(\tilde{d}) \ \Delta \ I_n|> M_{\varepsilon} \text{ for $j\in [K]$}) \leq \varepsilon \nonumber\\ 
    \implies &\liminf_{n\to \infty} \IP_{A_n, B_n, E_n, \tilde{F}_n, G_n}(|\hat{I}_j(\tilde{d}) \ \Delta \ I_n|\leq  M_{\varepsilon} \text{ for $j\in [K]$}) \geq 1-\varepsilon, \label{eq: final-lb}
\end{align}
where in (\ref{eq: final-lb}) we invoke $$\lim_{n\to\infty} \IP(A_n \cap B_n \cap E_n \cap \tilde{F}_n \cap G_n) =1.$$ Recalling that $E_n=\{\hat{K}=K\}$ completes the proof. 

\subsection{\upd{Proof of Theorem \ref{thm:multiple-oracle}}}

\upd{The proof of Theorem \ref{thm:multiple-oracle} follows along similar lines to that of Theorem \ref{thm:single-watermark}, but due to the presence of multiple watermarked patches, it requires some careful technical setup.}

\begin{figure}[htbp]
    \centering
    \includegraphics[width=\linewidth]{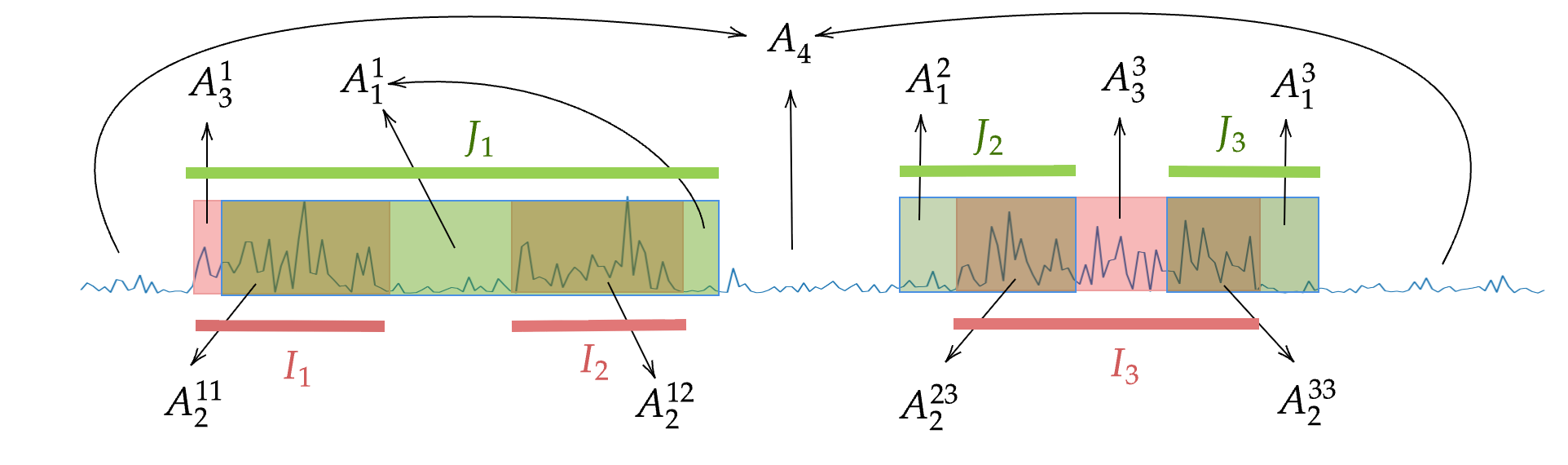}
    \caption{\upd{An instance of Case $\mathcal P^{\varnothing}_{(1,2,3,3,3,3)}$ from Table \ref{tab:case-k_3}.}}
    \label{fig:multiple-wm-oracle}
\end{figure}
    
\upd{Let $I_1, \ldots, I_K$ denote the true watermarked patches. Correspondingly, we denote by $J:=\{J_1, J_2, \ldots, J_K\} \in \mathcal{I}_K$ to be a set of candidate patches. Define the following:
    \begin{align*}
        A_1^k(J) &= J_k \cap (\cup_{i=1}^K I_i)^c, \ k \in [K],\\
        A_2^{ij}(J) &= J_i \cap I_j, \ i,j \in [K], \\
        A_3^k(J) &= I_k \cap (\cup_{i=1}^K J_i)^c, \ k \in [K], \\
        A_4(J) &= (\cup_{k=1}^K I_k)^c \cap (\cup_{k=1}^K J_k)^c,
    \end{align*}
    and correspondingly, denote $x_1^k = |A_1^k|$, $x_2^{ij}= |A_2^{ij}|$, $x_3^k = |A_3^k|$ and $x_4 = |A_4|$. Finally, denote $S_1^k =\sum_{i \in A_1^k} X_i$, $S_1^{k, \varepsilon}=\sum_{i \in A_1^k} \varepsilon_i$, and likewise define $S_2^{ij}$, $S_3^k$, $S_4$, and $S_2^{ij, \varepsilon}, S_3^{k, \varepsilon}$, and $S_4^\varepsilon$. Note that some of the sets $A$ be empty; for example, in the set-up of Figure \ref{fig:multiple-wm-oracle}, the set $A_3^2$ is empty and hence is not marked; in this case the entire watermarked patch $I_2$ happens to be covered by the candidate interval $J_1$. With these notations, it follows that the sets in $\{A_1^k\}$ are disjoint; a similar observation holds true for $\{A_3^k\}$. Moreover, it is straightforward to verify
    \begin{align*}
        |(\cup_{k=1}^K I_K)^c|&= x_4 + \sum_{k=1}^K x_1^k, \text{ and}, \\
        |(\cup_{k=1}^K J_K)^c|&= x_4 + \sum_{k=1}^K x_3^k.
    \end{align*}
    For a candidate set of intervals $J= \{J_1, \ldots, J_K\}$, let \[V_J := \sum_{k \notin \cup_{i=1}^K J_i}(X_k - \mu_0 -\rho \tilde{d}) = S_4 + \sum_{k=1}^K S_3^k- (\mu_0 + \rho \tilde{d})(x_4 + \sum_{k=1}^K x_3^k). \]
    Let $I_0 :=\{I_1, \ldots, I_K\}$ be the collection of true watermarked patches. The identifiability of $I_0$ and $J$ vis-à-vis $V_J$ and $V_{I_0}$ is guaranteed by the definition of $\mathcal{I}$. Then, a generalized form of the basic inequality \eqref{eq:lb} can be produced as:
    \begin{align}
        V_J - V_I &= \sum_{k=1}^K \big( S_3^k - S_1^k - (\mu_0 + \rho \tilde{d}) (x_3^k - x_1^k) \big) \nonumber\\
        &= \sum_{k=1}^K \big( S_3^{k, \varepsilon} - S_1^{k, \varepsilon} + \sum_{t \in A_3^k} (\mu_t - \mu_0) + \rho \tilde{d}) (x_1^k - x_3^k) \big) \nonumber\\
        & \overset{(a)}{\ge} \sum_{k=1}^K (S_3^{k,\varepsilon} - S_1^{k,\varepsilon} + x_3^k(d- \rho \tilde{d}) + x_1^k \rho\tilde{d})\nonumber\\
        & \overset{(b)}{\ge} \sum_{k=1}^K (S_3^{k,\varepsilon} - S_1^{k,\varepsilon} + (x_1^k + x_3^k)\rho \tilde{d}), \label{eq:new-basic-ineq}
    \end{align}
    where (a) follows from Assumption \ref{ass:alt-mean} and (b) uses $d \geq 2\rho \Tilde{d}$. Henceforth, we will leverage \eqref{eq:new-basic-ineq} to control the error metric $\sum_{j\in [K]} |J_j \Delta I_j|$ for any set of candidate patches $\{J_1, J_2, \ldots, J_K\} \in \mathcal{I}_K$.}

    \upd{To that end, recall $I^\star$ from Assumption \ref{ass:min-sep}, and let $M$ be such that $M < I^\star$ for all sufficiently large $n\in \N$. We will provide a principled choice of $M$ at an appropriate place in our subsequent arguments.}

    \upd{Consider the following sequence of inequalities:
    \begin{align}
        &\ \IP\Big(\min_{J: \sum_{k=1}^K |I_k \Delta J_k|>M } V_J - V_I \leq 0 \Big)\nonumber \\
        \overset{(a)}{=} & \IP\Big(\min_{J: \sum_{k=1}^K |I_k \Delta J_k|>M, \ |I_k \Delta J_k|=x_1^k + x_3^k \text{ for } k\in [K] } V_J - V_I \leq 0 \Big) \nonumber\\ & \hspace*{2cm}+ \IP\Big(\min_{J: \sum_{k=1}^K |I_k \Delta J_k|>M, \ x_2^{k \ell} >0 \text{ for some $k \ne \ell \in [K]$} } V_J - V_I \leq 0 \Big)\nonumber \\
        \overset{(b)}{\leq} & \IP\Big(\min_{J: |I_k \Delta J_k|=x_1^k + x_3^k \text{ for } k\in [K], \ M <\sum_{k=1}^K (x_1^k + x_3^k) < I^\star } V_J - V_I \leq 0 \Big) + \IP\Big(\min_{J: x_2^{k \ell} >0 \text{ for some $k \ne \ell \in [K]$} } V_J - V_I \leq 0 \Big) \label{eq:good-cases-bad-cases},
    \end{align}
    where (a) uses Lemma \ref{lemma:good-decompose}.  For the first term in \eqref{eq:good-cases-bad-cases}, we recall the five sub-classes used in the proof of Theorem \ref{thm:single-watermark-app}. Using $|I_k \Delta J_k|=x_1^k + x_3^k$ for every $k\in [K]$, this term can be controlled by looking at the $5^K$ cases, with each candidate interval $J_k$ satisfying one of the five classes $\mathcal{P}_1, \mathcal{P}_2, \ldots, \mathcal{P}_5$. In particular, the choice of $M$ is specified here as argued in Theorem \ref{thm:single-watermark-app}; since the arguments carry over verbatim and do not provide any additional insight, we omit the niceties here.}

    \upd{Instead, we focus on the second term in \eqref{eq:good-cases-bad-cases}. Observe that 
    \begin{align}
        & \ \IP\bigg( \min_{J: x_2^{k \ell} >0 \text{ for some $k \ne \ell \in [K]$} } V_J - V_I \leq 0 \bigg) \nonumber\\
    \leq & \  \IP\bigg( \min_{J: x_2^{k \ell} >0 \text{ for some $k \ne \ell \in [K]$} }  \sum_{k=1}^K (S_3^{k,\varepsilon} - S_1^{k,\varepsilon} + (x_1^k + x_3^k)\rho \tilde{d}) \leq 0 \bigg) \nonumber\\
    \leq & \ \IP\bigg( \max_{J: x_2^{k \ell} >0 \text{ for some $k \ne \ell \in [K]$}} \frac{\sum_{k=1}^K (S_3^{k,\varepsilon} - S_1^{k,\varepsilon}) }{\sum_{k=1}^K (x_1^k + x_3^k)} \geq \rho \tilde{d} \bigg). \label{eq:oracle-2nd-term-daddy}
    \end{align}
    This term is a particular artifact of the multiple watermarked-patch case, and, therefore, requires some special attention. To alleviate the number of cases, we restrict the rest of the mathematical arguments to the $K=3$ case. Nevertheless, to ease the exposition for the reader, we find it necessary to introduce the required notations in their full, conceivable generality, before reducing it to the $K=3$ case to further ground our understanding.} 

    \upd{For a candidate set of intervals $\{J_1, \ldots, J_K \}$, let 
    \[ \mathcal{N}(J):= \{i: J_i \cap (\cup_{j=1}^K I_j) = \phi \}, \]
    and for $i \in [K]\setminus \mathcal{N}$, let
    \allowdisplaybreaks
    \begin{align*}
        a_i(J) = \min\{j: x_2^{ij}>0 \}, \ b_i(J) = \max\{j: x_2^{ij}>0\}.
    \end{align*}
    Lemma \ref{lemma:x2-ordering} instructs that $a_{i+1} \geq b_i$. For a set $\mathcal{N}\subset [K]$, if $[K] \setminus \mathcal{N}=\{i_1, \ldots, i_{K-|\mathcal{N}|}\}$, then, for a given set of integers $a_{i_1} \leq b_{i_1} \leq a_{i_2} \leq b_{i_2}\leq \ldots \leq a_{i_{K- |\mathcal{N}|}} \leq b_{i_{K- |\mathcal{N}|}}$ such that $(a_{i_k} , b_{i_k})=(i_k, i_k)$ for at least one $k\in [K- |\mathcal{N}|]$, denote:
    \[ \mathcal{P}^{\mathcal{N}}_{a_{i_1}, b_{i_1}, \ldots, a_{i_{K- |\mathcal{N}|}} , b_{i_{K- |\mathcal{N}|}}}:= \{J: \mathcal{N}(J)= \mathcal{N}, \ a_{i_k}(J) = a_{i_k}, b_{i_k}(J) = b_{i_k} \ \text{for $k\in [K- |\mathcal{N}|]$}  \}. \]
    This class exhausts $\{J: x_2^{k \ell} >0 \text{ for some $k \ne \ell \in [K]$} \}$. To convey the multitude of distinct cases we are practically dealing with, in Table \ref{tab:case-k_3}, we display all possible $\mathcal{P}$-classes for $K=3$.}

\begin{table}[!htbp]
\centering
\small
\begin{tabular}{c c p{10.8cm}}
\hline
$\mathcal N$ & $[3]\setminus \mathcal N$ & Admissible classes \\ \hline
\\

$\varnothing$ & $\{1,2,3\}$ &
$\mathcal P^{\varnothing}_{(1,1,1,1,1,1)},\ 
\mathcal P^{\varnothing}_{(1,1,1,1,1,2)},\ 
\mathcal P^{\varnothing}_{(1,1,1,1,1,3)},\ 
\mathcal P^{\varnothing}_{(1,1,1,1,2,2)},\ 
\mathcal P^{\varnothing}_{(1,1,1,1,2,3)},\ 
\mathcal P^{\varnothing}_{(1,1,1,1,3,3)},$
\newline
$\mathcal P^{\varnothing}_{(1,1,1,2,2,2)},\ 
\mathcal P^{\varnothing}_{(1,1,1,2,2,3)},\ 
\mathcal P^{\varnothing}_{(1,1,1,2,3,3)},\ 
\mathcal P^{\varnothing}_{(1,1,1,3,3,3)},\ 
\mathcal P^{\varnothing}_{(1,1,2,2,2,2)},\ 
\mathcal P^{\varnothing}_{(1,1,2,2,2,3)},$
\newline
$\mathcal P^{\varnothing}_{(1,1,2,3,3,3)},\ 
\mathcal P^{\varnothing}_{(1,1,3,3,3,3)},\ 
\mathcal P^{\varnothing}_{(1,2,2,2,2,2)},\ 
\mathcal P^{\varnothing}_{(1,2,2,2,2,3)},\ 
\mathcal P^{\varnothing}_{(1,2,2,2,3,3)},$
\newline
$\mathcal P^{\varnothing}_{(1,2,2,3,3,3)},\ 
\mathcal P^{\varnothing}_{(1,2,3,3,3,3)},\ 
\mathcal P^{\varnothing}_{(1,3,3,3,3,3)},\ 
\mathcal P^{\varnothing}_{(2,2,2,2,2,2)},\ 
\mathcal P^{\varnothing}_{(2,2,2,2,2,3)},\ 
\mathcal P^{\varnothing}_{(2,2,2,2,3,3)},$
\newline
$\mathcal P^{\varnothing}_{(2,2,2,3,3,3)},\ 
\mathcal P^{\varnothing}_{(2,2,3,3,3,3)},\ 
\mathcal P^{\varnothing}_{(2,3,3,3,3,3)},\ 
\mathcal P^{\varnothing}_{(3,3,3,3,3,3)}$

\\ \hline
\\

$\{1\}$ & $\{2,3\}$ &
$\mathcal P^{\{1\}}_{(1,1,1,1)},\ 
\mathcal P^{\{1\}}_{(1,1,1,2)},\ 
\mathcal P^{\{1\}}_{(1,1,1,3)},\ 
\mathcal P^{\{1\}}_{(1,1,2,2)},\ 
\mathcal P^{\{1\}}_{(1,1,2,3)},\ 
\mathcal P^{\{1\}}_{(1,1,3,3)},$
\newline
$\mathcal P^{\{1\}}_{(1,2,2,2)},\ 
\mathcal P^{\{1\}}_{(1,2,2,3)},\ 
\mathcal P^{\{1\}}_{(1,2,3,3)},\ 
\mathcal P^{\{1\}}_{(1,3,3,3)},\ 
\mathcal P^{\{1\}}_{(2,2,2,2)},\ 
\mathcal P^{\{1\}}_{(2,2,2,3)},$
\newline
$\mathcal P^{\{1\}}_{(2,3,3,3)},\ 
\mathcal P^{\{1\}}_{(3,3,3,3)}$

\\ \hline
\\

$\{2\}$ & $\{1,3\}$ &
$\mathcal P^{\{2\}}_{(1,1,1,1)},\ 
\mathcal P^{\{2\}}_{(1,1,1,2)},\ 
\mathcal P^{\{2\}}_{(1,1,1,3)},\ 
\mathcal P^{\{2\}}_{(1,1,2,2)},\ 
\mathcal P^{\{2\}}_{(1,1,2,3)},\ 
\mathcal P^{\{2\}}_{(1,2,2,2)},$
\newline
$\mathcal P^{\{2\}}_{(1,2,2,3)},\ 
\mathcal P^{\{2\}}_{(1,2,3,3)},\ 
\mathcal P^{\{2\}}_{(1,3,3,3)},\ 
\mathcal P^{\{2\}}_{(2,2,2,2)},\ 
\mathcal P^{\{2\}}_{(2,2,2,3)},\ 
\mathcal P^{\{2\}}_{(2,2,3,3)},$
\newline
$\mathcal P^{\{2\}}_{(2,3,3,3)},\ 
\mathcal P^{\{2\}}_{(3,3,3,3)}$

\\ \hline
\\

$\{3\}$ & $\{1,2\}$ &
$\mathcal P^{\{3\}}_{(1,1,1,1)},\ 
\mathcal P^{\{3\}}_{(1,1,1,2)},\ 
\mathcal P^{\{3\}}_{(1,1,1,3)},\ 
\mathcal P^{\{3\}}_{(1,1,2,3)},\ 
\mathcal P^{\{3\}}_{(1,1,3,3)},\ 
\mathcal P^{\{3\}}_{(1,2,2,2)},$
\newline
$\mathcal P^{\{3\}}_{(1,2,2,3)},\ 
\mathcal P^{\{3\}}_{(1,2,3,3)},\ 
\mathcal P^{\{3\}}_{(1,3,3,3)},\ 
\mathcal P^{\{3\}}_{(2,2,2,2)},\ 
\mathcal P^{\{3\}}_{(2,2,2,3)},\ 
\mathcal P^{\{3\}}_{(2,2,3,3)},$
\newline
$\mathcal P^{\{3\}}_{(2,3,3,3)},\ 
\mathcal P^{\{3\}}_{(3,3,3,3)}$

\\ \hline
\\

$\{1,2\}$ & $\{3\}$ &
$\mathcal P^{\{1,2\}}_{(1,1)},\ 
\mathcal P^{\{1,2\}}_{(1,2)},\ 
\mathcal P^{\{1,2\}}_{(1,3)},\ 
\mathcal P^{\{1,2\}}_{(2,2)},\ 
\mathcal P^{\{1,2\}}_{(2,3)}$

\\ \hline
\\

$\{1,3\}$ & $\{2\}$ &
$\mathcal P^{\{1,3\}}_{(1,1)},\ 
\mathcal P^{\{1,3\}}_{(1,2)},\ 
\mathcal P^{\{1,3\}}_{(1,3)},\ 
\mathcal P^{\{1,3\}}_{(2,3)},\ 
\mathcal P^{\{1,3\}}_{(3,3)}$

\\ \hline
\\

$\{2,3\}$ & $\{1\}$ &
$\mathcal P^{\{2,3\}}_{(1,2)},\ 
\mathcal P^{\{2,3\}}_{(1,3)},\ 
\mathcal P^{\{2,3\}}_{(2,2)},\ 
\mathcal P^{\{2,3\}}_{(2,3)},\ 
\mathcal P^{\{2,3\}}_{(3,3)}$

\\ \hline

\end{tabular}
\caption{\upd{All classes for $K=3$ under the partition
$\{\mathcal P^{\mathcal N}_{a_{i_1},b_{i_1},\ldots,a_{i_m},b_{i_m}}\}$.}}
 \label{tab:case-k_3}
\end{table}

    \upd{Subsequently, we focus on the hardest case $\mathcal{N}=\phi$. We focus on the following few representative cases:
    \[P^{\varnothing}_{(1,1,1,1,1,1)}, \mathcal P^{\varnothing}_{(1,1,1,1,1,2)}, \mathcal P^{\varnothing}_{(1,1,1,1,2,3)}. \] 
    The techniques from these cases can be verbatim applied to deal with the other cases. For example, the case $\mathcal P^{\varnothing}_{(1,2,3,3,3,3)}$ in Figure \ref{fig:multiple-wm-oracle} can be seen to be symmetric to $\mathcal P^{\varnothing}_{(1,1,1,1,2,3)}$, and can be dealt with similarly. Subsequently, we deal with these $3$ cases one-by-one. At this stage, it is also important to recall the cases $\mathcal{P}_1, \mathcal{P}_2 \ldots, \mathcal{P}_5$ from the proof of Theorem \ref{thm:single-watermark-app}.}

\textbf{\upd{Case 1: $P^{\varnothing}_{(1,1,1,1,1,1)}$}}

\upd{Note that, for Case $P^{\varnothing}_{(1,1,1,1,1,1)}$, $A_3^2=I_2$, $A_3^3=I_3$ and $A_1^2$ is empty. Moreover, there are four sub-cases depending on whether the interactions between $I_1$ and $J_1$ and $I_3$ and $J_3$ are of the types $\mathcal{P}_1$ or $\mathcal{P}_4$, and $\mathcal{P}_1$ or $\mathcal{P}_5$, respectively. These sub-cases are depicted in Figure \ref{fig:case-1-wm-oracle}. Subsequently, we focus on sub-case (b) of Figure \ref{fig:case-1-wm-oracle}, whose treatment addresses all the necessary technical nuances required to deal with the other three cases.}

 \begin{figure}[htbp]
        \centering
        \includegraphics[width=0.9\linewidth]{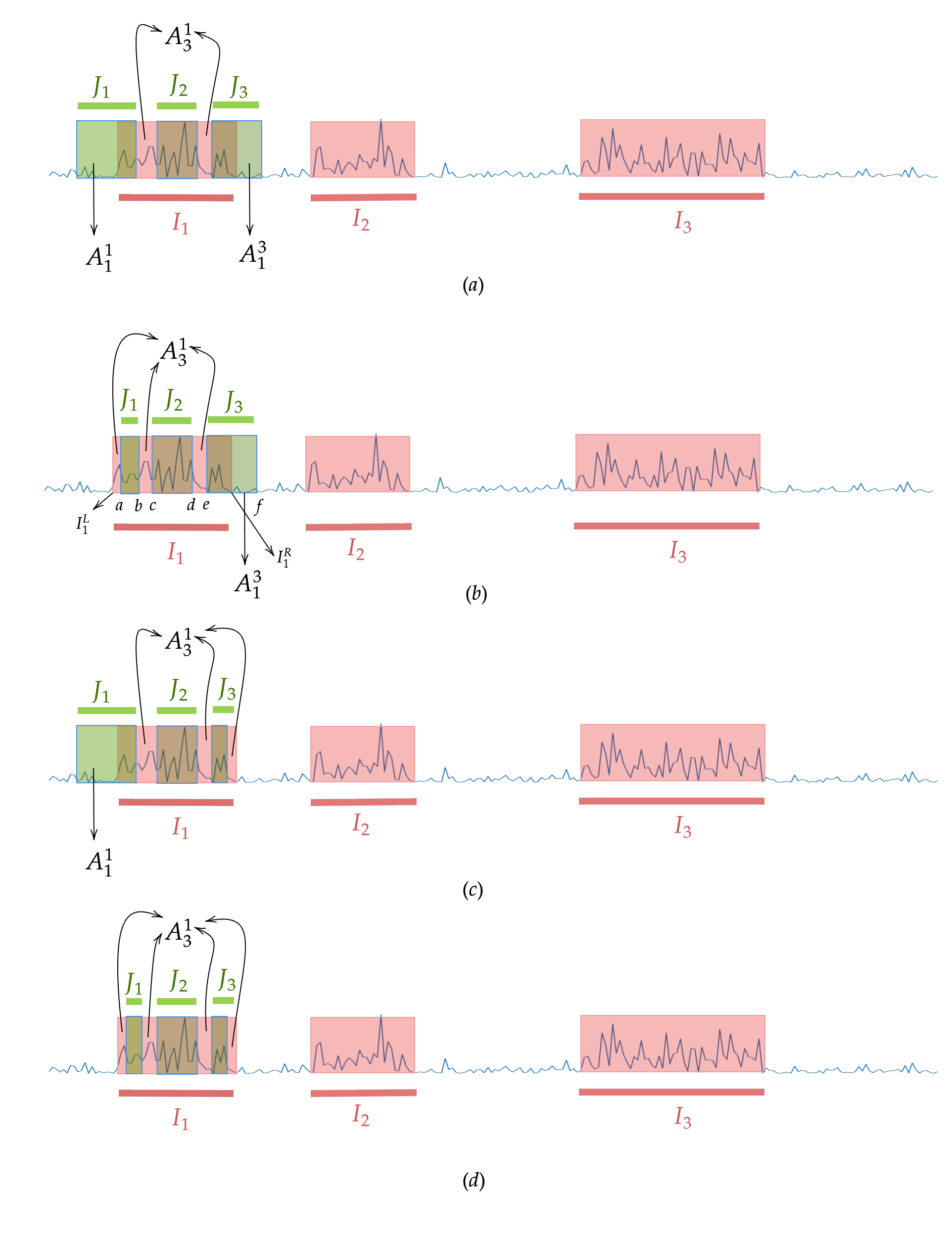}
        \caption{\upd{Representative cases for $P^{\varnothing}_{(1,1,1,1,1,1)}$ in the proof of Theorem \ref{thm:multiple-oracle}.}}
        \label{fig:case-1-wm-oracle}
\end{figure}

\upd{To that end, we begin by considering the following sequence of inequalities arising out of \eqref{eq:oracle-2nd-term-daddy}.
\begin{align}
    & \ \IP\bigg( \max_{J: \text{Case (b) of }P^{\varnothing}_{(1,1,1,1,1,1)}} \frac{\sum_{k=1}^K (S_3^{k,\varepsilon} - S_1^{k,\varepsilon}) }{\sum_{k=1}^K (x_1^k + x_3^k)} \geq \rho \tilde{d} \bigg) \nonumber\\
    \leq & \ \IP \bigg( \max_{(a,b,c,d,e,f): I_1^L < a < b < c < d< e < I_1^R <  f} \frac{S_{[I_1^L, a]}^\varepsilon+ S_{[b,c]}^\varepsilon+ S_{[d,e]}^\varepsilon + S_{I_2}^\varepsilon + S_{I_3}^\varepsilon - S_{[I_1^R, f]}^\varepsilon}{(a - I_1^L) + (c-b) + (e-d) + (f- I_1^R) +|I_2| + |I_3|} \geq \rho \tilde{d} \bigg) \nonumber\\
    \leq & \ \IP\big(\frac{S_{I_2}^\varepsilon}{|I_2|} \geq \frac{\rho \tilde{d}}{6} \big) + \IP\big(\frac{S_{I_3}^\varepsilon}{|I_3|} \geq \frac{\rho \tilde{d}}{6} \big) + \IP\bigg( \max_{a: I_1^L < a < I_1^R} \frac{S_{[I_1^L, a]}}{(a-I_1^L) + |I_2| + |I_3|} \geq  \frac{\rho \tilde{d}}{6}\bigg) \nonumber\\ & \hspace*{2cm}+ 2\IP \bigg( \max_{(b,c): I_1^L < b <c< I_1^R} \frac{S_{[b,c]}^\varepsilon}{(c-b) + |I_2| + |I_3|} \geq \frac{\rho \tilde{d}}{6} \bigg) \nonumber\\
    & \hspace*{2cm}+ \IP \bigg( \min_{f: I_1^R <f< I_2^L} \frac{S_{[I_1^R, f]}^\varepsilon}{(f- I_1^R) +|I_2| + |I_3|} \leq  -\frac{\rho \tilde{d}}{6} \bigg) \nonumber  \\
     := & \Gamma^1_1 + \Gamma^1_2 + \Gamma^1_3 +\Gamma^1_4 + \Gamma^1_5. \label{eq:oracle-case-1}
\end{align}
For $\Gamma^1_1$ and $\Gamma^1_2$, note that 
\begin{align}
    \IP\big(\frac{S_{I_k}^\varepsilon}{|I_k|} \geq \frac{\rho \tilde{d}}{6} \big) & \leq \inf_{\theta \geq 0}\exp(-6^{-1}\theta \rho \tilde{d} |I_k|) \IE[\exp(\theta S_{I_k}^\varepsilon)] \leq  \bigg(\inf_{\theta \geq 0} \exp(-6^{-1}\theta \rho \tilde{d}) \sup_{\IP}\IE_{1,\IP}[\exp(\theta \varepsilon)] \bigg)^{|I_k|}, \label{eq:oracle-case-1-gamma-1}
\end{align}
after which, in light of $|I_k| \geq I^\star \to \infty$, we can employ an argument similar to \eqref{eq:big-ub}-\eqref{eq:19} to conclude that \[ \Gamma^1_1 \wedge \Gamma^1_2 \to 0 \ \text{as } n \to \infty. \]
We briefly note that the final inequality in \eqref{eq:oracle-case-1-gamma-1} is derived via an argument verbatim to the successive projection-based analysis in \eqref{eq:inner-condn-exp-1}. In particular, let $\mathcal{F}_{j}=\sigma(\{(\omega_{s-1}, \zeta_s)\}: I_k^L+1 \leq s \leq j\} )$.
\begin{align}
    \IE[\exp(\theta S_{I_k}^\varepsilon)] &=  \IE[ \IE[\exp(\theta S_{I_k}^\varepsilon) \mid \mathcal{F}_{I_k^R-1}]] \nonumber\\
    &= \IE[ \exp(\theta S_{[I_k^L, I_k^R-1]}^\varepsilon) \IE[\exp(\theta \varepsilon_{I_k^R}) \mid \mathcal{F}_{I_k^R-1}] ] \nonumber.
\end{align}
On the other hand, for $\Gamma^1_3$, we proceed as follows. Following the arguments in \eqref{eq:big-ub}-\eqref{eq:19}, there exists $\theta_0 >0$, possibly depending on $\rho$, $\tilde{d}$, such that
\[ \kappa(\theta_0):= \inf_{\theta \ge 0}\exp(-6^{-1}\theta \rho \tilde{d}) \sup_{\IP} \IE_{1,\IP}[\exp(\theta \varepsilon)] \in (0,1). \]
Note that $\kappa(\theta_0)$ is independent of $n$, and consequently, independent of $I^\star$. Therefore,  
 \begin{align}
       \IP\bigg( \max_{a: I_1^L < a < I_1^R} \frac{S_{[I_1^L, a]}}{(a-I_1^L) + |I_2| + |I_3|} \geq  \frac{\rho \tilde{d}}{6}\bigg) & \leq \sum_{a=0}^{|I_1|} \IP\big(S_{[I_1^L, I_1^L +a]} \geq (a + |I_2| + |I_3|) \frac{\rho \tilde{d}}{6}\big) \nonumber\\
      & \leq \exp\big(-6^{-1}\theta_0 \rho \tilde{d} (|I_2| + |I_3|) \big)  \sum_{a=0}^{|I_1|} \bigg(\kappa(\theta_0) \bigg)^a \nonumber\\
      & \leq  \bigg(\exp\big(-6^{-1}\theta_0 \rho \tilde{d} \big)\bigg)^{|I_2| + |I_3|} \frac{1}{1- \kappa(\theta_0)} \nonumber\\
       & \to 0 , \ \text{as $n\to \infty$}, \label{eq:oracle-case-1-gamma-3}
 \end{align}
 where the final limiting assertion follows from $|I_2| + |I_3| \geq |I^\star|\to 0$ as $n\to \infty$. Observe that $\Gamma^1_5$ can also be dealt with similarly to $\Gamma^1_3$. Finally, for $\Gamma^1_4$, the same $\theta_0$ can be employed to deduce:
 \begin{align}
     & \ \ \ \ \ \IP \bigg( \max_{(b,c): I_1^L < b <c< I_1^R} \frac{S_{[b,c]}^\varepsilon}{(c-b) + |I_2| + |I_3|} \geq \frac{\rho \tilde{d}}{6} \bigg) \nonumber\\
     & \leq \sum_{j=1}^{|I_1|} |I_1|\IP \bigg(S_{[b, b+j]} \geq  \frac{\rho \tilde{d}}{6} (j + |I_2| + |I_3|) \bigg) \nonumber \\
     & \leq |I_1| \big(\exp(-6^{-1}\theta_0 \rho \tilde{d})\big)^{|I_2|+ |I_3|} \sum_{j=1}^{|I_1|}\kappa(\theta_0)^{j} \nonumber\\
     &\leq \frac{1}{1- \kappa(\theta_0)} \exp\big(\log n - 6^{-1}\theta_0 \rho \tilde{d} (|I_2| + |I_3|) \big) \to 0, \text{ as $n\to \infty$}, \label{eq:oracle-b-c}
 \end{align}
 where the last assertion is in light of $|I_2| + |I_3| \geq |I^\star| \gg \log n$.
 This completes the proof for the case $\mathcal P^{\varnothing}_{(1,1,1,1,1,1)}$.}

\clearpage
\textbf{\upd{Case 2: $\mathcal P^{\varnothing}_{(1,1,1,1,1,2)}$}}

\upd{In this case, $x_1^3 = d(I_1, I_2) = I_2^L - I_1^R$, $A_3^3 = I_3$ and $A_1^2$ is empty. The resulting two cases are shown in Figure \ref{fig:case-2-wm-oracle}; again, we show the analysis for Case (a) of Figure \ref{fig:case-2-wm-oracle}, and Case (b) will follow along verbatim. We proceed along similar lines to the previous case.}

\begin{figure}[htbp]
        \centering
        \includegraphics[width=0.9\linewidth]{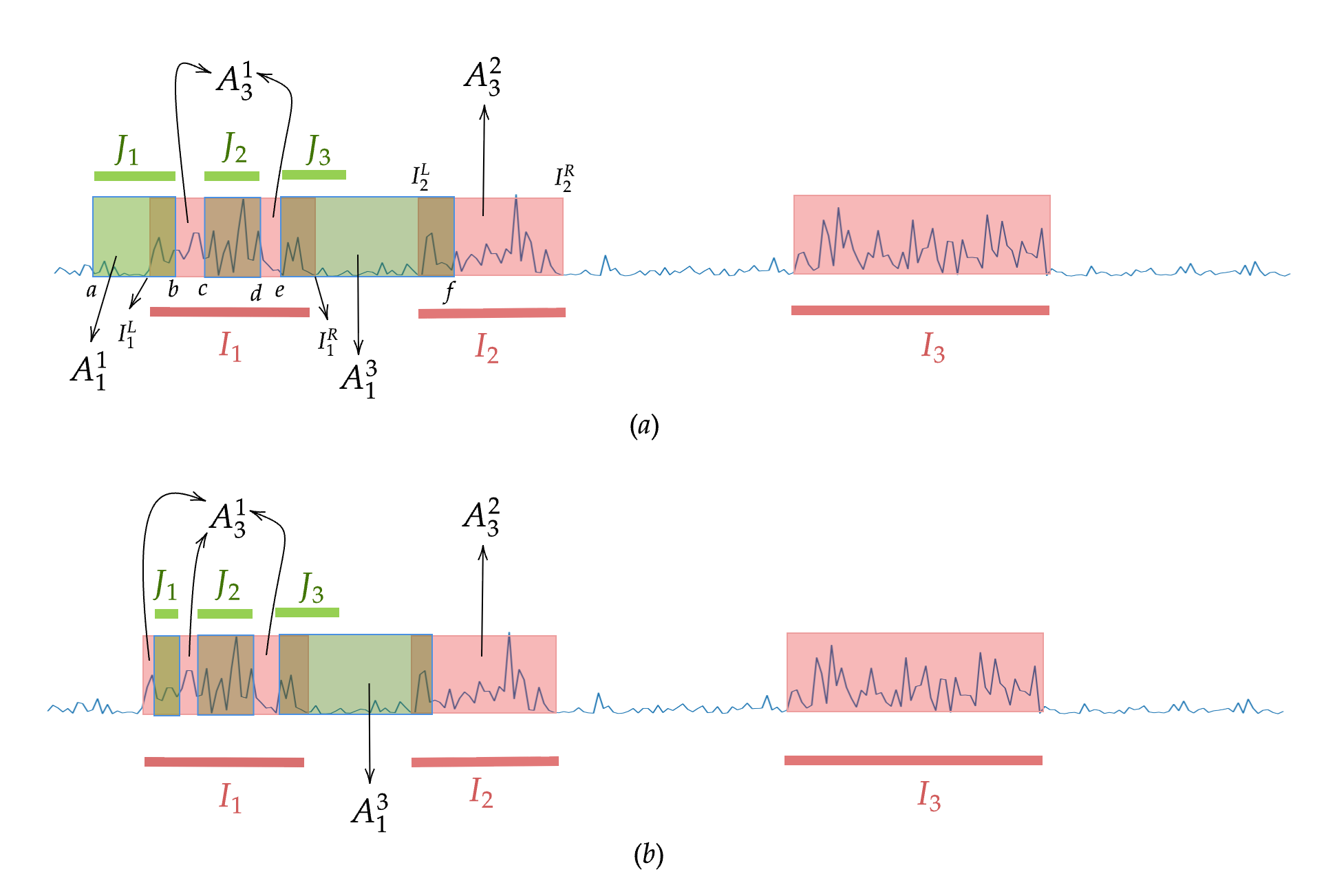}
        \caption{\upd{Representative cases for $P^{\varnothing}_{(1,1,1,1,1,2)}$ in the proof of Theorem \ref{thm:multiple-oracle}.}}
        \label{fig:case-2-wm-oracle}
\end{figure}
\allowdisplaybreaks
\upd{\begin{align}
    & \ \IP\bigg( \max_{J: \text{Case (a) of }P^{\varnothing}_{(1,1,1,1,1,2)}} \frac{\sum_{k=1}^K (S_3^{k,\varepsilon} - S_1^{k,\varepsilon}) }{\sum_{k=1}^K (x_1^k + x_3^k)} \geq \rho \tilde{d} \bigg) \nonumber\\
    \leq & \ \IP \bigg( \max_{(a,b,c,d,e,f): a < I_1^L < b < c < d< e < I_1^R < I_2^L <  f} \frac{ S_{[b,c]}^\varepsilon+ S_{[d,e]}^\varepsilon + S_{[f, I_2^R]}^\varepsilon + S_{I_3}^\varepsilon - S_{[I_1^R, I_2^L]}^\varepsilon- S_{[a, I_1^L]}^\varepsilon+ }{(I_1^L-a) + (c-b) + (e-d) + (I_2^L- I_1^R) +(I_2^R-f) + |I_3|} \geq \rho \tilde{d} \bigg) \nonumber\\
    \leq & \ \IP\big(-\frac{S_{[I_1^R, I_2^L]}^\varepsilon}{(I_2^L- I_1^R)} \geq \frac{\rho \tilde{d}}{6} \big) + \IP\big(\frac{S_{I_3}^\varepsilon}{|I_3|} \geq \frac{\rho \tilde{d}}{6} \big) + \IP\bigg( \max_{f: I_2^L < f < I_2^R} \frac{S_{[f, I_2^R]}}{(I_2^L- I_1^R) +(I_2^R-f) + |I_3|} \geq  \frac{\rho \tilde{d}}{6}\bigg) \nonumber\\ & \hspace*{2cm}+ 2\IP \bigg( \max_{(b,c): I_1^L < b <c< I_1^R} \frac{S_{[b,c]}^\varepsilon}{(c-b) + |I_2| + |I_3|} \geq \frac{\rho \tilde{d}}{6} \bigg) \nonumber\\
    & \hspace*{2cm}+ \IP \bigg( \min_{a: a < I_2^L} \frac{S_{[a, I_1^L]}^\varepsilon}{(I_1^L-a) +|I_2^L - I_1^R| + |I_3|} \leq  -\frac{\rho \tilde{d}}{6} \bigg) \nonumber  \\
    := & \Gamma^2_1 + \Gamma^2_2 + \Gamma^2_3 +\Gamma^2_4 + \Gamma^2_5 \nonumber. 
\end{align}
Clearly, $\Gamma^2_1 , \Gamma^2_2, \Gamma^2_3, \Gamma^2_5$ can be dealt with as in \eqref{eq:oracle-case-1-gamma-1} and \eqref{eq:oracle-case-1-gamma-3} respectively. In particular, for the term $\Gamma^2_1$ we use the fact that $I_2^L - I_1^R \geq I^\star \to \infty$ as $n \to \infty$. Finally, noting that $\Gamma^2_4= \Gamma^1_4$ completes the proof of this case.}

\textbf{\upd{Case 3: $\mathcal P^{\varnothing}_{(1,1,1,1,2,3)}$}}

\upd{In this case, there are total $16$ sub-cases: there are four interactions, between $(I_1, J_1)$, $(I_1, J_2)$, $(I_2, J_3)$ and $(I_3, J_3)$, and the interactions can be one among the types $(\mathcal{P}_1, \mathcal{P}_4)$, $(\mathcal{P}_1, \mathcal{P}_5)$, $(\mathcal{P}_2, \mathcal{P}_5)$ and $(\mathcal{P}_2, \mathcal{P}_4)$ respectively. Some illustrative examples of these cases can be found in Figure \ref{fig:case-3-wm-oracle}. Among these cases, we work with the case when the four interactions are of the type $(\mathcal{P}_1, \mathcal{P}_5, \mathcal{P}_2, \mathcal{P}_4)$ respectively (Case (c) in Figure \ref{fig:case-3-wm-oracle}).} 

\begin{figure}[h]
        \centering
        \includegraphics[width=0.9\linewidth]{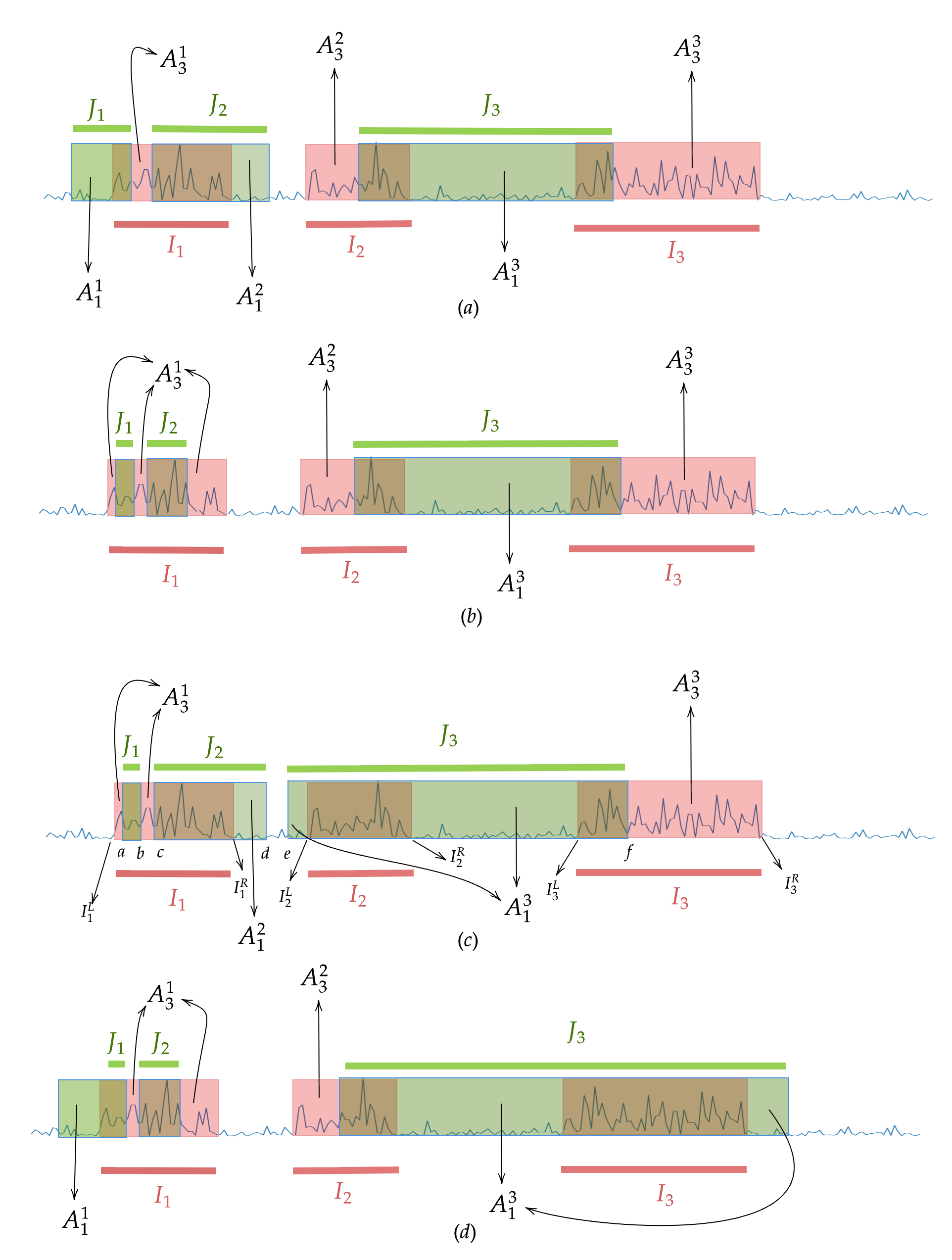}
        \caption{\upd{Some representative cases for $P^{\varnothing}_{(1,1,1,1,2,3)}$, used in the proof of Theorem \ref{thm:multiple-oracle}. For example, the interactions between $(I_1, J_1)$, $(I_1, J_2)$, $(I_2, J_3)$ and $(I_3, J_3)$, are of the types (a) $(\mathcal{P}_4, \mathcal{P}_5, \mathcal{P}_5, \mathcal{P}_4)$, (b) $(\mathcal{P}_1, \mathcal{P}_1, \mathcal{P}_5, \mathcal{P}_4)$, (c) $(\mathcal{P}_1, \mathcal{P}_5, \mathcal{P}_2, \mathcal{P}_4)$, and (d) $(\mathcal{P}_4, \mathcal{P}_1, \mathcal{P}_5, \mathcal{P}_2)$ respectively.}}
        \label{fig:case-3-wm-oracle}
\end{figure}

\upd{Consequently, we write:
\allowdisplaybreaks
\begin{align}
    & \ \IP\bigg( \max_{J: \text{Case (c) of Figure \ref{fig:case-3-wm-oracle}}} \frac{\sum_{k=1}^K (S_3^{k,\varepsilon} - S_1^{k,\varepsilon}) }{\sum_{k=1}^K (x_1^k + x_3^k)} \geq \rho \tilde{d} \bigg) \nonumber\\
    \leq & \ \IP \bigg( \max_{(a,b,c,d,e,f):  I_1^L< a < b < c< I_1^R < d< e < I_2^L < I_2^R < I_3^L< f < I_3^R} \nonumber\\ &\hspace*{3cm} \frac{ S_{[I_1^L,a]}^\varepsilon+ S_{[b,c]}^\varepsilon - S_{[I_1^R, d]}^\varepsilon - S_{[e, I_2^L]}^\varepsilon - S_{[I_2^R, I_3^L]}^\varepsilon +  S_{[f, I_3^R]}^\varepsilon}{(a- I_1^L) + (c-b) + (d-I_1^R) + (I_2^L -e)+ (I_3^L- I_2^R) +(I_3^R-f) } \geq \rho \tilde{d} \bigg) \nonumber \nonumber\\
    \leq & \ \IP\bigg( - \frac{S_{[I_2^R, I_3^L]}^\varepsilon}{I_3^L- I_2^R} \ge \frac{\rho \tilde{d}}{6}  \bigg) +  \IP \bigg( \max_{a: I_1^L <a< I_1^R}\frac{ S_{[I_1^L,a]}^\varepsilon}{(a- I_1^L) + (I_3^L- I_2^R)}  \ge \frac{\rho \tilde{d}}{6}  \bigg) \nonumber\\ & \hspace{1cm} +  \IP \bigg( \max_{f: I_3^L <a< I_3^R}\frac{ S_{[f, I_3^R]}^\varepsilon}{(I_3^R- f) + (I_3^L- I_2^R)}  \ge \frac{\rho \tilde{d}}{6}  \bigg) +\IP \bigg( \min_{d: I_1^R <a< I_2^L}\frac{ S_{[I_1^R, d]}^\varepsilon}{(d- I_1^R) + (I_3^L- I_2^R)}  \le -\frac{\rho \tilde{d}}{6}  \bigg) \nonumber \\ 
    & \hspace{1cm} + \IP \bigg( \min_{e: I_1^R <a< I_2^L}\frac{ S_{[e, I_2^L]}^\varepsilon}{(I_2^L- e) + (I_3^L- I_2^R)}  \le -\frac{\rho \tilde{d}}{6}  \bigg) + \IP \bigg( \min_{b,c: I_1^L <b< c< I_1^R}\frac{ S_{[b,c]}^\varepsilon}{(c-b) + (I_3^L- I_2^R)}  \le -\frac{\rho \tilde{d}}{6}  \bigg) \nonumber,
\end{align}
and all the terms of this decomposition can be dealt similar to the terms $\Gamma_1^1, \ldots, \Gamma_5^1$. In particular, for the term 
\[ \IP \bigg( \min_{b,c: I_1^L <b< c< I_1^R}\frac{ S_{[b,c]}^\varepsilon}{(c-b) + (I_3^L- I_2^R)}  \le -\frac{\rho \tilde{d}}{6}  \bigg), \]
a treatment similar to the $\Gamma_4^1$ term (analyzed in  \eqref{eq:oracle-b-c}) can be applied in light of  $(I_3^L- I_2^R) \ge I^\star \gg \log n$. Therefore, this term can also be controlled. }

\upd{Employing the same techniques, all the cases can be controlled, leading to 
\[ \IP\Big(\min_{J: x_2^{k \ell} >0 \text{ for some $k \ne \ell \in [K]$} } V_J - V_I \leq 0 \Big) \to 0, \text{ as $n\to \infty$}, \]
for the second term in \eqref{eq:good-cases-bad-cases}. As discussed before and similar to the treatment following \eqref{eq:choice-of-delta} , given $\delta>0$, the choice of $M = M_{\delta} / (\rho \tilde{d})$ for some $M_{\delta}$ ensures the first term remains $\delta>0$. This completes the proof.}

\subsection{Additional propositions}\label{se:proof-add}

Firstly, we provide a formal proof that the pivot statistics corresponding to unwatermarked tokens are i.i.d., a fundamental fact behind the construction and validity of our algorithm.

\begin{proof}[Proof of Lemma \ref{lemma:pivot-iid}]
    Let $t \in S$. Since $\omega_t$ and $\zeta_t$ are independent conditional on $\omega_{1:(t-1)}$, hence $\mathcal{L}(Y_t)| \omega_{1:t} \overset{d}{=} \mathcal{L}(Y)$. Hence, $\{Y_t\}_{t \in S}$ are identically distributed since given $\omega_{1:t}$, the distribution of $Y_t$ is solely a function of the key $\zeta_t$ that are i.i.d.. Moreover, for $s<t\in S$, even if there is a watermarked region $I_k \subset (s,t)$, $\zeta_s$ and $\omega_l$ are independent for all $l \in (s, t]$. In view of the fact that conditional on $\omega_{1:s}$, $Y_s$ and $Y_t$ are completely determined by $\zeta_s$ and $(w_{s+1:t}, \zeta_{s+1:t})$ respectively, we deduce that $Y_s$ and $Y_t$ are independent conditional on $\omega_{1:s}$. Hence, for two Borel sets $A$ and $B$,
    \allowdisplaybreaks \begin{align}
        \IP(Y_s\in A, Y_t \in B) =& \IE[ \IP(Y_s \in A | \omega_{1:s}) \IE[\IP(Y_t \in B | \omega_{1:t}) | \omega_{1:s} ]] \nonumber\\ 
        =& \IP(Y_s\in A)\IP( Y_t \in B) \ \text{(from Definition \ref{def:pivotal})} \nonumber.
    \end{align}
    This completes the proof.
\end{proof}
Next, we collect the additional results that we have used in our theoretical arguments. The proofs are provided subsequently.

\upd{\begin{proposition}\label{lem:eqdv-condn}
    Under the assumptions of Theorems \ref{thm:single-watermark} and \ref{thm:multiple-watermark}, there exists $\tau>0$ such that 
    \[ \kappa:= \inf_{\theta \geq 0} \theta (\mu_0+ \tau d) + \log \sup_{\IP} \IE_{1,\IP}[\exp(-\theta X)] <0 .\]
\end{proposition}}

\upd{\begin{remark}[Discussion on Proposition \ref{lem:eqdv-condn}] \label{rem:rqdv-condn}
We also briefly discuss the necessity of the rather technical  Proposition \ref{lem:eqdv-condn} in Theorem \ref{thm:multiple-watermark}. This result can be construed as a Donsker-Varadhan strengthened version of Assumption \ref{ass:alt-mean}. For an appropriate choice of the score function $h$ and some NTP distribution $P^\star$ depending on $\mathcal{P}$, the Donsker Varadhan representation \citep{donsker} entails
\[ \inf_{\theta \geq 0} \theta \mu_0 + \log \sup_{\IP\in \mathcal{P}} \IE_{1,\IP}[\exp(-\theta X)] = - D_{\operatorname{KL}}(\mathcal{L}_0(X), \mathcal{L}_{1, P^\star}(X) ), \]
where $D_{\operatorname{KL}}$ denotes the Kullback-Leibler divergence, $\mathcal{L}_0$ denotes the law of unwatermarked pivot statistic, and $\mathcal{L}_{1, P^\star}$ denotes the law of watermarked pivot statistic when the NTP is $P^\star$. In light of this, $\kappa$ lifts the minimum separation between the unwatermarked and watermarked distributions into a gap between the cumulant functions, and can therefore be understood to be mild. Proposition \ref{lem:eqdv-condn} establishes a weak uniform control over the behavior of pivot statistics under watermarked segments. This allows us to rigorously bypass the possibly arbitrary and strong dependence across the pivot statistics corresponding to watermarked tokens while deriving Theorem \ref{thm:multiple-watermark}.
\end{remark}}

\begin{proposition}\label{lem:gumbel-example}
    Let $h(x)=-\log(1-x)$, and suppose $\mathcal{P}_{\Delta}:=\{\max_{w \in \mathcal{W}} P_w \leq 1-\Delta\}$ for some fixed $\Delta>0$. Then it follows that 
    \allowdisplaybreaks \begin{align}
 \inf_{\IP \in \mathcal{P}_{\Delta}}\IE_{1,\IP}[h(Y)] \geq \sum_{n=1}^{\infty}\Big(\frac{1}{n} - \lfloor \frac{1}{1-\Delta} \rfloor \frac{(1-\Delta)^2}{1+n(1-\Delta)} - \frac{1-(1-\Delta)\lfloor \frac{1}{1-\Delta}\rfloor}{1+n(1-(1-\Delta)\lfloor \frac{1}{1-\Delta}\rfloor)}\Big). \label{eq:gumbel-lb}        
    \end{align}
\end{proposition}

\begin{proposition}\label{prop:corollary-to-Theorem-3}
    Consider $\tilde{d}$ and $\Psi(\cdot)$ from Theorem \ref{thm:single-watermark-app}. If there exists a constant $c>0$ such that ${d}\geq c$, then 
    \allowdisplaybreaks \begin{align}
        (\sup_{\theta\geq 0} \{\theta \rho \Tilde{d}- \Psi(\theta)\})^{-1} = O((\rho \tilde{d})^{-1}) \label{eq: O(1) bound}.
    \end{align}
    Recall $\varepsilon$ from Theorem \ref{thm:single-watermark}. Suppose we additionally have that $$\max\{ \IE_0[\exp(r|\varepsilon|)] ,\sup_{\IP\in \mathcal{P}}\IE_{1,\IP}[\exp(r|\varepsilon|)] \} \leq \exp(r^2/2) \ \text{for all $r\in [0, \eta]$,}$$  $\eta$ being the same as in Theorem \ref{thm:single-watermark}. Then, choosing $\rho>0$ such that $\rho\tilde{d}< \frac{5}{2}\eta$, it holds that 
    \allowdisplaybreaks \begin{align}
        (\sup_{\theta\geq 0} \{\theta \rho \Tilde{d}- \Psi(\theta)\})^{-1} = O((\rho \tilde{d})^{-2}) \label{eq: dsq bound}.
    \end{align}
\end{proposition}

\begin{proposition} \label{prop:1}
    Let $\IE_0[|X-\mu_0|^{p}]<\infty$ for some $p\geq 2$. Let
    $\mathcal{Q}=\mathcal{Q}_n$ and $b=b_n$ be selected as in
    Theorem~\ref{thm:multiple-watermark}. Then it follows that
    $\mathcal{Q}/b \to \mu_0$ as $n\to \infty$.
\end{proposition}
\color{black}

\begin{proposition}\label{prop:2}
    Let $X_i$ be i.i.d. with mean $\mu_0$, and let $B_k$ and $S_k$ be defined as in Steps $2$ and $3$ of \texttt{WISER}\ in Figure \ref{fig:algo}. Then it follows that 
    \[\IP_0\left( \sum_{k=1}^{\lceil n/b \rceil} I\{S_k > \mathcal{Q}\} \geq C_0 \sqrt{\log n} \right) \to 0, \ \text{as $n\to \infty$}, \]
    where $\mathcal{Q}$ is defined as in Theorem \ref{thm:multiple-watermark}.
\end{proposition}

\upd{\begin{proof}[Proof of Proposition \ref{lem:eqdv-condn}]
For a generic pivot statistics $X$ with the corresponding next token distribution $P$, let $\mu_P= \E_{1,P}[X]$ and $\varepsilon =  X - \mu_P$. Note that, 
\begin{align}
\IE_{1,\IP}[\exp(-\theta X)] &= \exp(-\theta \mu_P) \IE_{1,\IP}[\exp(-\theta \varepsilon)] \nonumber\\
&\overset{(a)}{\leq} \big(1+ \E_{1,P}[\theta^2 \varepsilon^2 \exp(\theta |\varepsilon|) /2 ] \big) \exp(-\theta \mu_P) \nonumber\\ 
&\overset{(b)}{\leq}\exp\big(-\theta \mu_P + \theta^2 \E_{1,P}[\varepsilon^2 \exp(\theta |\varepsilon|) /2 ]\big) \label{eq:dv-1}
\end{align}
where in (a) we use the elementary bound $\exp(x) \leq 1 + x + x^2 \exp(|x|)/2$ for all $x \in \R$ along with $\E_{1,P}[\varepsilon]=0$ for all $\IP \in \mathcal{P}$; (b) uses $1+x \leq \exp(x)$.  Recall $\eta>0$ from the statement of Theorem \ref{thm:single-watermark}. For $0 \leq x < \eta/2$, routine algebraic manipulations yield that for any $w>0$,  $$w^2 \exp(xw) \leq \exp(w \eta ) \sup_{w\geq 0} w^2 \exp(-\eta w/2) \leq \frac{16}{e^2 \eta^2} \exp(w \eta ) . $$
Consequently, \eqref{eq:dv-1} along with Assumption \ref{ass:alt-mean} entails for all $\theta \in [0, \eta/2)$ that,
\begin{align}
    \sup_{\IP\in \mathcal{P}}\IE_{1,\IP}[\exp(-\theta X)] \leq \exp\bigg(-\theta (\mu_0 +d) + 
\frac{8\theta^2}{e^2 \eta^2} \sup_{\IP\in \mathcal{P}} \E_{1,P}[\exp(\eta|\varepsilon|)]\bigg). \label{eq:dv-2}
\end{align} 
Clearly, for all $\theta \in [0, \eta/2)$, \eqref{eq:dv-2} implies that
\begin{align}
    \theta (\mu_0+ \tau d) + \log \sup_{\IP} \IE_{1,\IP}[\exp(-\theta X)] \leq - \theta (1-\tau)d + \frac{8\theta^2}{e^2 \eta^2} \sup_{\IP\in \mathcal{P}} \E_{1,P}[\exp(\eta|\varepsilon|)], \label{eq:dv-3}
\end{align}
which, in turn, produces that with $M = \min\{ \eta/2, (1-\tau)d e^2 \eta^2 \sup_{\IP\in \mathcal{P}} \E_{1,P}[\exp(\eta|\varepsilon|)] /8 \},$ it holds for all $\theta \in (0, M)$ that 
\begin{align}
     - \theta (1-\tau)d + \frac{8\theta^2}{e^2 \eta^2} \sup_{\IP\in \mathcal{P}} \E_{1,P}[\exp(\eta|\varepsilon|)] < 0. \nonumber
\end{align}
This completes the proof in light of \eqref{eq:dv-3} and 
\[ \inf_{\theta \geq 0} \theta (\mu_0+ \tau d) + \log \sup_{\IP} \IE_{1,\IP}[\exp(-\theta X)] \leq \inf_{0 < \theta <M} - \theta (1-\tau)d + \frac{8\theta^2}{e^2 \eta^2} \sup_{\IP\in \mathcal{P}} \E_{1,P}[\exp(\eta|\varepsilon|)].\]
\end{proof}}

\begin{proof}[Proof of Proposition \ref{lem:gumbel-example}]
     From Lemma 3.1 of \cite{li2025statistical}, it follows
\allowdisplaybreaks \begin{align}
    \IE_{1,\IP}[h(X)] &=\sum_{w=1}^{|\mathcal{W}|} \int_0^1 x^{1/P_w-1}\bigl(-\log(1-x)\bigr)\,\text{d}x \nonumber\\
    &= \sum_{w=1}^{|\mathcal{W}|} \sum_{n=1}^{\infty}\int_0^1 \frac{x^{1/P_w-1+n}}{n} \,\text{d}x  \nonumber\\
    &= \sum_{w=1}^{|\mathcal{W}|} \sum_{n=0}^{\infty} \frac{1}{n(n+1/P_w)}\nonumber\\
    &= \sum_{n=1}^{\infty}(\frac{1}{n} - \sum_{w=1}^{|\mathcal{W}|} \frac{P_w}{n+1/P_w})\nonumber\\
    &\geq \sum_{n=1}^{\infty}\left(\frac{1}{n} - \lfloor \frac{1}{1-\Delta} \rfloor \frac{(1-\Delta)^2}{1+n(1-\Delta)} - \frac{1-(1-\Delta)\lfloor \frac{1}{1-\Delta}\rfloor}{1+n(1-(1-\Delta)\lfloor \frac{1}{1-\Delta}\rfloor)}\right),
\end{align}
where the final inequality follows from noting the convexity of $g: x \mapsto \sum_{i=1}^d\frac{x_i}{n+1/x_i}$, $\sum_{i=1}^d x_i=1$, and noting that the optimum value of $g$ on the set $\mathcal{P}_{\Delta}$ occurs at the extrema defined by 
\[
P^\star_\Delta = 
\Big(
\underbrace{\,1-\Delta, \ldots, 1-\Delta\,}_{\lfloor \tfrac{1}{1-\Delta}\rfloor \text{ times}},
\,1-(1-\Delta)\cdot \lfloor \tfrac{1}{1-\Delta}\rfloor,\,
0,\ldots
\Big).
\]
\end{proof}

\begin{proof}[Proof of Proposition \ref{prop:corollary-to-Theorem-3}]
    Denote $\Lambda(x):=\sup_{\theta\geq 0} \{\theta \rho x- \Psi(\theta)\}$. Note that, an argument same as (\ref{eq:fixed-P}) shows that $\Psi^{'}_+(0)=0$, where $\Psi^{'}_+(\cdot)$ denote the right derivative. Therefore, in light of ${d}\geq c$ for some constant $c>0$, there exists $\theta_0>0$ such that 
    $\frac{|\Psi(\theta)|}{\theta} \leq \frac{\rho \tilde{d}}{2}$ for all $\theta\in(0,\theta_0)$. Therefore, 
    \[\Lambda(\tilde{d}) \geq 2^{-1}\theta_0\rho \tilde{d} - 4^{-1}\theta_0 \rho c \geq 4^{-1} \theta_0 \rho \tilde{d} ,\]
    which immediately implies (\ref{eq: O(1) bound}). Moving on, we work with the additional assumption that $\sup_{\IP\in \mathcal{P}}\IE_{1,\IP}[\exp(r|\varepsilon|)] \leq \exp(r^2/2)$. This immediately implies that for all $\theta \in [0, 
    \frac{\eta}{2}]$, \[\max\{ \log \sup_{\IP} \IE_{1,\IP}[\exp(2\theta \varepsilon)], \log \sup_{\IP} \IE_{1,\IP}[\exp(-2\theta \varepsilon)] \} \leq 2 \theta^2. \] Therefore, for all $\theta \in [0,  \frac{\eta}{2}]$ it must hold that
    \[ \Psi(\theta) \leq \frac{5}{2}\theta^2. \]
    Consequently, in light of $\rho\tilde{d}< \frac{5}{2}\eta$, one obtains,
    \begin{align*}
        \Lambda(x) \geq \sup_{\theta\in[0, \frac{\eta}{2}]} \{\theta\rho x - \frac{5}{2}\theta^2 \}= \frac{\rho^2x^2}{10},  
    \end{align*}
    which establishes (\ref{eq: dsq bound}).
\end{proof}
\begin{proof}[Proof of Proposition \ref{prop:1}]
   Our proof has two key steps: firstly, we will prove that if there is no watermarking in the entire sequence, then \allowdisplaybreaks \begin{align}
       \max_{1\leq k \leq \lceil n/b\rceil} \frac{S_k}{b_n} \overset{\IP}{\to} \mu_0 \label{eq:in-prob-conv}.
   \end{align}
   Subsequently, we follow an argument similar to the proof of equation (29) in \cite{li2025statistical}, with crucial tweaks to accommodate the maximum over the block means. Let us first work towards (\ref{eq:in-prob-conv}). We note that a similar result (for the $p$-th moments) appears in Proposition E.2 in~\cite{deb2020measuring} but without proof. For the sake of completion, we provide an independent proof of (\ref{eq:in-prob-conv}) without invoking the aforementioned result. Whenever convenient, we will denote $b_n$ by $b$, suppressing the subscript $n$ which indicates the sequential nature of the choice of block length $b$ corresponding to the total number of tokens, $n$. Fix $\varepsilon>0$.
   Note that 
   \allowdisplaybreaks \begin{align}\label{eq:unif-bound}
   \IP_0(\max_{1\leq k \leq \lceil n/b\rceil} b^{-1}(S_k - \mu_0)> \varepsilon) \leq \frac{n}{b} \IP(b^{-1}(S_1- \mu_0) > \varepsilon),    
   \end{align}
   where for the last inequality we use that $S_k$'s are i.i.d. under $H_0$, i.e. no watermarking. Moving on, we apply the Fuk-Nagaev inequality (Corollary 4,~\cite{fuk1971probability}),
   \allowdisplaybreaks \begin{align}\label{eq:f-n}
       \IP_0(b^{-1}(S_1- \mu_0) > \varepsilon) \leq c_{1}\frac{b}{(b\varepsilon)^{p}}\IE_0[|X-\mu|^{p}] + \exp(-c_{2}\frac{b \varepsilon^2}{\sigma^2}), \ \sigma^2:=\IE_0[X^2],
   \end{align}
   where $c_{1}, c_{2}>0$ are constants depending solely on $p$. Note that $b_n / n^{1/p} \to \infty$, and hence $\frac{n}{(b\varepsilon)^{p}} \to 0$ as $n\to \infty$. On the other hand, $nb^{-1}\exp(-c_{2}\frac{b \varepsilon^2}{\sigma^2})\to 0$ as $n\to \infty$. Therefore, from (\ref{eq:unif-bound}) and (\ref{eq:f-n}), one obtains (\ref{eq:in-prob-conv}). 
   
   Now suppose that $\limsup_{n\to \infty}\mathcal{Q}/b > \mu_0$. Then there exists $\gamma>0$ and a strictly increasing sequence $\{n_k\}\subseteq\N$ such that $\mathcal{Q}_{n_k}/b_{n_k} > \mu_0+\gamma$ for all sufficiently large $k\in \N$. Let $S_{\ell}^{(n)}= \sum_{i=(\ell-1)b_n +1}^{\ell b_n \wedge n} X_{i}$ for $n \in \N$. Since (\ref{eq:in-prob-conv}) implies that \[\max_{1\leq \ell \leq \lceil n_k/b_{n_k}\rceil} \frac{S_{\ell}^{(n_k)}}{b_{n_k}} \overset{{\IP}}{\to} \mu_0, \text{ as $k\to\infty$},\]
   therefore, there exists a strictly increasing sub-sequence $\{n_{k_r}\}\subseteq \{n_k\}$ such that 
   \[ \max_{1\leq l \leq \lceil n_{k_r}/b_{n_{k_r}}\rceil} \frac{S_{\ell}^{(n_{k_r})}}{b_{n_{k_r}}} \overset{\text{a.s.}}{\to} \mu_0, \text{ as $r\to\infty$, and } \mathcal{Q}_{n_{k_r}}/b_{n_{k_r}} > \mu_0+\gamma \text{ for all sufficiently large $r$. } \]
   Therefore, by the dominated convergence theorem,
   \allowdisplaybreaks \begin{align}
       \alpha= \lim_{r\to \infty}\IP\left(\max_{1\leq l \leq \lceil n_{k_r}/b_{n_{k_r}}\rceil} \frac{S_{\ell}^{(n_{k_r})}}{b_{n_{k_r}}} > \frac{\mathcal{Q}_{n_{k_r}}}{b_{n_{k_r}}}\right) &\leq \lim_{r\to \infty}\IP\left(\max_{1\leq l \leq \lceil n_{k_r}/b_{n_{k_r}}\rceil} \frac{S_{\ell}^{(n_{k_r})}}{b_{n_{k_r}}} > \mu_0+\gamma\right) \nonumber \\
       &= \IP(\mu_0 > \mu_0+\gamma)=0,
   \end{align}
   which is a contradiction. Hence, $\limsup_{n\to \infty}\mathcal{Q}/b \leq \mu_0$. Very similarly, one can show $\liminf_{n\to \infty}\mathcal{Q}/b \geq \mu_0$, which completes the proof.
\end{proof} 

\begin{proof}[Proof of Proposition \ref{prop:2}]
  Under null, the events \(\{S_k>\mathcal{Q}\}\), \(1\le k\le \lceil n/b\rceil\), are independent, since the blocks are disjoint and the \(X_i\)'s are i.i.d. By the definition of \(\mathcal{Q}\),
\[
\prod_{k=1}^{\lceil n/b\rceil}\IP_0(S_k\le \mathcal{Q})
=
\IP_0\left(\max_{1\le k\le \lceil n/b\rceil}S_k\le \mathcal{Q}\right)
\ge 1-\alpha .
\]
Therefore, using \(x\le -\log(1-x)\),
\[
\sum_{k=1}^{\lceil n/b\rceil}\IP_0(S_k>\mathcal{Q})
\le
-\sum_{k=1}^{\lceil n/b\rceil}\log \IP_0(S_k\le \mathcal{Q})
\le
-\log(1-\alpha).
\]
Hence, by Markov's inequality,
\[
\IP_0\left(
\sum_{k=1}^{\lceil n/b\rceil}\mathbf 1\{S_k>\mathcal{Q}\}
\ge C_0\sqrt{\log n}
\right)
\le
\frac{-\log(1-\alpha)}{C_0\sqrt{\log n}}
\to 0.
\]
This completes the proof.
\end{proof}

\upd{Here, we collect some preparatory results helpful in proving Theorem \ref{thm:multiple-oracle}.}

\upd{\begin{lemma}\label{lemma:x2-ordering}
Let $\{I_1,\ldots,I_K\}$ and $\{J_1,\ldots,J_K\}$ be disjoint intervals, each ordered from
left to right. If
\[
x_2^{ij}=|J_i\cap I_j|>0,
\]
then
\[
x_2^{rk}=0 \qquad \text{for all } r>i,\ k<j.
\]
Equivalently, if $i_1<i_2$ and $x_2^{i_1j_1}>0$, $x_2^{i_2j_2}>0$, then $j_1\le j_2$.
\end{lemma}
\begin{proof}[Proof of Lemma \ref{lemma:x2-ordering}]
If $x_{ij}>0$ with $x_2^{rk}>0$, then $J_i$ intersects $I_j$ while the later interval
$J_r$ intersects the earlier interval $I_k$. Since $i < r$ and $k < j$, this is impossible
for disjoint ordered intervals, concluding the proof.
\end{proof}}

\upd{\begin{lemma}\label{lemma:good-cases}
   Consider the setting of Theorem \ref{thm:multiple-oracle}. For a candidate set of patches $\{J_1, J_2, \ldots, J_K\} \in \mathcal{I}_K$, let $x_1^k$, $x_2^{ij}$, $x_3^k$ and $x_4$ be defined as in the proof of Theorem \ref{thm:multiple-oracle}. Recall $I^\star$ from Assumption \ref{ass:min-sep}. If $x_2^{k\ell}>0$ for some $k\neq \ell$, then
\[
\sum_{r=1}^K (x_1^r+x_3^r)\ge I^\star.
\]
\end{lemma}}

\begin{proof}[\upd{Proof of Lemma \ref{lemma:good-cases}}]
\upd{We consider two cases.}

\upd{\textit{Case 1:} \textit{There exists $r\in[K]$ such that $J_r$ intersects at least two true intervals.}}

\upd{Since the intervals $\{I_j\}_{j=1}^K$ are disjoint and ordered, $J_r$ must then intersect two consecutive true intervals, say $I_s$ and $I_{s+1}$. As $J_r$ is itself an interval, it
must contain the whole gap between $I_s$ and $I_{s+1}$. Therefore,
\[
x_1^r
=
\Bigl|J_r\cap \Bigl(\bigcup_{j=1}^K I_j\Bigr)^c\Bigr|
\ge d(I_s,I_{s+1})
\ge I^\star.
\]
Hence
\[
\sum_{r=1}^K (x_1^r+x_3^r)\ge x_1^r\ge I^\star.
\]}

\upd{\textit{Case 2:} \textit{Every $J_r$ intersects at most one true interval.}}

\upd{Since $x_2^{k\ell}>0$ with $k\neq \ell$, the overlap matching is not diagonal. We claim that in
this case, some true interval is not intersected by any candidate interval. Indeed, if every
$I_j$ intersected some $J_r$, then, because there are exactly $K$ disjoint $I$-intervals and
exactly $K$ disjoint $J$-intervals, and each $J_r$ intersects at most one $I_j$, the overlap
relation would define a bijection between $\{J_1,\ldots,J_K\}$ and $\{I_1,\ldots,I_K\}$. Since
both families are ordered left-to-right, this bijection must preserve order, hence must be the
identity map. This contradicts the existence of an off-diagonal overlap $x_2^{k\ell}>0$ with
$k\neq \ell$.}

\upd{Therefore there exists $s\in[K]$ such that $I_s\cap J_r=\varnothing \text{ for all } r\in[K].$ Equivalently, $I_s\subset \Bigl(\bigcup_{r=1}^K J_r\Bigr)^c,$ so that
\[
x_3^s
=
\Bigl|I_s\cap \Bigl(\bigcup_{r=1}^K J_r\Bigr)^c\Bigr|
=
|I_s|
\ge I^\star.
\]
Thus, in this case too it holds
\[
\sum_{r=1}^K (x_1^r+x_3^r)\ge x_3^s\ge I^\star.
\]}
\end{proof}

\upd{\begin{lemma} \label{lemma:good-decompose}
Let $\{I_1,\ldots,I_K\}$ and $\{J_1,\ldots,J_K\}$ be two collections of disjoint intervals, and define $x_1^k, x_2^{ij}, x_3^k$ and $x_4$ as in the proof of Theorem \ref{thm:multiple-oracle}. Then, for every $k\in[K]$,
\[
|I_k \Delta J_k|
=
x_1^k + x_3^k + \sum_{\ell\neq k} x_2^{k\ell} + \sum_{\ell\neq k} x_2^{\ell k}.
\]
\end{lemma}}

\upd{\begin{proof}[Proof of Lemma \ref{lemma:good-decompose}]
Fix $k\in[K]$. Since the intervals $\{I_\ell\}_{\ell=1}^K$ are disjoint, we may decompose
\[
J_k
=
\Bigl(J_k \cap \Bigl(\bigcup_{\ell=1}^K I_\ell\Bigr)^c\Bigr)
\;\cup\;
\bigcup_{\ell=1}^K (J_k\cap I_\ell),
\]
disjointly. Hence
\[
J_k\setminus I_k
=
\Bigl(J_k \cap \Bigl(\bigcup_{\ell=1}^K I_\ell\Bigr)^c\Bigr)
\;\cup\;
\bigcup_{\ell\neq k}(J_k\cap I_\ell),
\]
again disjointly, so that $|J_k\setminus I_k|
=
x_1^k+\sum_{\ell\neq k}x_2^{k\ell}.$ Similarly, since the intervals $\{J_\ell\}_{\ell=1}^K$ are disjoint,
\[
I_k\setminus J_k
=
\Bigl(I_k \cap \Bigl(\bigcup_{\ell=1}^K J_\ell\Bigr)^c\Bigr)
\;\cup\;
\bigcup_{\ell\neq k}(I_k\cap J_\ell),
\]
disjointly. By definition of $x_2^{\ell, k}$, we obtain $|I_k\setminus J_k|
=
x_3^k+\sum_{\ell\neq k}x_2^{\ell k}.$ Finally, in light of $|I_k\Delta J_k|
=
|I_k\setminus J_k|+|J_k\setminus I_k|,$
combining above completes the proof.
\end{proof}}
\end{document}